\newcommand*{\ICML}{}

\newcommand*{\CAMREADY}{}

\ifdefined\ARXIV
	\documentclass[11pt]{article}
	\usepackage{times}
	
	\usepackage{natbib}
	
	\renewcommand{\cite}[1]{\citep{#1}}
	\setcitestyle{authoryear,round,citesep={;},aysep={,},yysep={;}}
	
	\usepackage[utf8]{inputenc} 
	\usepackage[T1]{fontenc}    
	\usepackage{booktabs}       
	\usepackage{microtype}      
	\usepackage{xcolor}         
	
	\usepackage{tabularx}
	\usepackage{thmtools}
	\usepackage{thm-restate}
	\usepackage[margin=1in]{geometry}
	
	\usepackage{algorithm}
	\usepackage[noend]{algpseudocode}
	
	\usepackage{comment}
	\usepackage{hyperref}
	\definecolor{mydarkblue}{rgb}{0,0.08,0.55}
	\hypersetup{ %
		pdftitle={},
		pdfauthor={},
		pdfsubject={},
		pdfkeywords={},
		pdfborder=0 0 0,
		pdfpagemode=UseNone,
		colorlinks=true,
		linkcolor=mydarkblue,
		citecolor=mydarkblue,
		filecolor=mydarkblue,
		urlcolor=mydarkblue,
		pdfview=FitH}
	
	\renewenvironment{abstract}
	{\centerline{\large\bf Abstract}\vspace{0.7ex}%
		\bgroup\leftskip 20pt\rightskip 20pt\small\noindent\ignorespaces}%
	{\par\egroup\vskip 0.25ex}

	\title{\selectfont\textbf{Implicit Regularization in Hierarchical Tensor Factorization and Deep Convolutional Neural Networks}}
	\author{%
		Noam Razin\\ 
		{\fontsize{10pt}{10pt}\selectfont Tel Aviv University}\\
		{\fontsize{9pt}{9pt}\selectfont\texttt{noamrazin@mail.tau.ac.il}} \\
		\and
		Asaf Maman\\
		{\fontsize{10pt}{10pt}\selectfont Tel Aviv University}\\
		{\fontsize{9pt}{9pt}\selectfont\texttt{asafmaman@mail.tau.ac.il}}		
		\and
		Nadav Cohen\\
		{\fontsize{10pt}{10pt}\selectfont Tel Aviv University}\\
		{\fontsize{9pt}{9pt}\selectfont\texttt{cohennadav@cs.tau.ac.il}}		
	}	
	\date{}
\fi

\ifdefined\NEURIPS
	\PassOptionsToPackage{numbers}{natbib}
	\documentclass{article}
	\ifdefined\CAMREADY
		\usepackage[final]{neurips_2020}
	\else
		\usepackage{neurips_2020}
	\fi
	\usepackage{footmisc}
	\usepackage[utf8]{inputenc} 
	\usepackage[T1]{fontenc}    
	\usepackage{hyperref}       	
	\usepackage{url}            	
	\usepackage{booktabs}       
	\usepackage{amsfonts}       
	\usepackage{nicefrac}       	
	\usepackage{microtype}      
\fi

\ifdefined\CVPR
	\documentclass[10pt,twocolumn,letterpaper]{article}
	\usepackage{footmisc}
	\usepackage{cvpr}
	\usepackage{times}
	\usepackage{epsfig}
	\usepackage{graphicx}
	\usepackage{amsmath}
	\usepackage{amssymb}
	\usepackage[square,comma,numbers]{natbib}
\fi
\ifdefined\AISTATS
	\documentclass[twoside]{article}
	\ifdefined\CAMREADY
		\usepackage[accepted]{aistats2015}
	\else
		\usepackage{aistats2015}
	\fi
	\usepackage{url}
	\usepackage[round,authoryear]{natbib}
\fi
\ifdefined\ICML
	\documentclass{article}
	\usepackage{footmisc}
	\usepackage{microtype}
	\usepackage{graphicx}
	\usepackage{booktabs}
	\usepackage{hyperref}
	
	\ifdefined\CAMREADY
		\usepackage[accepted]{icml2022}
	\else
		\usepackage{icml2022}
	\fi
\fi
\ifdefined\ICLR
	\documentclass{article}
	\usepackage{iclr2019_conference,times}
	\usepackage{footmisc}
	\usepackage{hyperref}
	\usepackage{url}

	\ifdefined\CAMREADY
		\iclrfinalcopy
	\fi
\fi
\ifdefined\COLT
	\ifdefined\CAMREADY
		\documentclass{colt2017}
	\else
		\documentclass[anon,12pt]{colt2017}
	\fi
	\usepackage{times}
\fi

\usepackage{color}
\usepackage{amsfonts}
\usepackage{graphicx}
\usepackage{nicefrac}
\usepackage{bbm}
\usepackage{wrapfig}
\usepackage{tikz}
\usepackage[titletoc,title]{appendix}
\usepackage{stmaryrd}
\usepackage{comment}
\usepackage{endnotes}
\usepackage{hyperendnotes}
\usepackage{enumerate}
\usepackage{enumitem}
\usepackage[normalem]{ulem}

\usepackage[small]{caption}
\usepackage[font=footnotesize]{subcaption}

\ifdefined\COLT
\else
	\usepackage{amsmath,amsthm,amssymb,mathtools}
\fi

\ifdefined\COLT
	\newtheorem{claim}[theorem]{Claim}
	\newtheorem{fact}[theorem]{Fact}
	\newtheorem{procedure}{Procedure}
	\newtheorem{conjecture}{Conjecture}	
	\newtheorem{hypothesis}{Hypothesis}	
	\newcommand{\qed}{\hfill\ensuremath{\blacksquare}}
\else
	\newtheorem{lemma}{Lemma}
	
	\newtheorem{theorem}{Theorem}
	\newtheorem{proposition}{Proposition}

	\theoremstyle{definition}
	\newtheorem{definition}{Definition}
\fi

\ifdefined\ABBRVREFS
	\newcommand{\lem}{Lem.}
	\newcommand{\lems}{Lem.}
	
	\newcommand{\thm}{Thm.}
	\newcommand{\prop}{Prop.}
	
	\newcommand{\defin}{Def.}
	\newcommand{\eq}{Eq.}
	\newcommand{\eqs}{Eqs.}
	\newcommand{\fig}{Fig.}
	\newcommand{\figs}{Figs.}
	\newcommand{\tab}{Table}
		
	\newcommand{\sect}{Section}
	\newcommand{\sects}{Sections}
	\newcommand{\subsect}{Section}
	\newcommand{\subsects}{Sections}	
	\newcommand{\app}{Appendix}
		
	\newcommand{\subapp}{Appendix}
	
\else
	\newcommand{\lem}{Lemma}
	\newcommand{\lems}{Lemmas}
	
	\newcommand{\thm}{Theorem}
	\newcommand{\prop}{Proposition}
	
	\newcommand{\defin}{Definition}
	\newcommand{\eq}{Equation}
	\newcommand{\eqs}{Equations}
	\newcommand{\fig}{Figure}
	\newcommand{\figs}{Figures}
	\newcommand{\tab}{Table}
		
	\newcommand{\sect}{Section}
	\newcommand{\sects}{Sections}
	\newcommand{\subsect}{Section}
	\newcommand{\subsects}{Sections}	
	\newcommand{\app}{Appendix}
		
	\newcommand{\subapp}{Appendix}
		
\fi

\definecolor{green}{rgb}{0.0, 0.6, 0}

\def\be{\begin{equation}}
	\def\ee{\end{equation}}
\def\beas{\begin{eqnarray*}}
	\def\eeas{\end{eqnarray*}}
\def\bea{\begin{eqnarray}}
	\def\eea{\end{eqnarray}}
\newcommand{\hbf}{{\mathbf h}}
\newcommand{\xbf}{{\mathbf x}}
\newcommand{\ybf}{{\mathbf y}}

\newcommand{\ubf}{{\mathbf u}}
\newcommand{\vbf}{{\mathbf v}}

\newcommand{\wbf}{{\mathbf w}}

\newcommand{\ebf}{{\mathbf e}}

\newcommand{\Abf}{{\mathbf A}}
\newcommand{\Bbf}{{\mathbf B}}
\newcommand{\Wbf}{{\mathbf W}}
\newcommand{\Xbf}{{\mathbf X}}

\newcommand{\Qbf}{{\mathbf Q}}
\newcommand{\Ubf}{{\mathbf U}}
\newcommand{\Vbf}{{\mathbf V}}

\newcommand{\A}{{\mathcal A}}
\newcommand{\B}{{\mathcal B}}

\newcommand{\RR}{{\mathcal R}}
\renewcommand{\S}{{\mathcal S}}
\newcommand{\T}{{\mathcal T}}
\newcommand{\U}{{\mathcal U}}
\newcommand{\V}{{\mathcal V}}
\newcommand{\W}{{\mathcal W}}

\renewcommand{\L}{\mathcal{L}}

\newcommand{\R}{{\mathbb R}}
\newcommand{\N}{{\mathbb N}}

\newcommand{\abs}[1]{\left\lvert #1 \right\rvert}
\newcommand{\norm}[1]{\left\| #1 \right\|}
\newcommand{\normbig}[1]{\bigl \| #1 \bigr \|}
\newcommand{\normnoflex}[1]{\| #1 \|}

\newcommand{\inprod}[2]  {\left\langle{#1},{#2}\right\rangle}
\newcommand{\inprodbig}[2]  {\bigl \langle{#1},{#2} \bigr \rangle}
\newcommand{\inprodBig}[2]  {\Bigl \langle{#1},{#2} \Bigr \rangle}
\newcommand{\inprodbigg}[2]  {\biggl \langle{#1},{#2} \biggr \rangle}
\newcommand{\inprodnoflex}[2]{\langle{#1},{#2}\rangle}

\newcommand{\vectbig}[1]{\text{vec}\bigl ( #1 \bigr )} 
\newcommand{\vectnoflex}[1]{\text{vec}( #1 )} 
\newcommand{\mat}[2]{\left\llbracket#1 ; #2\right\rrbracket}

\newcommand{\matbig}[2]{\bigl\llbracket#1 ; #2 \bigr\rrbracket}

\newcommand{\tenp}{\otimes}
\newcommand{\kronp}{\odot}

\definecolor{xcolor-gray}{gray}{0.95}

\newcommand{\rank}{\mathrm{rank}}

\newcommand{\seprank}{\mathrm{sep}}

\newcommand{\matrixend}{\Wbf_{\text{M}}}
\newcommand{\tftensorend}{\W_{\text{T}}}
\newcommand{\tensorend}{\W_{\text{H}}}
\newcommand{\tensorendmap}{\mathcal{H}}
\newcommand{\htmodetree}{\T}
\newcommand{\mfendloss}{\L_{\text{M}}}
\newcommand{\mfobj}{\phi_{\text{M}}}
\newcommand{\tfendloss}{\L_{\text{T}}}
\newcommand{\tfobj}{\phi_{\text{T}}}
\newcommand{\htfendloss}{\L_{\text{H}}}
\newcommand{\htfobj}{\phi_{\text{H}}}
\newcommand{\mfcomp}[1]{\mathcal{C}^{ ( #1 ) }_{\text{M}}}
\newcommand{\tfcomp}[1]{\mathcal{C}^{ ( #1 ) }_{\text{T}}}
\newcommand{\htfcomp}[2]{\mathcal{C}^{ ( #1, #2 ) }_{\text{H}}}

\newcommand{\mfsing}[1]{\sigma^{ ( #1 ) }_{\text{M}}}
\newcommand{\tfcompnorm}[1]{\sigma^{ ( #1 ) }_{\text{T}}}
\newcommand{\htfcompnorm}[2]{\sigma^{ ( #1, #2 ) }_{\text{H}}}

\newcommand{\tensorpart}[2]{\W^{(#1, #2)}}

\newcommand{\weightmat}[1]{\Wbf^{(#1)}}
\newcommand{\padcol}{\mathrm{PadC}}
\newcommand{\padrow}{\mathrm{PadR}}
\newcommand{\localcomp}{\mathrm{LC}}

\newcommand{\children}{C}
\newcommand{\subtree}{\htmodetree}
\newcommand{\parent}{Pa}
\newcommand{\interior}{\mathrm{int}}
\newcommand{\lsib}{\mathrm{\overleftarrow{S}}}
\newcommand{\rsib}{\mathrm{\overrightarrow{S}}}

\DeclareFontFamily{U}{mathx}{\hyphenchar\font45}
\DeclareFontShape{U}{mathx}{m}{n}{<-> mathx10}{}
\DeclareSymbolFont{mathx}{U}{mathx}{m}{n}
\DeclareMathAccent{\widebar}{0}{mathx}{"73}

\definecolor{darkspringgreen}{rgb}{0.09, 0.45, 0.27}
\usetikzlibrary{decorations.pathreplacing,calligraphy}
\usetikzlibrary{patterns}

\ifdefined\ENABLEENDNOTES
	\renewcommand{\theendnote}{\alph{endnote}} 
\else
	\renewcommand{\endnote}[1]{\null} 
\fi

\ifdefined\CVPR
	\usepackage[breaklinks=true,bookmarks=false]{hyperref}
	\ifdefined\CAMREADY
		\cvprfinalcopy 
	\fi
	
\fi

\ifdefined\ARXIV
	\newcommand*{\ABBR}{}
\fi
\ifdefined\ICML
	\newcommand*{\ABBR}{}
\fi
\ifdefined\NEURIPS
	\newcommand*{\ABBR}{}
\fi
\ifdefined\ICLR
	\newcommand*{\ABBR}{}
\fi
\ifdefined\COLT
	\newcommand*{\ABBR}{}
\fi
\ifdefined\ABBR
	\newcommand{\eg}{{\it e.g.}}
	\newcommand{\ie}{{\it i.e.}}
	\newcommand{\cf}{{\it cf.}}

\fi

\usepackage{titlesec}
\begin{document}
	

	\ifdefined\ARXIV
		\setenumerate{itemsep=0pt}
		\setitemize{itemsep=1pt}
		\setlength{\parskip}{0.1em}
	
		\maketitle
	\fi
	\ifdefined\NEURIPS
	\title{Paper Title}
		\author{
			Author 1 \\
			Author 1 Institution \\	
			\texttt{author1@email} \\
			\And
			Author 1 \\
			Author 1 Institution \\	
			\texttt{author1@email} \\
		}
		\maketitle
	\fi
	\ifdefined\CVPR
		\title{Paper Title}
		\author{
			Author 1 \\
			Author 1 Institution \\	
			\texttt{author1@email} \\
			\and
			Author 2 \\
			Author 2 Institution \\
			\texttt{author2@email} \\	
			\and
			Author 3 \\
			Author 3 Institution \\
			\texttt{author3@email} \\
		}
		\maketitle
	\fi
	\ifdefined\AISTATS
		\twocolumn[
		\aistatstitle{Paper Title}
		\ifdefined\CAMREADY
			\aistatsauthor{Author 1 \And Author 2 \And Author 3}
			\aistatsaddress{Author 1 Institution \And Author 2 Institution \And Author 3 Institution}
		\else
			\aistatsauthor{Anonymous Author 1 \And Anonymous Author 2 \And Anonymous Author 3}
			\aistatsaddress{Unknown Institution 1 \And Unknown Institution 2 \And Unknown Institution 3}
		\fi
		]	
	\fi
	\ifdefined\ICML
		\icmltitlerunning{Implicit Regularization in Hierarchical Tensor Factorization and Deep Convolutional Neural Networks}
		\twocolumn[
		\icmltitle{Implicit Regularization in Hierarchical Tensor Factorization \\ and Deep Convolutional Neural Networks} 
		\icmlsetsymbol{equal}{*}
		\begin{icmlauthorlist}
			\icmlauthor{Noam Razin}{tau} 
			\icmlauthor{Asaf Maman}{tau}
			\icmlauthor{Nadav Cohen}{tau}
		\end{icmlauthorlist}
		\icmlaffiliation{tau}{Blavatnik School of Computer Science, Tel Aviv University, Israel}
		\icmlcorrespondingauthor{Noam Razin}{noamrazin@mail.tau.ac.il}
		\icmlkeywords{Implicit Regularization, Hierarchical Tensor Factorization, Convolutional Neural Networks}
		\vskip 0.3in
		]
		\printAffiliationsAndNotice{} 
	\fi
	\ifdefined\ICLR
		\title{Paper Title}
		\author{
			Author 1 \\
			Author 1 Institution \\
			\texttt{author1@email}
			\And
			Author 2 \\
			Author 2 Institution \\
			\texttt{author2@email}
			\And
			Author 3 \\ 
			Author 3 Institution \\
			\texttt{author3@email}
		}
		\maketitle
	\fi
	\ifdefined\COLT
		\title{Paper Title}
		\coltauthor{
			\Name{Author 1} \Email{author1@email} \\
			\addr Author 1 Institution
			\And
			\Name{Author 2} \Email{author2@email} \\
			\addr Author 2 Institution
			\And
			\Name{Author 3} \Email{author3@email} \\
			\addr Author 3 Institution}
		\maketitle
	\fi

	\begin{abstract}
In the pursuit of explaining implicit regularization in deep learning, prominent focus was given to matrix and tensor factorizations, which correspond to simplified neural networks.
It was shown that these models exhibit an implicit tendency towards low matrix and tensor ranks, respectively.
Drawing closer to practical deep learning, the current paper theoretically analyzes the implicit regularization in hierarchical tensor factorization, a model equivalent to certain deep convolutional neural networks.
Through a dynamical systems lens, we overcome challenges associated with hierarchy, and establish implicit regularization towards low hierarchical tensor rank.
This translates to an implicit regularization towards locality for the associated convolutional networks.
Inspired by our theory, we design explicit regularization discouraging locality, and demonstrate its ability to improve the performance of modern convolutional networks on non-local tasks, in defiance of conventional wisdom by which architectural changes are needed.
Our work highlights the potential of enhancing neural networks via theoretical analysis of their implicit regularization.
\end{abstract}

	\ifdefined\COLT
		\medskip
	\fi

	
	\section{Introduction}
\label{sec:intro}

One of the central mysteries in deep learning is the ability of neural networks to generalize when having far more learnable weights than training examples.
The fact that this generalization takes place even in the absence of any explicit regularization~\cite{zhang2017understanding} has given rise to a common view by which gradient-based optimization induces an \emph{implicit regularization}~---~a tendency to fit data with predictors of low complexity (see, \eg,~\citet{neyshabur2017implicit}).
Efforts to mathematically formalize this intuition have led to theoretical focus on \emph{matrix} and \emph{tensor factorizations}.\footnote{
	The term “tensor factorization'' refers throughout to the classic \emph{CP factorization} (see~\citet{kolda2009tensor,hackbusch2012tensor} for an introduction to tensor factorizations).
}

Matrix factorization refers to minimizing a given loss (over matrices) by parameterizing the solution as a product of matrices, and optimizing the resulting objective via gradient descent.
Tensor factorization is a generalization of this procedure to multi-dimensional arrays.
There, a tensor is learned through gradient descent over a sum-of-outer-products parameterization.
Various works analyzed the implicit regularization in matrix and tensor factorizations~\cite{gunasekar2017implicit,li2018algorithmic,du2018algorithmic,arora2019implicit,razin2020implicit,chou2020gradient,li2021towards,razin2021implicit,ge2021understanding}.
Though initially conjectured to be equivalent to norm minimization (see~\citet{gunasekar2017implicit}), recent studies~\cite{arora2019implicit,razin2020implicit,chou2020gradient,li2021towards,razin2021implicit} suggest that this is not the case in general, and instead adopt a dynamical systems approach.
They establish that gradient descent (with small learning rate and near-zero initialization) induces a momentum-like effect on the components of a factorization, leading them to move slowly when small and quickly when large.
This implies a form of incremental learning that results in low matrix rank solutions for matrix factorization, and low tensor rank solutions for tensor factorization.

From a deep learning perspective, matrix factorization can be seen as a linear neural network (fully connected neural network with linear activation), and, in a similar vein, tensor factorization corresponds to a certain shallow (depth two) non-linear convolutional neural network (see~\citet{cohen2016expressive,razin2021implicit}).
As theoretical surrogates for deep learning, the practical relevance of these models is limited.
The former lacks non-linearity, while the latter misses depth~---~both crucial features of modern neural networks.
A natural extension of matrix and tensor factorizations that accounts for both non-linearity and depth is \emph{hierarchical tensor factorization},\footnote{
	The term “hierarchical tensor factorization'' refers throughout to a variant of the Hierarchical Tucker factorization \cite{hackbusch2009new}, presented in \sect~\ref{sec:htf}.
}
which corresponds to a class of \emph{deep non-linear} convolutional neural networks~\cite{cohen2016expressive} that have demonstrated promising performance in practice~\cite{cohen2014simnets,cohen2016deep,sharir2016tensorial,stoudenmire2018learning,grant2018hierarchical,felser2021quantum}, and have been key to the study of expressiveness in deep learning~\cite{cohen2016expressive,cohen2016convolutional,cohen2017inductive,cohen2017analysis,cohen2018boosting,sharir2018expressive,levine2018benefits,levine2018deep,balda2018tensor,khrulkov2018expressive,khrulkov2019generalized,levine2019quantum}.

In this paper, we conduct the first theoretical analysis of implicit regularization in hierarchical tensor factorization.
As opposed to tensor factorization, which is a simple construct dating back to at least the early 20'th century~\cite{hitchcock1927expression}, hierarchical tensor factorization was formally introduced only recently~\cite{hackbusch2009new}, and is much more elaborate.
We circumvent the challenges brought forth by the added hierarchy through identification of \emph{local components}, and characterization of their evolution under gradient descent (with small learning rate and near-zero initialization).
The characterization reveals that they are subject to a momentum-like effect, identical to that in matrix and tensor factorizations.
Accordingly, local components are learned incrementally, leading to solutions with low \emph{hierarchical tensor rank}~---~a central concept in tensor analysis~\cite{grasedyck2010hierarchical,grasedyck2013literature}. 
Theoretical and empirical demonstrations validate our analysis.

For the deep convolutional networks corresponding to hierarchical tensor factorization, hierarchical tensor rank is known to measure the strength of dependencies modeled between spatially distant input regions (patches of pixels in the context of image classification)~---~see~\citet{cohen2017inductive,levine2018benefits,levine2018deep}.
The established tendency towards low hierarchical tensor rank therefore implies a bias towards local (short-range) dependencies, in accordance with the fact that convolutional networks often struggle or completely fail to learn tasks entailing long-range dependencies (see, \eg,~\citet{wang2016temporal,linsley2018learning,mlynarski2019convolutional,hong2020graph, kim2020disentangling}).
However, while this failure is typically attributed solely to a limitation in expressive capability (\ie~to an inability of convolutional networks to represent functions modeling long-range dependencies~---~see~\citet{cohen2017inductive,linsley2018learning,kim2020disentangling}), our analysis reveals that it also originates from implicit regularization.
This suggests that the difficulty in learning long-range dependencies may be countered via \emph{explicit} regularization, in contrast to conventional wisdom by which architectural modifications are needed.
Through a series of controlled experiments we confirm this prospect, demonstrating that explicit regularization designed to promote high hierarchical tensor rank can significantly improve the performance of modern convolutional networks (\eg~ResNet18 and ResNet34 from~\citet{he2016deep}) on tasks involving long-range dependencies.

Our results bring forth the possibility that deep learning architectures considered suboptimal for certain tasks (\eg~convolutional networks for natural language processing tasks) may be greatly improved through a right choice of explicit regularization.
Theoretical understanding of implicit regularization may be key to discovering such regularizers.

\medskip

The remainder of the paper is organized as follows.
\sect~\ref{sec:prelim} outlines existing dynamical characterizations of implicit regularization in matrix and tensor factorizations.
\sect~\ref{sec:htf} presents the hierarchical tensor factorization model, as well as its interpretation as a deep non-linear convolutional network.
In \sect~\ref{sec:inc_rank_lrn} we characterize the dynamics of gradient descent over hierarchical tensor factorization, establishing that they lead to low hierarchical tensor rank.
\sect~\ref{sec:low_htr_implies_locality} explains why low hierarchical tensor rank means locality for the corresponding convolutional network.
In \sect~\ref{sec:countering_locality} we demonstrate that the locality of modern convolutional networks can be countered using dedicated explicit regularization.
Finally, \sect~\ref{sec:related} reviews related work and \sect~\ref{sec:summary} provides a summary.

	\section{Preliminaries: Matrix and Tensor Factorizations}
\label{sec:prelim}

To put our work into context, we overview known characterizations of implicit regularization in matrix and tensor factorizations.

Throughout the paper, when referring to a norm we mean the standard Frobenius (Euclidean) norm, denoted $\norm{\cdot}$.
For $N \in \N$, we let $[N] := \{ 1, \ldots, N \}$.
For vectors, matrices, or tensors, parenthesized superscripts denote elements in a collection, \eg~$( \wbf^{(n)} \in \R^D )_{ n = 1}^N$, while subscripts refer to entries, \eg~$\Wbf_{i,j} \in \R$ is the $(i,j)$'th entry of $\Wbf \in \R^{D, D'}$.
An exception to this rule are the subscripts “M'', “T'', and “H'', which signify relation to matrix, tensor, and hierarchical tensor factorizations, respectively.
A colon indicates all entries in an axis, \eg~$\Wbf_{i, :} \in \R^{D'}$ is the $i$'th  row and $\Wbf_{:, j} \in \R^{D}$ is the $j$'th column of~$\Wbf$.

\begin{figure*}[t]
	\vspace{-1mm}
	\begin{center}
		\hspace{-3mm}
		\includegraphics[width=1\textwidth]{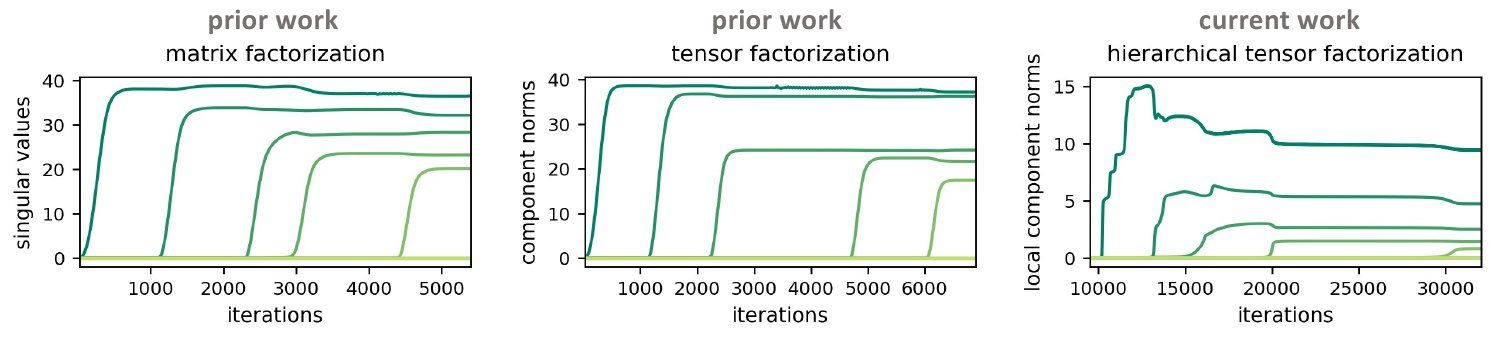}
	\end{center}
	\vspace{-4mm}
	\caption{
		Dynamics of gradient descent over matrix, tensor, and hierarchical tensor factorizations~---~incremental learning leads to low matrix, tensor, and hierarchical tensor ranks, respectively.
		\textbf{Left:} top $10$ singular values of the end matrix in a depth $3$ matrix factorization when minimizing the mean squared error over observed entries from a matrix rank $5$ ground truth (matrix completion loss).
		\textbf{Middle:} top $10$ component norms of an order $3$ tensor factorization when minimizing the mean squared error over observed entries from a tensor rank $5$ ground truth (tensor completion loss).
		\textbf{Right:} top $10$ local component norms at node $\{1, 2, 3, 4 \}$ of an order $4$ hierarchical tensor factorization induced by a perfect binary mode tree (\defin~\ref{def:mode_tree}), when minimizing the mean squared error over observed entries from a hierarchical tensor rank $(5, 5, 5, 5, 5, 5)$ (\defin~\ref{def:ht_rank}) ground truth (tensor completion loss).
		\textbf{All:} initial factorization weights were sampled independently from a zero-mean Gaussian distribution. 
		Notice that, in accordance with existing analyses for matrix and tensor factorizations (\sect~\ref{sec:prelim}) and our analysis for hierarchical tensor factorization (\sect~\ref{sec:inc_rank_lrn}), the singular values, component norms, and local component norms move slowly when small and quickly when large, creating an incremental learning process that results in effectively low matrix, tensor, and hierarchical tensor rank solutions, respectively.
		In all factorizations this implicit regularization led to accurate reconstruction of the low rank ground truth (reconstruction errors were $0.001$, $0.001$, and $0.005$, respectively).
		For further details such as loss definitions and factorization dimensions, as well as additional experiments for hierarchical tensor factorization, see	\app~\ref{app:experiments}.
	}
	\label{fig:mf_tf_htf_dynamics}
	\vspace{-1mm}
\end{figure*}

\subsection{Matrix Factorization: Incremental Matrix Rank Learning}
\label{sec:prelim:mf}

Consider the task of minimizing a differentiable and locally smooth\footnote{
	A differentiable function $g: \R^D \to \R$ is \emph{locally smooth} if for any compact subset $\B \subset \R^D$ there exists $\beta \in \R_{\geq 0}$ such that $\normnoflex{ \nabla g(\xbf) - \nabla g(\ybf) } \leq \beta \cdot \norm{\xbf - \ybf}$ for all $\xbf, \ybf \in \B$.
}
loss $\mfendloss : \R^{D, D'} \,{\to}\, \R_{\geq 0}$ ($D, D' \in \N$).
For example, $\mfendloss$ can be a matrix completion loss~---~mean squared error over observed entries from a ground truth matrix.
Matrix factorization with hidden dimensions $D_2, \ldots, D_{L} \in \N$ refers to parameterizing the solution $\matrixend \in \R^{D, D'}$ as a product of $L$ matrices, \ie~as $\matrixend = \Wbf^{(1)} \cdots \Wbf^{(L)}$, where $\Wbf^{(l)} \in \R^{D_{l }, D_{l + 1}}$ for $l = 1, \ldots, L$, $D_1 := D$, and $D_{L + 1} := D'$, and minimizing the resulting objective $\mfobj \big ( \Wbf^{(1)}, \ldots, \Wbf^{(L)} \big ) := \mfendloss (\matrixend)$ using gradient descent.
We call $\matrixend$ the \emph{end matrix} of the factorization.
It is possible to explicitly constrain the values that $\matrixend$ can take by limiting the hidden dimensions $D_2, \ldots, D_L$.
However, from an implicit regularization perspective, the case of interest is where the search space is unconstrained, thus we consider $D_2, \ldots, D_L \geq \min \{ D, D' \}$.
Matrix factorization can be viewed as applying a linear neural network for minimizing $\mfendloss$, and as such, serves a prominent theoretical model in deep learning (see \eg~\citet{gunasekar2017implicit,du2018algorithmic,li2018algorithmic,arora2019implicit,gidel2019implicit,mulayoff2020unique,blanc2020implicit,gissin2020implicit,razin2020implicit,chou2020gradient,yun2021unifying,li2021towards}).

Several characterizations of implicit regularization in matrix factorization have suggested that gradient descent, with small learning rate and near-zero initialization, induces a form of incremental matrix rank learning~\cite{gidel2019implicit,gissin2020implicit,chou2020gradient,li2021towards}.
Below we follow the presentation of~\citet{arora2019implicit}, which in line with other analyses, modeled small learning rate through the infinitesimal limit, \ie~via \emph{gradient flow}:
\[
\begin{split}
\frac{d}{dt} \Wbf^{(l)} (t) = - \frac{\partial}{\partial \Wbf^{(l) } } \mfobj \big ( \Wbf^{(1)} (t), \ldots, \Wbf^{(L)} (t) \big )
\end{split}
\]
for all $t \geq 0$ and $l \in [L]$.
Under gradient flow, the difference $\Wbf^{(l)} (t)^\top \Wbf^{(l)} (t) - \Wbf^{(l + 1)} (t) \Wbf^{(l + 1)} (t)^\top$ remains constant through time for any $l \in [L - 1]$~\cite{arora2018optimization}.
This implies that the \emph{unbalancedness magnitude}, defined as $\max\nolimits_{l} \normnoflex{  \Wbf^{(l)} (t)^\top \Wbf^{(l)} (t) - \Wbf^{(l + 1)} (t) \Wbf^{(l + 1)} (t)^\top }$, does not change through time, thus becomes relatively small as optimization moves away from the origin, more so the closer initialization is to zero.
Accordingly, it is common practice to treat the case of unbalancedness magnitude zero as an idealization of standard near-zero initializations (see, \eg,~\citet{saxe2014exact,arora2018optimization,bartlett2018gradient,lampinen2019analytic,arora2019implicit,elkabetz2021continuous,bah2022learning}).

With unbalancedness magnitude zero, the $r$'th singular value of the end matrix $\matrixend (t) = \Wbf^{(1)} ( t ) \cdots \Wbf^{(L)} ( t )$ ($r \in \left [ \min \{ D, D' \} \right ]$), denoted $\mfsing{r} (t) \in \R$, evolves by (\cf~\citet{arora2019implicit}):\footnote{
	The dynamical characterization of singular values in \eq~\eqref{eq:sing_val_mf_dyn} requires $\mfendloss$ to be analytic, a property met by standard loss functions such as the square and cross-entropy losses.
}
\be
\vspace{0.5mm}
\frac{d}{dt} \mfsing{r} (t) = \mfsing{r} (t)^{2 - \frac{2}{L}} L \inprodbig{ - \nabla \mfendloss ( \matrixend (t) ) }{ \mfcomp{r} (t)  }
\text{\,,}
\label{eq:sing_val_mf_dyn}
\ee
where $\mfcomp{r} (t) := \ubf^{(r)} (t) \vbf^{(r)} (t)^\top \in \R^{D, D'}$ is the $r$'th singular component of $\matrixend (t)$, meaning $\ubf^{(r)} (t) \in \R^D$ and $\vbf^{(r)} (t) \in \R^{D'}$ are, respectively, left and right singular vectors of $\matrixend (t)$ corresponding to~$\mfsing{r} (t)$.
As evident from \eq~\eqref{eq:sing_val_mf_dyn}, 
two factors govern the evolution rate of a singular value $ \mfsing{r} (t)$.
The first factor, $\inprodnoflex{ - \nabla \mfendloss ( \matrixend (t) ) }{ \mfcomp{r} (t)  }$, is a projection of the singular component $\mfcomp{r} (t)$ onto $- \nabla \mfendloss ( \matrixend (t) )$, the direction of steepest descent with respect to the end matrix.
The more the singular component is aligned with $- \nabla \mfendloss ( \matrixend (t) )$, the faster the singular value grows.
The second, more critical factor, is $\mfsing{r} (t)^{2 - \frac{2}{L } } L$, which implies that the rate of change of the singular value is proportional to its size exponentiated by $2 - 2 / L$ (recall that $L$ is the depth of the matrix factorization).
This brings rise to a momentum-like effect, which attenuates the movement of small singular values and accelerates the movement of large ones.
We may thus expect that if the matrix factorization is initialized near the origin, singular values progress slowly at first, and then, one after the other they reach a critical threshold and quickly rise, until convergence is attained.
Such incremental learning phenomenon leads to low matrix rank solutions.
It is demonstrated empirically in \fig~\ref{fig:mf_tf_htf_dynamics} (left), which reproduces an experiment from~\citet{arora2019implicit}.
We note that under certain technical conditions, the incremental matrix rank learning phenomenon can be used to prove exact matrix rank minimization~\cite{li2021towards}.

\subsection{Tensor Factorization: Incremental Tensor Rank Learning}
\label{sec:prelim:tf}

A depth two matrix factorization boils down to parameterizing a sought-after solution as a sum of tensor (outer) products between column vectors of $\Wbf^{(1)}$ and row vectors of $\Wbf^{(2)}$. 
Namely, since $\matrixend = \Wbf^{(1)} \Wbf^{(2)}$ we may write $\matrixend = \sum_{r = 1}^{R} \Wbf^{(1)}_{:, r} \tenp \Wbf^{(2)}_{r, :}$, where $R$ is the dimension shared between $\Wbf^{(1)}$ and $\Wbf^{(2)}$, and $\tenp$ stands for the tensor product.\footnote{
	Given $\U \in \R^{D_1, \ldots, D_N}, \V \in \R^{ H_1, \ldots, H_K }$, their tensor product $\U \tenp \V \in \R^{D_1, \ldots, D_N, H_1, \ldots, H_K}$ is defined by $[\U \tenp \V]_{d_1, \ldots, d_{N + K}} = \U_{d_1, \ldots, d_N} \cdot \V_{d_{N + 1}, \ldots, d_{N + K}}$.
	For vectors $\ubf \in \R^D, \vbf \in \R^{D'}$, the tensor product $\ubf \tenp \vbf$ is equal to $\ubf \vbf^\top \in \R^{D, D'}$.
}
Note that the minimal number of summands $R$ required for $\matrixend$ to express a given matrix $\Wbf$ is precisely the latter's matrix rank.

By allowing each summand to be a tensor product of more than two vectors, we may transition from a factorization for matrices to a factorization for tensors.
In tensor factorization, a sought-after solution $\tftensorend \in \R^{D_1, \ldots, D_N}$~---~an \emph{order} $N \geq 3$ tensor with \emph{modes} (axes) of \emph{dimensions} $D_1, \ldots, D_N \in \N$~---~is parameterized as $\tftensorend = \sum_{r = 1}^R \Wbf^{(1)}_{:,r} \tenp \cdots \tenp \Wbf^{(N)}_{:, r}$, where $\Wbf^{(n)} \in \R^{D_n, R}$ for $n \in [N]$.
Each term $ \Wbf^{(1)}_{:,r} \tenp \cdots \tenp \Wbf^{(N)}_{:, r}$ in this sum is called a \emph{component}, and $\tftensorend$ is referred to as the \emph{end tensor} of the factorization.
Given a differentiable and locally smooth loss $\tfendloss : \R^{D_1, \ldots, D_N} \to \R_{\geq 0}$, \eg~mean squared error over observed entries from a ground truth tensor (\ie~a tensor completion loss), the goal is to minimize the objective $\tfobj \big (\Wbf^{(1)}, \ldots, \Wbf^{(N)} \big ) := \tfendloss ( \tftensorend )$.
In analogy with matrix factorization, the minimal number of components $R$ required for $\tftensorend$ to express a given tensor $\W \in \R^{D_1, \ldots, D_N}$ is defined to be the latter's \emph{tensor rank}, and the case of interest is when $R$ is sufficiently large to not restrict tensor rank (\ie~to admit an unconstrained search space).

Similarly to how matrix factorization corresponds to a linear neural network, tensor factorization is known (see \citet{cohen2016expressive,razin2021implicit}) to be equivalent to a certain shallow (depth two) non-linear convolutional network (with multiplicative non-linearity).
By virtue of this equivalence, illustrated in \fig~\ref{fig:tf_htf_as_convnet} (top), tensor factorization is considered closer to practical deep learning than matrix factorization.

\begin{figure*}[t]
	\vspace{-2mm}
	\begin{center}
		\includegraphics[width=0.77\textwidth]{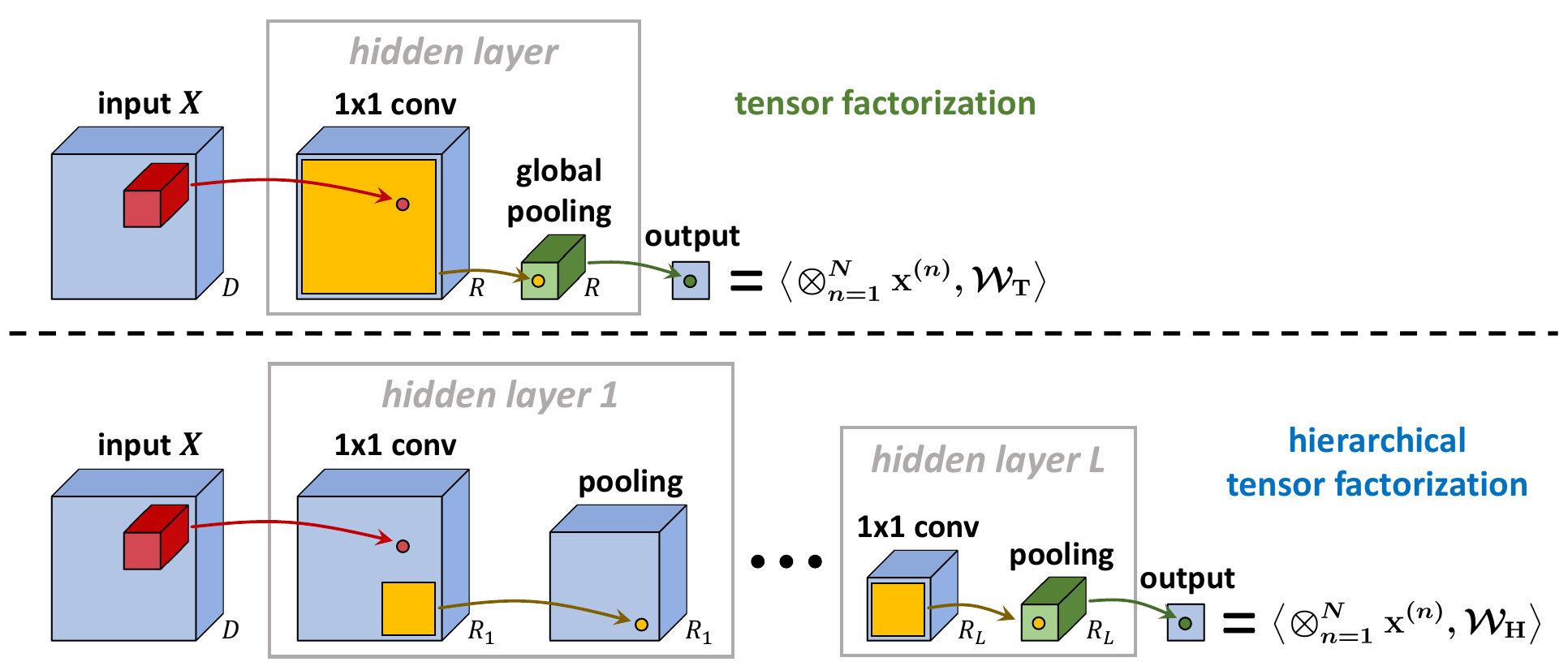}
	\end{center}
	\vspace{-2.75mm}
	\caption{
		Tensor factorization corresponds to a class of shallow (depth two) non-linear convolutional networks, while hierarchical tensor factorization corresponds to a class of \emph{deep} non-linear convolutional networks.
		These correspondences have been studied extensively (see references in \subsect~\ref{sec:htf:informal}).
		For completeness, we briefly describe them herein and provide a formal proof in \app~\ref{app:htf_cnn}.
		\textbf{Top:} the shallow network equivalent to tensor factorization processes an input $( \xbf^{(1)}, \ldots, \xbf^{(N)} ) \in \R^{D_1} \times \cdots \times \R^{D_N}$ (illustration assumes $D_1 = \cdots = D_N = D$ to avoid clutter) using a single hidden layer, which consists of: \emph{(i)} locally connected linear operator with $R$ channels, computing \smash{$( \Wbf^{(1)} )^\top \xbf^{(1)}, \ldots, ( \Wbf^{(N)} )^\top \xbf^{(N)}$} with learnable weights $\Wbf^{(1)}, \ldots, \Wbf^{(N)}$ (this operator is referred to as ‘‘$1 \times 1$ conv'' in appeal to the common case of weight sharing, \ie~$\Wbf^{(1)} = \cdots = \Wbf^{(N)}$); and \emph{(ii)} channel-wise global product pooling (multiplicative non-linearity).
		Summing over the resulting activations then yields the scalar output $\inprodbig{ \tenp_{n = 1}^N \xbf^{(n)} }{ \sum_{r = 1}^R \tenp_{n = 1}^N \Wbf^{(n)}_{:, r} } = \inprodbig{  \tenp_{n = 1}^N \xbf^{(n)}  }{ \tftensorend }$.
		Hence, functions realized by this class of networks are naturally represented via tensor factorization, where the number of components $R$ and the weight matrices $\Wbf^{(1)}, \ldots, \Wbf^{(N)}$ of the factorization correspond to the width and learnable weights of the network, respectively.
		\textbf{Bottom:} for a hierarchical tensor factorization induced by a perfect $P$-ary mode tree (\defin~\ref{def:mode_tree}), the equivalent network is a deep variant of that associated with tensor factorization.
		It has $L = \log_P N$ hidden layers instead of just one, with channel-wise product pooling operating over windows of size $P$ as opposed to globally.
		After passing an input \smash{$( \xbf^{(1)}, \ldots, \xbf^{(N)} ) \in \R^{D_1} \times \cdots \times \R^{D_N}$} through all hidden layers, a final linear layer produces the network's scalar output $\inprodbig{  \tenp_{n = 1}^N \xbf^{(n)}  }{ \tensorend }$, where $\tensorend$ is the end tensor of the hierarchical tensor factorization (\eq~\eqref{eq:ht_end_tensor}), whose weight matrices are equal to the network's learnable weights.
		Thus, functions realized by this class of networks are naturally represented via hierarchical tensor factorization.
		We note that, as shown in~\citet{cohen2016convolutional}, by considering \emph{generalized hierarchical tensor factorizations} it is possible to account for various non-linearities beyond multiplicative, \ie~for product pooling being converted to a different pooling operator (\eg~max or average), optionally preceded by a non-linear activation (\eg~rectified linear unit).
	}
	\label{fig:tf_htf_as_convnet}
	\vspace{-2mm}
\end{figure*}

As in matrix factorization, gradient flow over tensor factorization induces invariants of optimization.
In particular, the differences between squared norms of vectors in the same component, \ie~\smash{$\normnoflex{ \Wbf_{:,r}^{ (n) } (t) }^2 - \normnoflex{ \Wbf_{:, r}^{ (n') } (t) }^2$} for $n, n' \in [N]$ and $r \in [R]$, are constant through time~\cite{razin2021implicit}.
This leads to the following definition of unbalancedness magnitude: \smash{$\max_{n, n', r} \big | \normnoflex{ \Wbf_{:,r}^{ (n) } (t) }^2 - \normnoflex{ \Wbf_{:, r}^{ (n') } (t) }^2  \big |$}, which does not change during optimization, therefore remains small throughout if initialization is close to the origin.
Under the idealized assumption of unbalancedness magnitude zero (corresponding to infinitesimally small initialization), the norm of the $r$'th component in the factorization ($r \in [R]$), \ie~\smash{$\tfcompnorm{r} (t) := \normnoflex{ \tenp_{n = 1}^N \Wbf^{(n)}_{:, r} (t) }$}, evolves by (\cf~\citet{razin2021implicit}): 
\be
\frac{d}{dt} \tfcompnorm{r} (t) = \tfcompnorm{r} (t)^{2 - \frac{2}{N}} N \inprodbig{ - \nabla \tfendloss ( \tftensorend (t) ) }{ \tfcomp{r} (t) }
\text{\,,}
\label{eq:comp_norm_tf_dyn}
\ee
where $\tfcomp{r} (t) :=  \tenp_{n = 1}^N \widebar{\Wbf}^{(n)}_{:, r} (t)$, with $\widebar{\Wbf}^{(n)}_{:, r} (t)$ defined as $\Wbf^{(n)}_{:, r} (t) / \normnoflex{\Wbf^{(n)}_{:, r} (t)}$ for all $n \in [N]$ (by convention, if $\Wbf^{(n)}_{:, r} (t) = 0$ then $\widebar{\Wbf}^{(n)}_{:, r} (t) = 0$), denotes the $r$'th normalized component.
Comparing \eq~\eqref{eq:comp_norm_tf_dyn} to \eq~\eqref{eq:sing_val_mf_dyn} reveals that the evolution rate of a component norm in tensor factorization is \emph{structurally identical} to that of a singular value in matrix factorization.
Specifically, it is determined by two factors, analogous to those in \eq~\eqref{eq:sing_val_mf_dyn}:
\emph{(i)}~a projection of the normalized component \smash{$\tfcomp{r} (t)$ onto $ - \nabla \tfendloss ( \tftensorend (t) )$}, which encourages growth of components that are aligned with the direction of steepest descent with respect to the end tensor; 
and
\emph{(ii)}~\smash{$\tfcompnorm{r} (t)^{2 - \frac{2}{N}} N$}, which induces a momentum-like effect, leading component norms to move slower when small and faster when large.
This suggests that, in analogy with matrix factorization, components tend to be learned incrementally, yielding a bias towards low tensor rank.
\fig~\ref{fig:mf_tf_htf_dynamics} (middle) demonstrates the phenomenon empirically, reproducing an experiment from~\citet{razin2021implicit}.
Similarly to the case of matrix factorization, under certain technical conditions, the incremental tensor rank learning phenomenon can be used to prove exact tensor rank minimization~\cite{razin2021implicit}.

	\section{Hierarchical Tensor Factorization}
\label{sec:htf}

In this section we present the hierarchical tensor factorization model.
We begin by informally introducing the core concepts (\subsect~\ref{sec:htf:informal}), after which we delve into the formal definitions (\subsect~\ref{sec:htf:formal}).

\subsection{Informal Overview and Interpretation as Deep Non-Linear Convolutional Network}
\label{sec:htf:informal}

As discussed in \subsect~\ref{sec:prelim:tf}, tensor factorization produces an order $N$ end tensor through a sum of components, each combining $N$ vectors using the tensor product operator.
It is customary to represent this computation through a shallow tree structure with $N$ leaves, corresponding to the weight matrices $\weightmat{1}, \ldots, \weightmat{N}$, that are directly connected to the root, which computes the end tensor $\tftensorend = \sum_{r = 1}^R \Wbf^{(1)}_{:,r} \tenp \cdots \tenp \Wbf^{(N)}_{:, r}$.

Generalizing this scheme to an arbitrary tree gives rise to hierarchical tensor factorization.
Given a tree, or formally, a \emph{mode tree} of the hierarchical tensor factorization (\defin~\ref{def:mode_tree}), the scheme progresses from leaves to root.
Each internal node combines tensors produced by its children to form higher-order tensors, until finally the root outputs an order $N$ end tensor.
Different mode trees bring about different hierarchical tensor factorizations, which are essentially a composition of many local tensor factorizations, each corresponding to a different location in the mode tree.
We refer to the components of these local tensor factorizations as the \emph{local components} (\defin~\ref{def:local_comp}) of the hierarchical factorization~---~see \fig~\ref{fig:tf_htf_comp_tree} for an illustration.

\begin{figure*}
	\vspace{-0.25mm}
	\begin{center}
		\includegraphics[width=0.86\textwidth]{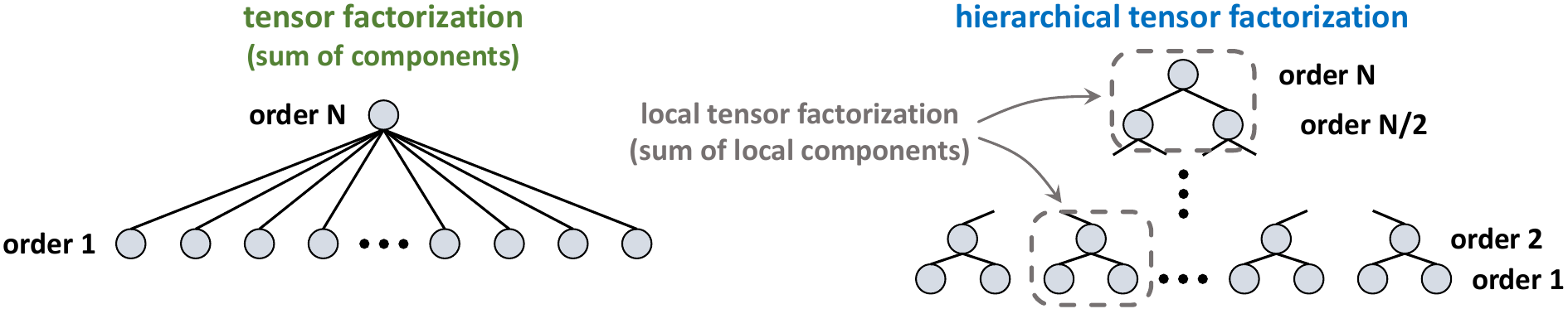}
	\end{center}
	\vspace{-2.75mm}
	\caption{
		Hierarchical tensor factorization consists of multiple local tensor factorizations.
		\textbf{Left:} tensor factorization represents an order $N$ tensor as a sum of components, each combining $N$ vectors through the tensor product operator.
		Accordingly, it is represented by a shallow tree where all leaves are directly connected to the root.
		\textbf{Right:} hierarchical tensor factorization adheres to an arbitrary tree structure (figure depicts a perfect binary tree), producing an order $N$ tensor by iteratively combining multiple local tensor factorizations.
		The components of the local tensor factorizations constituting the hierarchical tensor factorization are defined to be its local components.
		For a formal description of hierarchical tensor factorization see \subsect~\ref{sec:htf:formal}.
	}
	\label{fig:tf_htf_comp_tree}
	\vspace{-1.5mm}
\end{figure*}

A mode tree of a hierarchical tensor factorization induces a notion of rank called \emph{hierarchical tensor rank} (\defin~\ref{def:ht_rank}).
The hierarchical tensor rank is a tuple whose entries correspond to locations in the mode tree.
The value held by an entry is characterized by the number of local components at the corresponding location, similarly to how tensor rank is characterized by the number of components in a tensor factorization (see \subsect~\ref{sec:prelim:tf}).
Motivated by matrix and tensor ranks being implicitly minimized in matrix and tensor factorizations, respectively (\sect~\ref{sec:prelim}), in \sect~\ref{sec:inc_rank_lrn} we explore the possibility of hierarchical tensor rank being implicitly minimized in hierarchical tensor factorization.
That is, we investigate the prospect of gradient descent (with small learning rate and near-zero initialization) over hierarchical tensor factorization learning solutions that can be represented with few local components at all locations of the mode tree.

\vspace{-2mm}

\paragraph*{Equivalence to a class of deep non-linear convolutional networks}

As discussed in \sect~\ref{sec:prelim}, matrix factorization can be seen as a linear neural network, and, in a similar vein, tensor factorization corresponds to a certain shallow (depth two) non-linear convolutional network.
A drawback of these models as theoretical surrogates for deep learning is that the former lacks non-linearity, while the latter misses depth.
Hierarchical tensor factorization accounts for both of these limitations: for appropriate mode trees, it is known (see~\citet{cohen2016expressive}) to be equivalent to a class of deep non-linear convolutional networks (with multiplicative non-linearity).
These networks have demonstrated promising performance in practice~\cite{cohen2014simnets,cohen2016deep,sharir2016tensorial,stoudenmire2018learning,grant2018hierarchical,felser2021quantum}, and their equivalence to hierarchical tensor factorization has been key to the study of expressiveness in deep learning~\cite{cohen2016expressive,cohen2016convolutional,cohen2017inductive,cohen2017analysis,cohen2018boosting,sharir2018expressive,levine2018benefits,levine2018deep,balda2018tensor,khrulkov2018expressive,khrulkov2019generalized,levine2019quantum}.
The equivalence is illustrated in \fig~\ref{fig:tf_htf_as_convnet} (bottom) and rigorously proven in \app~\ref{app:htf_cnn}.

\subsection{Formal Presentation}
\label{sec:htf:formal}
The structure of a hierarchical tensor factorization is determined by a mode tree.
\begin{definition}
\label{def:mode_tree}
	Let $N \in \N$.
	A \emph{mode tree} $\T$ over $[N]$ is a rooted tree in which:
	\begin{itemize}[topsep=0pt, partopsep=0pt,itemsep=1pt]
		\item every node is labeled by a subset of $[N]$;
		\item there are exactly $N$ leaves, labeled $\{ 1 \}, \ldots, \{ N \}$; and
		\item the label of an interior (non-leaf) node is the union of the labels of its children.
	\end{itemize}
\end{definition}

\noindent
We identify nodes with their labels, \ie~with the corresponding subsets of~$[N]$, and accordingly treat~$\T$ as a subset of~$2^{[N]}$.
Furthermore, we denote the set of all interior nodes by $\interior (\T) \subset \T$, the parent of a non-root node $\nu \in \T \setminus \{ [N] \}$ by $\parent (\nu) \in \T$, and the children of $\nu \in \interior (\T)$ by $\children (\nu) \subset \T$.
When enumerating over children of a node, \ie~over $\children (\nu)$ for $\nu \in \interior (\htmodetree)$, an arbitrary fixed ordering is assumed.

One may consider various mode trees, each leading to a different hierarchical tensor factorization.
Notable choices include: \emph{(i)} a shallow tree (comprising only leaves and root), which reduces the hierarchical tensor factorization to a tensor factorization (\subsect~\ref{sec:prelim:tf}); and \emph{(ii)} a perfect binary tree (applicable if $N$ is a power of $2$) whose corresponding hierarchical tensor factorization is perhaps the most extensively studied.
\fig~\ref{fig:tf_htf_comp_loc_comp}(a) illustrates these two choices.

\begin{figure*}
	\vspace{-1mm}
	\begin{center}
		\includegraphics[width=0.9\textwidth]{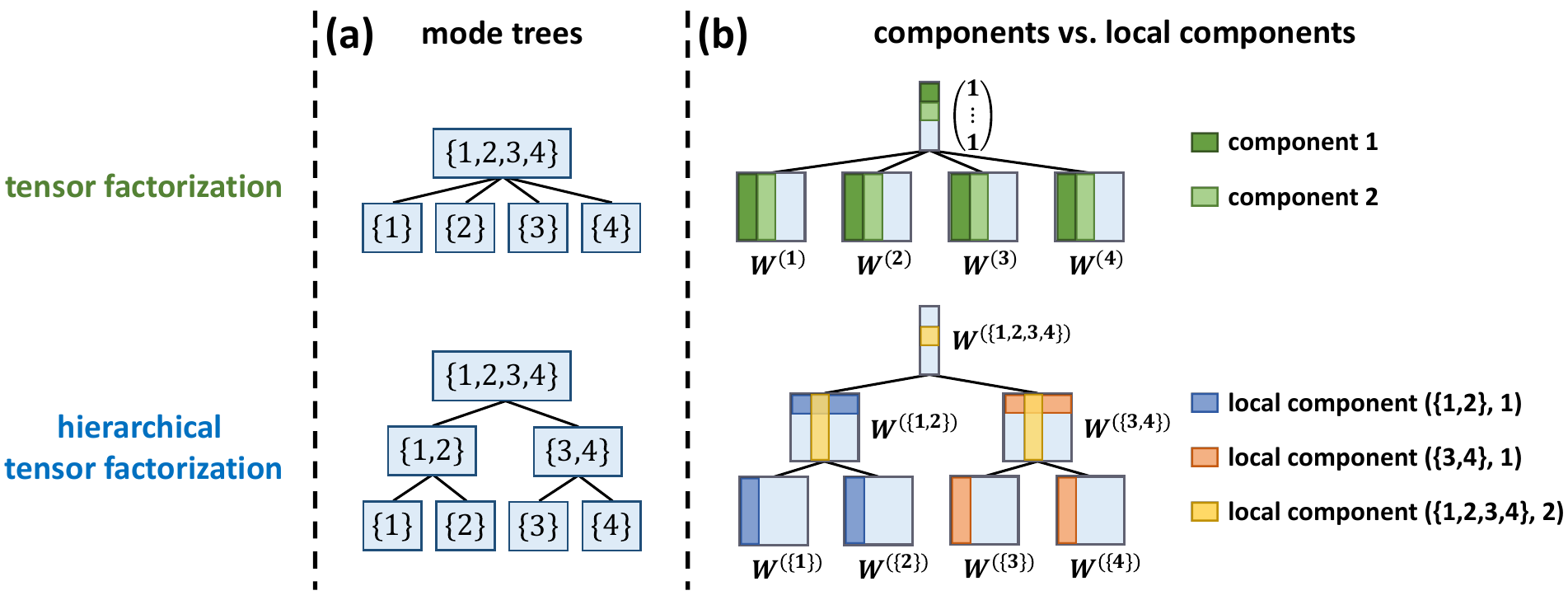}
	\end{center}
	\vspace{-3mm}
	\caption{
		\textbf{(a)} Exemplar mode trees (\defin~\ref{def:mode_tree}) for order $N = 4$ hierarchical tensor factorization.
		Top corresponds to the degenerate case of tensor factorization, while bottom represents the most common choice (perfect binary tree).
		\textbf{(b)} Components of a tensor factorization (top) vs. local components (\defin~\ref{def:local_comp}) of a hierarchical tensor factorization (bottom).
		The $r$'th component of a tensor factorization can be seen as the tensor product between a linear coefficient, which is set to~$1$, and the $r$'th columns of $\Wbf^{(1)}, \ldots, \Wbf^{(N)}$.
		The local components of a hierarchical tensor factorization are the components of the local tensor factorizations forming it. For example, the $r$'th local component at node $\{1, 2\}$ in the hierarchical tensor factorization illustrated above is the tensor product between the $r$'th row of $\weightmat{\{1, 2\}}$ and the $r$'th columns of its children's weight matrices $\weightmat{\{ 1 \}}$ and $\weightmat{\{ 2\}}$.
	}
	\label{fig:tf_htf_comp_loc_comp}
	\vspace{-1.5mm}
\end{figure*}

Along with a mode tree $\htmodetree$, what defines a hierarchical tensor factorization are the number of local components at each interior node, denoted $( R_{\nu} \in \N )_{\nu \in \interior ( \htmodetree )}$.
The induced hierarchical tensor factorization is parameterized by weight matrices $( \weightmat{\nu} \in \R^{R_\nu, R_{\parent (\nu)} } )_{ \nu \in \htmodetree}$, where $R_{\parent ( [N] )} := 1$ and $R_{\{1\}} := D_1, \ldots, R_{\{N\}} := D_N$.
It creates the \emph{end tensor} $\tensorend \in \R^{D_1, \ldots, D_N}$ by constructing intermediate tensors of increasing order while traversing $\htmodetree$ from leaves to root as follows:
\be
\begin{split}
	&\text{for all $\nu \in \left \{ \{1\}, \ldots, \{ N\} \right \}$ and $r \in [ R_{ \parent (\nu ) } ]$:} \\
	& \quad\underbrace{ \tensorpart{\nu}{r} }_{\text{order $1$}} := \weightmat{\nu}_{:, r} \text{\,,} \\[1mm]
	&\text{for all $\nu \in \interior (\htmodetree) \setminus \left \{ [N] \right \}$ and $r \in [ R_{\parent (\nu )} ]$ (traverse} \\
	& \text{interior nodes of $\htmodetree$ from leaves to root, non-inclusive):} \\
	&\quad \underbrace{\tensorpart{\nu}{r}}_{\text{order $\abs{\nu}$}} := \pi_\nu \biggl ( \sum\nolimits_{r' = 1}^{R_\nu} \weightmat{\nu}_{r', r} \Bigl [ \tenp_{\nu_c \in \children ( \nu )} \tensorpart{ \nu_c }{ r' } \Bigr ] \biggr ) \text{\,,} \\[0mm]
	&\underbrace{\tensorend}_{\text{order $N$}} := \pi_{ [N] } \biggl ( \sum\nolimits_{r' = 1}^{R_{ [N] }} \weightmat{ [N] }_{r', 1} \Bigl [ \tenp_{ \nu_c \in \children ( [N] )} \tensorpart{\nu_c}{r'} \Bigr ] \biggr ) 
	\text{\,,}
\end{split}
\label{eq:ht_end_tensor}
\ee
where $\pi_\nu$, for $\nu \in \htmodetree$, is a \emph{mode permutation} operator which arranges the modes (axes) of its input such that they comply with an ascending order of $\nu$.\footnote{
For $\nu \in \interior (\htmodetree)$, denote its $K := \abs{ \children (\nu) }$ children by $\nu_1, \ldots, \nu_K$, and the elements of $\nu_k$ by $j^k_1 < \cdots < j^k_{ \abs{\nu_k}}$ for $k \in [K]$.
Let $h : [\abs{\nu}] \to [\abs{\nu}]$ be the permutation sorting $\big ( j^1_1 , \ldots,  j^1_{ \abs{\nu_1} } ,  \ldots, j^{K}_1, \ldots, j^K_{ \abs{\nu_K} } \big )$ in ascending order.
Then, the mode permutation operator for $\nu$ is defined by: $\pi_\nu (\W)_{d_1, \ldots, d_{ \abs{\nu} } } = \W_{d_{h(1)}, \ldots, d_{h ( \abs{\nu} ) }}$, where $\W$ is an order $\abs{ \nu }$ tensor.
}

Hierarchical tensor factorization can be viewed as a composition of multiple local tensor factorizations, one for each interior node in the mode tree.
The local tensor factorization for $\nu \in \interior (\htmodetree)$ comprises $R_\nu$ components, referred to as the local components at node $\nu$ of the hierarchical tensor factorization~---~see \fig~\ref{fig:tf_htf_comp_loc_comp}(b) for an illustration, and \defin~\ref{def:local_comp} below.

\begin{definition}
	\label{def:local_comp}
	For $\nu \in \interior (\htmodetree)$ and $ r \in [R_\nu]$, the $(\nu, r)$'th \emph{local component} of the hierarchical tensor factorization is $\weightmat{\nu}_{r, :} \! \tenp \! \bigl ( \tenp_{\nu_c \in \children (\nu)} \weightmat{\nu_c}_{:, r} \bigr )$.
	We use $\localcomp (\nu, r)$ to denote the set comprising $\weightmat{\nu}_{r, :}$ and $\big ( \weightmat{\nu_c}_{:, r} \big )_{\nu_c \in \children ( \nu ) }$, and \smash{$\htfcompnorm{\nu}{r} := \normbig{ \tenp_{\wbf \in \localcomp (\nu, r)} \wbf }$} to denote the norm of the $(\nu, r)$'th local component.
\end{definition}

Mode trees of hierarchical tensor factorizations give rise to the notion of hierarchical tensor rank (\cf~\citet{grasedyck2013literature}), which is based on matrix ranks of specific \emph{matricizations} of a tensor (\cf~\subsect~3.4 in~\citet{kolda2006multilinear}).

\begin{definition}
\label{def:matricization}
The \emph{matricization} of $\W \in \R^{D_1, \ldots, D_N}$ with respect to $I \subset [N]$, denoted $\mat{\W}{I} \in \R^{\prod_{i \in I} D_i, \prod_{j \in [N] \setminus I} D_j}$, is its arrangement as a matrix where rows correspond to modes indexed by $I$ and columns correspond to the remaining modes.\footnote{
	Denoting the elements in $I$ by $i_1 < \cdots < i_{\abs{I}}$ and those in $[N] \setminus I$ by $j_1 < \cdots < j_{ N - \abs{I} }$, the matricization $\mat{\W}{I}$ holds the entries of $\W$ such that $\W_{d_1, \ldots, d_N}$ is placed in row index $1 + \sum_{l = 1}^{\abs{I}} (d_{i_{l}} - 1) \prod_{l' = 1}^{l - 1} D_{i_{l'}}$ and column index $1 + \sum_{l = 1}^{  N - \abs{I} } ( d_{ j_{l} } - 1 ) \prod_{l' = 1}^{l - 1} D_{ j_{l'} }$.
}
\end{definition}

\begin{definition}
	\label{def:ht_rank}
	The \emph{hierarchical tensor rank} of $\W \in \R^{D_1, \ldots, D_N}$ with respect to mode tree $\T$ is the tuple comprising the matrix ranks of $\W$'s matricizations according to all nodes in $\htmodetree$ except for the root, \ie~$( \rank \mat{ \W }{ \nu  } )_{\nu \in {\T \setminus \{ [N] \} }}$.
	The order of entries in the tuple does not matter as long as it is consistent.
\end{definition}
\noindent
Unless stated otherwise, when referring to the hierarchical tensor rank of a hierarchical tensor factorization's end tensor, the rank is with respect to the mode tree of the factorization. 
Hierarchical tensor rank differs markedly from tensor rank.
Specifically, even when the hierarchical tensor rank is low, \ie~the matrix ranks of matricizations according to all nodes in the mode tree are low, the tensor rank is typically extremely high (exponential in the order of the tensor~---~see~\citet{cohen2016deep}).

Lemma~\ref{lem:loc_comp_ht_rank_bound} below states that the number of local components in a hierarchical tensor factorization controls the hierarchical tensor rank of its end tensor.
More precisely, $R_\nu$~---~the number of local components at $\nu \in \interior (\htmodetree)$~---~upper bounds the matrix rank of matricizations according to the children of $\nu$.

\begin{lemma}[adaptation of \thm~7 in~\citet{cohen2018boosting}]
	\label{lem:loc_comp_ht_rank_bound}
	For any interior node $\nu \in \interior (\htmodetree)$ and child $\nu_c \in \children (\nu)$, it holds that $\rank \mat{\tensorend }{ \nu_c} \leq R_{\nu}$.
\end{lemma}
\begin{proof}
Deferred to \subapp~\ref{app:proofs:loc_comp_ht_rank_bound}.
\end{proof}
\vspace{-1mm}
\noindent
We may explicitly restrict the hierarchical tensor rank of end tensors $\tensorend$ (\eq~\eqref{eq:ht_end_tensor}) by limiting $( R_\nu )_{\nu \in \interior (\htmodetree)}$.
However, since our interest lies in the implicit regularization of gradient descent, \ie~in the types of end tensors it will find without explicit constraints, we treat the case where $( R_\nu )_{\nu \in \interior (\htmodetree)}$ can be arbitrarily large.

Given a differentiable and locally smooth loss $\htfendloss : \R^{D_1, \ldots, D_N} \to \R_{\geq 0}$, we consider parameterizing the solution as a hierarchical tensor factorization (\eq~\eqref{eq:ht_end_tensor}), and optimizing the resulting (non-convex) objective:
\be
\htfobj \big ( \big ( \weightmat{\nu} \big )_{\nu \in \htmodetree} \big ) := \htfendloss ( \tensorend )
\text{\,.}
\label{eq:htf_obj}
\ee
In line with analyses of implicit regularization in matrix and tensor factorizations (see \sect~\ref{sec:prelim}), we model small learning rate for gradient descent via gradient flow:
\be
\frac{d}{dt} \weightmat{\nu} (t) = - \frac{\partial}{ \partial \weightmat{\nu} } \htfobj \big (  \big ( \weightmat{\nu'} (t) \big )_{\nu' \in \htmodetree} \big )
\label{eq:gf_htf}
\ee
for all $t \geq 0$ and $\nu \in \htmodetree$, where $( \weightmat{\nu} (t) )_{\nu \in \htmodetree}$ denote the weight matrices at time $t$ of optimization.

Over matrix and tensor factorizations, gradient flow initialized near zero is known to minimize matrix and tensor ranks, respectively (\sect~\ref{sec:prelim}).
In particular, it leads to solutions that can be represented using few components.
A natural question that arises is whether a similar phenomenon takes place in hierarchical tensor factorization: does gradient flow with small initialization learn solutions that can be represented with few local components at all locations of the mode tree?
That is, does it learn solutions of low hierarchical tensor rank?
In \sect~\ref{sec:inc_rank_lrn} we answer this question affirmatively.
	
	\section{Analysis: Incremental Hierarchical Tensor Rank Learning}
\label{sec:inc_rank_lrn}

In this section we theoretically analyze the implicit regularization in hierarchical tensor factorization.
Our analysis extends known results for matrix and tensor factorizations outlined in \sect~\ref{sec:prelim}.
In particular, we show that the implicit regularization in hierarchical tensor factorization induces an incremental learning process that results in low hierarchical tensor rank, similarly to how matrix and tensor factorizations incrementally learn solutions with low matrix and tensor ranks, respectively.
To facilitate this extension, while overcoming the challenges arising from the complexity of the hierarchical tensor factorization model, we characterize the evolution of the local components introduced in \sect~\ref{sec:htf}.
Our analysis is delivered in \subsects~\ref{sec:inc_rank_lrn:evolution},~\ref{sec:inc_rank_lrn:imp_htr_min}, and~\ref{sec:inc_rank_lrn:partial_order}.
For the convenience of the reader, \subsect~\ref{sec:inc_rank_lrn:informal} provides an informal overview.

\subsection{Informal Overview}
\label{sec:inc_rank_lrn:informal}

As discussed in \sect~\ref{sec:prelim}, for both matrix and tensor factorizations, there exists an invariant of optimization whose deviation from zero is referred to as unbalancedness magnitude, and it is common to treat the case of unbalancedness magnitude zero as an idealization of standard near-zero initializations.
With unbalancedness magnitude zero, singular values in a matrix factorization evolve by \eq~\eqref{eq:sing_val_mf_dyn}, and component norms in a tensor factorization move per \eq~\eqref{eq:comp_norm_tf_dyn}.
Equations~\eqref{eq:sing_val_mf_dyn} and~\eqref{eq:comp_norm_tf_dyn} are structurally identical, and are interpreted as implying incremental learning of singular values and component norms, respectively, \ie~of matrix and tensor ranks, respectively.
This interpretation was initially supported by experiments (such as those reported in \fig~\ref{fig:mf_tf_htf_dynamics} (left and middle)), and later via proofs of exact matrix and tensor rank minimization under certain technical conditions.

In \subsect~\ref{sec:inc_rank_lrn:evolution} we show that in analogy with matrix and tensor factorizations, hierarchical tensor factorization entails an invariant of optimization (\lem~\ref{lem:loc_comp_sq_norm_diff_invariant}), which leads to a corresponding notion of unbalancedness magnitude (\defin~\ref{def:unbal_mag}).
For the canonical case of unbalancedness magnitude zero (corresponding to standard near-zero initializations), we prove that the norm of the $r$'th local component associated with node $\nu$ in the mode tree, denoted $\htfcompnorm{\nu}{r} (t)$, evolves by (\thm~\ref{thm:loc_comp_norm_bal_dyn}):
\be
\vspace{3.5mm}
\frac{d}{dt} \htfcompnorm{\nu}{r} \! (t) \! = \!  \htfcompnorm{\nu}{r} \! (t)^{2 - \frac{2}{ L_{\nu} } } L_\nu \inprodbig{- \nabla \htfendloss ( \tensorend (t) ) }{ \htfcomp{\nu}{r} (t) }
\text{,}
\label{eq:informal_loc_comp_norm_bal_dyn}
\ee
where $L_\nu$ is the number of weight vectors in the local component and $\htfcomp{\nu}{r} (t)$ is the direction it imposes on the end tensor $\tensorend (t)$.
\app~\ref{app:dyn_arbitrary} generalizes the above theorem by relieving the assumption of unbalancedness magnitude zero.
Namely, it establishes that \eq~\eqref{eq:informal_loc_comp_norm_bal_dyn} holds approximately when unbalancedness magnitude at initialization is small.
\eq~\eqref{eq:informal_loc_comp_norm_bal_dyn} is structurally identical to Equations~\eqref{eq:sing_val_mf_dyn} and~\eqref{eq:comp_norm_tf_dyn}, therefore the evolution rate of a local component norm in hierarchical tensor factorization mirrors the evolution rates of a singular value in matrix factorization and a component norm in tensor factorization.
One is thus led to interpret \eq~\eqref{eq:informal_loc_comp_norm_bal_dyn} as implying incremental learning of local component norms, \ie~of hierarchical tensor rank (see \sect~\ref{sec:htf}).
We support this interpretation through experiments analogous to those typically conducted for supporting the interpretation of Equations~\eqref{eq:sing_val_mf_dyn} and~\eqref{eq:comp_norm_tf_dyn} as implying incremental learning of matrix and tensor ranks, respectively~---~see \fig~\ref{fig:mf_tf_htf_dynamics} (right) as well as \app~\ref{app:experiments}.
Moreover, we consider technical conditions similar to those assumed for proving exact matrix and tensor rank minimization by matrix and tensor factorizations, respectively, and establish theoretical results aimed at facilitating a proof of exact hierarchical tensor rank minimization~---~see \subsect~\ref{sec:inc_rank_lrn:imp_htr_min}.
Completing the missing steps for deriving such a proof is regarded as a promising direction for future work.

Lastly, we discuss the fact that hierarchical tensor rank does not adhere to a natural total ordering, and the potential of partially ordered complexity measures to further our understanding of implicit regularization in deep learning.
See \subsect~\ref{sec:inc_rank_lrn:partial_order} for details.

\subsection{Evolution of Local Component Norms}
\label{sec:inc_rank_lrn:evolution}

\lem~\ref{lem:loc_comp_sq_norm_diff_invariant} below establishes an invariant of optimization: the differences between squared norms of weight vectors in the same local component are constant through time.
\begin{lemma}
	\label{lem:loc_comp_sq_norm_diff_invariant}
	For all $\nu \in \interior (\htmodetree)$, $r \in [R_\nu]$, and $\wbf, \wbf' \in \localcomp (\nu, r)$:
	\[
	\norm{ \wbf (t) }^2 - \norm{ \wbf' (t) }^2 = \norm{ \wbf (0) }^2 - \norm{ \wbf' (0) }^2 \quad , t \geq 0
	\text{\,.}
	\]
\end{lemma}
\begin{proof}[Proof sketch (proof in \subapp~\ref{app:proofs:loc_comp_sq_norm_diff_invariant})]
	A straightforward derivation shows that $\tfrac{d}{dt} \normnoflex{ \wbf (t) }^2 = \tfrac{d}{dt} \normnoflex{\wbf' (t) }^2$ for all $t \geq 0$.
	Integrating both sides with respect to time completes the proof.
\end{proof}
\noindent
The above invariant leads to the following definition of unbalancedness magnitude.
\begin{definition}
	\label{def:unbal_mag}
	The \emph{unbalancedness magnitude} of a hierarchical tensor factorization (\eq~\eqref{eq:ht_end_tensor}) is:
	\[
	\max\nolimits_{ \nu \in \interior (\htmodetree) , r \in [R_\nu] , \wbf , \wbf' \in \localcomp (\nu, r) } \abs{ \norm{\wbf }^2 - \norm{\wbf'}^2 }
	\text{\,.}
	\]
\end{definition}
\noindent
\lem~\ref{lem:loc_comp_sq_norm_diff_invariant} implies that the unbalancedness magnitude remains constant throughout optimization.
In the common regime of near-zero initialization, it will start off small, and stay small throughout.
The closer initialization is to zero, the smaller the unbalancedness magnitude is.
In accordance with analyses for matrix and tensor factorizations (see \sect~\ref{sec:prelim}), we treat the case of unbalancedness magnitude zero as an idealization of standard near-zero initializations.
\thm~\ref{thm:loc_comp_norm_bal_dyn} analyzes this case, characterizing the dynamics for norms of local components.

\begin{theorem}
	\label{thm:loc_comp_norm_bal_dyn}
	Assume unbalancedness magnitude zero at initialization.
	Let $\tensorend (t)$ denote the end tensor (\eq~\eqref{eq:ht_end_tensor}) and $\big ( \htfcompnorm{\nu}{r} (t) \big )_{\nu \in \interior (\htmodetree), r \in [R_\nu] }$ denote the norms of local components (\defin~\ref{def:local_comp}) at time $t \geq 0$ of optimization.
	Then, for any $\nu \in \interior (\htmodetree)$ and $r \in[R_\nu]$:
	\be
	\frac{d}{dt} \htfcompnorm{\nu}{r} \! (t) \! = \!  \htfcompnorm{\nu}{r} \! (t)^{2 - \frac{2}{ L_{\nu} } } L_\nu \inprodbig{- \nabla \htfendloss ( \tensorend (t) ) }{ \htfcomp{\nu}{r} (t) }
	\text{\,,}
	\label{eq:loc_comp_norm_bal_dyn}
	\ee
	where $L_\nu := \abs{ \children (\nu) } + 1$ is the number of weight vectors in a local component at node $\nu$, and $\htfcomp{\nu}{r} (t) \in \R^{D_1, \ldots, D_N}$ is the end tensor obtained by normalizing the $r$'th local component at node $\nu$ and setting all other local components at node $\nu$ to zero, \ie~by replacing  in \eq~\eqref{eq:ht_end_tensor} $\tensorpart{\nu}{r'}$ with $\pi_{\nu} \big ( \big ( \htfcompnorm{\nu}{r} \big )^{-1} \weightmat{\nu}_{r, r'} \big [  \tenp_{\nu_c \in \children (\nu)} \tensorpart{\nu_c}{r}  \big ] \big )$ for all $r' \in [ R_{\parent ( \nu )}]$.
	By convention, $\htfcomp{\nu}{r} (t) = 0$ if $\htfcompnorm{\nu}{r} (t) = 0$.
\end{theorem}
\begin{proof}[Proof sketch (proof in \subapp~\ref{app:proofs:loc_comp_norm_bal_dyn})]
	If $\htfcompnorm{\nu}{r} (t)$ is zero at some $t \geq 0$, then we show that it must be identically zero through time, leading both sides of \eq~\eqref{eq:loc_comp_norm_bal_dyn} to be equal to zero.
	Otherwise, differentiating the local component's norm with respect to time, we obtain:
	\[
	\begin{split}
		\tfrac{d}{dt} \htfcompnorm{\nu}{r} (t) & =
		\inprodbig{ - \nabla \htfendloss ( 	\tensorend (t)) }{ \htfcomp{\nu}{r} (t) } \\
		& \hspace{4mm}\cdot \sum\nolimits_{ \wbf \in \localcomp (\nu, r) } \prod\nolimits_{ \wbf' \in \localcomp (\nu, r) \setminus \{ \wbf \} } \norm{ \wbf' (t) }^2
		\text{\,.}
	\end{split}
	\]
	Since the unbalancedness magnitude is zero at initialization, \lem~\ref{lem:loc_comp_sq_norm_diff_invariant} implies that $\normnoflex{ \wbf (t) }^2 = \normnoflex{ \wbf' (t) }^2 = \htfcompnorm{\nu}{r} (t)^{2 / L_\nu}$ for all $\wbf, \wbf' \in \localcomp (\nu, r)$, which together with the expression above for $\tfrac{d}{dt} \htfcompnorm{\nu}{r} (t)$ establishes \eq~\eqref{eq:loc_comp_norm_bal_dyn}.
\end{proof}
\vspace{-0.5mm}

As can be seen from \eq~\eqref{eq:loc_comp_norm_bal_dyn}, the evolution of local component norms in a hierarchical tensor factorization is structurally identical to the evolution of singular values in matrix factorization (\eq~\eqref{eq:sing_val_mf_dyn}) and component norms in tensor factorization (\eq~\eqref{eq:comp_norm_tf_dyn}).
Specifically, it is dictated by two factors: a projection term, $\inprodbig{- \nabla \htfendloss ( \tensorend (t) ) }{ \htfcomp{\nu}{r} (t) }$, and a self-dependence term, $\htfcompnorm{\nu}{r} (t)^{2 - \frac{2}{ L_{\nu} } }  L_\nu$.
Analogous to a singular component $\mfcomp{r} (t)$ in matrix factorization and a normalized component $\tfcomp{r} (t)$ in tensor factorization, $\htfcomp{\nu}{r} (t)$ is the direction that the $(\nu, r)$'th local component imposes on $\tensorend (t)$.\footnote{
	Indeed, just as in matrix factorization $\matrixend = \sum_{r} \mfsing{r} \cdot \mfcomp{r}$, and in tensor factorization $\tftensorend = \sum_{r} \tfcompnorm{r} \cdot \tfcomp{r}$, the end tensor of a hierarchical tensor factorization decomposes as $\tensorend  = \sum_{r = 1}^{R_\nu} \htfcompnorm{\nu}{r} \cdot \htfcomp{\nu}{r}$ (implied by \lems~\ref{lem:ht_multilinear} and~\ref{lem:tensorend_comp_eq_zeroing_cols_or_rows} in \app~\ref{app:proofs}).
}
The projection of $\htfcomp{\nu}{r} (t)$ onto $- \nabla \htfendloss ( \tensorend (t) )$ therefore promotes growth of local components that align $\tensorend (t)$ with $- \nabla \htfendloss ( \tensorend (t) )$, the direction of steepest descent.
More critical is the self-dependence term, $\htfcompnorm{\nu}{r} (t)^{2 - \frac{2}{ L_{\nu} } } L_\nu$, which induces a momentum-like effect that attenuates the movement of small local components and accelerates the movement of large ones.
It suggests that, in analogy with matrix and tensor factorizations, local components tend to be learned incrementally, yielding a bias towards low hierarchical tensor rank.
This prospect is affirmed empirically in \fig~\ref{fig:mf_tf_htf_dynamics} (right) as well as \app~\ref{app:experiments}, and is supported theoretically in \subsect~\ref{sec:inc_rank_lrn:imp_htr_min}.

\vspace{-1.25mm}

\paragraph*{Evolution of local component norms under arbitrary initialization}
\thm~\ref{thm:loc_comp_norm_bal_dyn} can be extended to account for arbitrary initialization, \ie~for initialization with unbalancedness magnitude different from zero.
For conciseness we defer this extension to \app~\ref{app:dyn_arbitrary}, while noting that if initialization has small unbalancedness magnitude~---~as is the case with any near-zero initialization~---~then local component norms approximately evolve per \eq~\eqref{eq:loc_comp_norm_bal_dyn}, \ie~the result of \thm~\ref{thm:loc_comp_norm_bal_dyn} approximately holds.

\subsection{Implicit Hierarchical Tensor Rank Minimization}
\label{sec:inc_rank_lrn:imp_htr_min}

As discussed in \sect~\ref{sec:prelim}, under certain technical conditions, the incremental matrix and tensor rank learning phenomena, induced by the implicit regularization in matrix and tensor factorizations, can be used to prove exact matrix and tensor rank minimization, respectively.
Below we consider similar technical conditions, and provide theoretical results aimed at facilitating an analogous proof for hierarchical tensor factorization, \ie~a proof that its implicit regularization leads to exact hierarchical tensor rank minimization.
We begin by illustrating how, under said conditions, the incremental hierarchical tensor rank learning phenomenon established in \thm~\ref{thm:loc_comp_norm_bal_dyn} leads to solutions with many small local components (\subsect~\ref{sec:inc_rank_lrn:imp_htr_min:illustration}).
We then show that this implies proximity to low hierarchical tensor rank (\subsect~\ref{sec:inc_rank_lrn:imp_htr_min:proximity}).
Throughout the above, the main step missing in order to derive a complete proof of exact hierarchical tensor rank minimization, is confirmation that a certain alignment inequality (\eq~\eqref{eq:alignment_assump}) holds throughout optimization.
We regard this as an important direction for future work.

\subsubsection{Illustrative Demonstration of Small Local Components}
\label{sec:inc_rank_lrn:imp_htr_min:illustration}

Below we qualitatively demonstrate how the dynamical characterization derived in \subsect~\ref{sec:inc_rank_lrn:evolution} implies that the implicit regularization in hierarchical tensor factorization can lead to solutions with small local components.
Under the setting and notation of \thm~\ref{thm:loc_comp_norm_bal_dyn}, consider an initialization $\big (  \Ubf^{(\nu)}  \in \R^{R_\nu, R_{\parent (\nu)} } \big )_{\nu \in \htmodetree}$ for the weight matrices of the hierarchical tensor factorization, scaled by $\alpha \in \R_{ > 0}$.
That is, $\weightmat{\nu} (0) = \alpha \cdot \Ubf^{(\nu)}$ for all $\nu \in \htmodetree$.
Focusing on some interior node $\nu \in \interior (\htmodetree)$, let $r, \bar{r} \in [R_\nu]$, and assume for simplicity that $\nu$ is not degenerate, in the sense that~it has more than one child.
Suppose also that at initialization the norm of the $(\nu, r)$'th local component is greater than the norm of the $(\nu, \bar{r})$'th local component, \ie~$\htfcompnorm{\nu}{r} (0) > \htfcompnorm{\nu}{\bar{r}} (0)$, and that $\htfcomp{\nu}{r} (t)$ is at least as aligned as $\htfcomp{\nu}{\bar{r}} (t)$ with the direction of steepest descent up to a time $T > 0$, \ie~for all $t \in [0, T]$: 
\be
\inprodbig{- \nabla \htfendloss ( \tensorend (t) ) }{ \htfcomp{\nu}{r} (t) } \geq \inprodbig{- \nabla \htfendloss ( \tensorend (t) ) }{ \htfcomp{\nu}{\bar{r}} (t) }
\text{.}
\label{eq:alignment_assump}
\ee
Then, by \thm~\ref{thm:loc_comp_norm_bal_dyn} for all $t \in [0, T]$:\footnote{
	A local component cannot reach the origin unless it was initialized there (implied by \lem~\ref{lem:bal_zero_stays_zero} in \app~\ref{app:proofs}).
	Accordingly, we disregard the trivial case where $\htfcompnorm{\nu}{\bar{r}} (t) = 0$ for some $t \in [0, T]$.
}
\[
\htfcompnorm{\nu}{\bar{r}} (t)^{ - 2 + \frac{2}{L_\nu}} \frac{d}{dt} \htfcompnorm{\nu}{\bar{r}} (t) \leq  \htfcompnorm{\nu}{r} (t)^{- 2 + \frac{2}{L_\nu} } \frac{d}{dt} \htfcompnorm{\nu}{r} (t)
\text{\,.}
\]
Integrating both sides with respect to time, we may upper bound $\htfcompnorm{\nu}{\bar{r}} (t)$ with a function of $\htfcompnorm{\nu}{r} (t)$:
\be
\htfcompnorm{\nu}{\bar{r}} (t) \leq \Big [ \htfcompnorm{\nu}{r} (t)^{ - \frac{L_\nu - 2}{L_\nu}} + \alpha^{ - (L_\nu - 2) } \cdot const \Big ]^{ - \frac{L_\nu}{L_\nu - 2} }
\text{\,,}
\label{eq:illustrative_upper_bound_comp_norms}
\ee
where $const$ stands for a positive value that does not depend on $t$ and $\alpha$.
\eq~\eqref{eq:illustrative_upper_bound_comp_norms} reveals a gap between \smash{$\htfcompnorm{\nu}{r} (t)$} and \smash{$\htfcompnorm{\nu}{\bar{r}} (t)$} that is more significant the smaller the initialization scale $\alpha$ is.
In particular, regardless of how large $\htfcompnorm{\nu}{r} (t)$ is, $\htfcompnorm{\nu}{\bar{r}} (t)$ is upper bounded by a value that approaches zero as $\alpha \to 0$.
Hence, initializing near zero produces solutions with small local components.

\subsubsection{Small Local Components Imply Proximity to Low Hierarchical Tensor Rank}
\label{sec:inc_rank_lrn:imp_htr_min:proximity}

The following proposition establishes that small local components in a hierarchical tensor factorization imply that its end tensor can be well approximated with low hierarchical tensor rank.

\begin{proposition}
	\label{prop:low_rank_dist_bound}
	Consider an assignment for the weight matrices \smash{$\big ( \weightmat{\nu} \in \R^{R_\nu, R_{\parent (\nu)} }  \big )_{ \nu \in \htmodetree}$} of a hierarchical tensor factorization, and let $B := \max_{\nu \in \htmodetree} \normnoflex{ \weightmat{\nu}}$.
	Assume without loss of generality that at each $\nu \in \interior (\htmodetree)$, local components are ordered by their norms, \ie~$\htfcompnorm{\nu}{1} \geq \cdots \geq \htfcompnorm{\nu}{R_\nu}$.
	Then, for any $\epsilon \geq 0$ and $( R'_{\nu} \in \N )_{\nu \in \interior (\htmodetree)}$, if \,$\sum\nolimits_{r = R'_\nu + 1}^{R_\nu} \htfcompnorm{\nu}{r} \leq \epsilon \cdot (\abs{\htmodetree} - N)^{-1} B^{\abs{\children(\nu)} + 1 - \abs{\htmodetree}}$ for all $\nu \in \interior (\htmodetree)$, it holds that:
	\[
	\inf_{ \substack{ \W \in \R^{D_1, \ldots, D_N} \text{ s.t. } \\ \forall \nu \in \htmodetree \setminus \{ [N] \} :~\rank \mat{\W }{ \nu} \leq R'_{\parent (\nu) } } } \norm{ \tensorend - \W } \leq \epsilon
	\text{\,,}
	\]
	\ie~$\tensorend$ is within $\epsilon$-distance from the set of tensors whose hierarchical tensor rank is no greater (element-wise) than $\big ( R'_{\parent (\nu)} \big )_{\nu \in \htmodetree \setminus \{ [N] \}}$.
\end{proposition}
\begin{proof}[Proof sketch (proof in \subapp~\ref{app:proofs:low_rank_dist_bound})]
	Let $\widebar{\W}_{HT}^{\S}$ be the end tensor obtained after pruning all local components indexed by $\S := \{ (\nu, r) : \nu \in \interior (\htmodetree) , r \in \{ R'_{\nu} + 1, \ldots, R_\nu \} \}$, \ie~after setting to zero the $r$'th row of $\weightmat{\nu}$ and the $r$'th column of $\weightmat{\nu_c}$ for all $(\nu, r) \in \S$ and $\nu_c \in \children (\nu)$.
	The desired result follows by showing that $\rank \mat{ \widebar{\W}_{HT}^{\S} }{ \nu} \leq R'_{ \parent (\nu) }$ for all $\nu \in \htmodetree \setminus \{ [N] \}$, and upper bounding $\norm{ \tensorend - \widebar{\W}_{HT}^{\S} }$ by $\epsilon$.
\end{proof}

\subsection{Partially Ordered Complexity Measure}
\label{sec:inc_rank_lrn:partial_order}

Existing attempts to explain implicit regularization in deep learning typically argue for reduction of some complexity measure that is \emph{totally ordered} (meaning that within any two values for this measure, there must be one smaller than or equal to the other), for example a norm~\cite{gunasekar2017implicit,soudry2018implicit,li2018algorithmic,woodworth2020kernel,lyu2021gradient}.
Recent evidence suggests that obtaining a complete explanation through such complexity measures may not be possible \cite{razin2020implicit,vardi2021implicit}.
Hierarchical tensor rank (which we have shown to be implicitly reduced by a class of deep non-linear convolutional networks) represents a new type of complexity measure, in the sense that it is \emph{partially ordered}.  Specifically, while it entails a standard (product) partial order~---~$(r_1, \ldots, r_K) \leq (r'_1, \ldots, r'_K)$ if and only if $r_i \leq r'_i$ for all $i \in [K]$~---~it does not admit a natural total order.  
Indeed, \prop~\ref{prop:htr_multiple_minima} below shows that there exist simple learning problems in which, among the data-fitting solutions, there are multiple minimal hierarchical tensor ranks, none smaller than or equal to the other.  
We believe the notion of a partially ordered complexity measure may pave the way to furthering our understanding of implicit regularization in deep learning.

\begin{proposition}
	\label{prop:htr_multiple_minima}
	For every order $N \in \N_{\geq 3}$ and mode dimensions $D_1, \ldots, D_N \in \N_{\geq 2}$, there exists a tensor completion problem (\ie~a loss $\L ( \W ) = \frac{1}{\abs{\Omega}} \sum_{(d_1, \ldots, d_N) \in \Omega} ( \W_{d_1, \ldots, d_N} - \W^*_{d_1, \ldots, d_N})^2$ with ground truth $\W^* \in \R^{D_1, \ldots, D_N}$ and set of observed entries $\Omega \subset [D_1] \times \cdots \times [D_N]$) in which, for every mode tree $\htmodetree$ over $[N]$ (\defin~\ref{def:mode_tree}), the set of hierarchical tensor ranks for tensors fitting the observations includes multiple minimal elements (under the standard product partial order), none smaller than or equal to the other.
	That is, the set $\RR_\htmodetree := \left \{ (\rank \mat{\W }{ \nu})_{\nu \in \htmodetree \setminus \{ [N] \}} : \W \in \R^{D_1, \ldots, D_N} , \L ( \W ) = 0 \right \}$ includes elements $(R_\nu )_{\nu \in \htmodetree \setminus \{ [N]\}}$ and $(R'_\nu )_{\nu \in \htmodetree \setminus \{ [N]\}}$ for which the following hold: \emph{(i)} there exists no $(R''_\nu )_{\nu \in \htmodetree \setminus \{ [N]\}} \in \RR_\htmodetree \setminus \{ (R_{\nu} )_{\nu  \in \htmodetree \setminus \{ [N] \}} , (R'_{\nu } )_{\nu \in \htmodetree \setminus \{ [N]\} } \}$ satisfying $(R''_\nu )_{\nu} \leq (R_\nu )_{\nu}$ or $(R''_\nu )_{\nu} \leq (R'_\nu )_{\nu}$; and \emph{(ii)} neither $(R_\nu )_{\nu} \leq (R'_\nu )_{\nu}$ nor $(R'_\nu )_{\nu} \leq (R_\nu )_{\nu}$.
\end{proposition}
\begin{proof}[Proof sketch (proof in \subapp~\ref{app:proofs:htr_multiple_minima})]
	We construct a tensor completion problem and two solutions $\W$ and $\W'$ (tensors fitting observed entries) such that, for every mode tree $\htmodetree$, the hierarchical tensor ranks with respect to $\htmodetree$ of $\W$ and $\W'$ are two different minimal elements of $\RR_\htmodetree$.
\end{proof}

	\section{Low Hierarchical Tensor Rank Implies Locality} 
\label{sec:low_htr_implies_locality}

In \sect~\ref{sec:inc_rank_lrn} we established that the implicit regularization in hierarchical tensor factorization favors solutions with low hierarchical tensor rank.
A natural question that arises is what are the implications of this tendency for the class of deep convolutional networks equivalent to hierarchical tensor factorization (illustrated in \fig~\ref{fig:tf_htf_as_convnet} (bottom)).
It is known~\cite{cohen2017inductive,levine2018benefits,levine2018deep} that for this class of networks, hierarchical tensor rank measures the strength of dependencies modeled between spatially distant input regions (patches of pixels in the context of image classification)~---~see brief explanation in \subsect~\ref{sec:low_htr_implies_locality:sep_rank} below, and formal derivation in \app~\ref{app:ht_sep_rank}.
An implicit regularization towards low hierarchical tensor rank thus implies a bias towards local (short-range) dependencies.
While seemingly benign, this observation is shown in \sect~\ref{sec:countering_locality} to bring forth a practical method for improving performance of contemporary convolutional networks (\eg~ResNet18 and ResNet34 from~\citet{he2016deep}) on tasks with long-range dependencies.

\subsection{Locality via Separation Rank}
\label{sec:low_htr_implies_locality:sep_rank}

Given a multivariate function $f$ with scalar output, a popular measure of dependencies between a set of input variables and its complement is known as \emph{separation rank}.  
The separation rank, formally presented in \defin~\ref{def:sep_rank} below, was originally introduced in~\citet{beylkin2002numerical}, and has since been employed for various applications~\cite{harrison2003multiresolution,hackbusch2006efficient,beylkin2009multivariate}, as well as analyses of expressiveness in deep learning~\cite{cohen2017inductive,cohen2017analysis,levine2018benefits,levine2018deep,levine2020limits,wies2021transformer,levine2022inductive}.
It is also prevalent in quantum mechanics, where it serves as a measure of entanglement~\cite{levine2018deep}.

Consider the convolutional network equivalent to a hierarchical tensor factorization with mode tree $\htmodetree$.
It turns out (see formal derivation in \app~\ref{app:ht_sep_rank}) that for functions realized by this network, separation ranks measuring dependencies between distinct regions of the input are precisely equal to entries of the hierarchical tensor rank with respect to $\htmodetree$ (recall that, as discussed in \sect~\ref{sec:htf}, the hierarchical tensor rank is a tuple).  
Thus, low hierarchical tensor rank implies that the separation ranks are low, which in turn means that dependencies modeled between distinct input regions are weak, \ie~that only local dependencies are prominent.

\begin{definition}
\label{def:sep_rank}
The \emph{separation rank} of $f : \times_{n = 1}^N \R^{D_n} \to \R$ with respect to $I \subset [N]$, denoted $\seprank (f; I)$, is the minimal $R \in \N \cup \{ 0 \}$ for which there exist $g_1, \ldots, g_R : \times_{i \in I} \R^{D_i} \to \R$ and $\bar{g}_1, \ldots, \bar{g}_R :\times_{j \in [N] \setminus I} \R^{D_j} \to \R$ such that:
\[
f \bigl ( \xbf^{(1)}, \ldots, \xbf^{(N)} \bigr ) \! = \! \sum_{r = 1}^R g_r \big ( \big ( \xbf^{(i)} \big )_{i \in I} \big ) \cdot \bar{g}_r \bigl ( \bigl ( \xbf^{(j)} \bigr )_{j \in [N] \setminus I } \bigr )
\text{.}
\]
\end{definition}

\paragraph*{Interpretation}
The separation rank of~$f$ with respect to~$I$ is the minimal number of summands required to express~$f$, where each summand is a product of two functions~---~one that operates over variables indexed by~$I$, and another that operates over the remaining variables.
If $\seprank (f; I) = 1$, the function is separable, meaning it does not model any interaction between the sets of variables.
In a statistical setting, where $f$ is a probability density function, this would mean that $( \xbf^{(i)} )_{i \in I}$ and $( \xbf^{(j)} )_{j \in [N] \setminus I}$ are statistically independent.
The higher $\seprank (f; I)$ is, the farther $f$ is from separability, \ie~the stronger the dependencies it models between $( \xbf^{(i)} )_{i \in I}$ and $( \xbf^{(j)} )_{j \in [N] \setminus I}$.

	\section{Countering Locality of Convolutional Networks via Regularization} 
\label{sec:countering_locality}

\begin{figure*}[t]
	\vspace{-1mm}
	\begin{center}
			\hspace{-5mm}
			\includegraphics[width=0.84\textwidth]{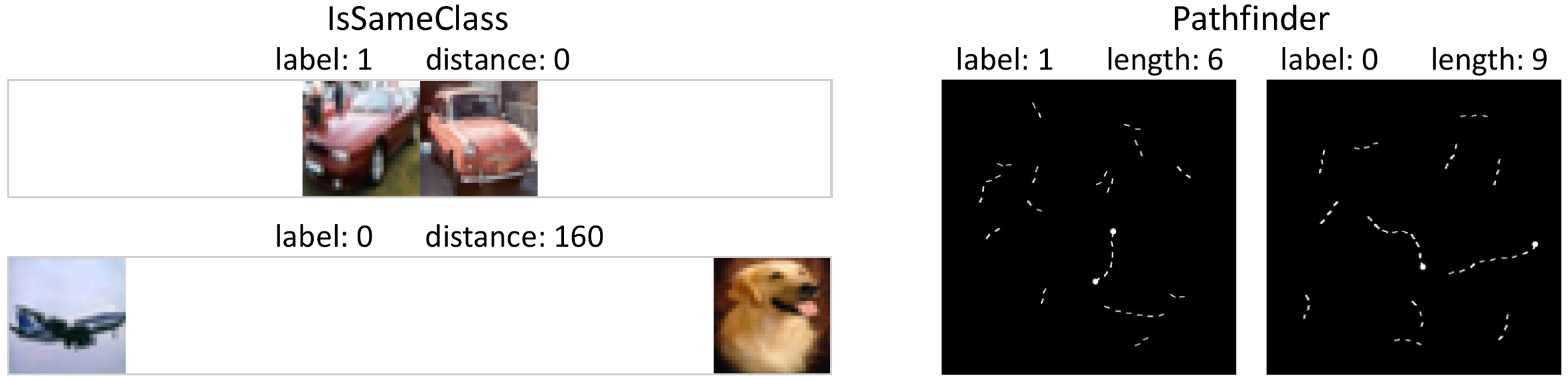}
	\end{center}
	\vspace{-2mm}
	\caption{
		Samples from IsSameClass and Pathfinder datasets.
		For further details on their creation process see \subapp~\ref{app:experiments:details:conv}.
		\textbf{Left:} positive and negative samples from IsSameClass datasets with $0$ and $160$ pixels between images, respectively.
		The label is $1$ if the two CIFAR10 images are of the same class, and $0$ otherwise.
		For the sake of illustration, background is displayed as white instead of black, and padding is not shown (\ie~only the raw $32 \times 224$ input is presented).
		\textbf{Right:} positive and negative samples from Pathfinder challenge~\cite{linsley2018learning} datasets with connecting path lengths $6$ and $9$, respectively. 
		A connecting path is one that joins the two circles, and if present, its length is measured in the number of dashes.
		The label of a sample is $1$ if it includes a connecting path (\ie~if the two circles are connected), and $0$ otherwise.
	}
	\label{fig:long_range_datasets_samples}
	\vspace{-2mm}
\end{figure*}

\begin{figure*}[t]
	\vspace{0.5mm}
	\begin{center}
		\hspace{-3mm}
		\includegraphics[width=0.855\textwidth]{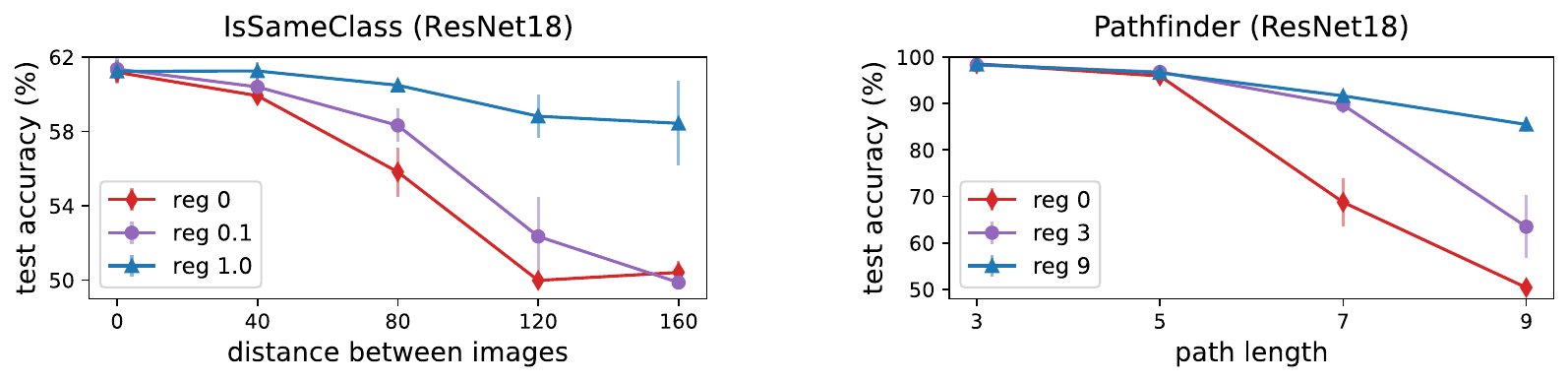}
	\end{center}
	\vspace{-4mm}
	\caption{
		Dedicated explicit regularization can counter the locality of convolutional networks, significantly improving performance on tasks with long-range dependencies.
		Plots present test accuracies achieved by a randomly initialized ResNet18 over IsSameClass (left) and Pathfinder (right) datasets, with varying spatial distances between salient regions of the input (CIFAR10 images in IsSameClass and connected circles in Pathfinder --- see \fig~\ref{fig:long_range_datasets_samples}).
		For each dataset, the network was trained via stochastic gradient descent to minimize a regularized objective, consisting of the binary cross-entropy loss and the dedicated regularization described in \subsect~\ref{sec:countering_locality:reg}.
		The legend specifies the regularization coefficients used.
		Markers and error bars report means and standard deviations, respectively, taken over five different runs for the corresponding combination of dataset and regularization coefficient.
		As expected, when increasing the (spatial) range of dependencies required to be modeled, the test accuracy obtained by an unregularized network (regularization coefficient zero) substantially deteriorates, reaching the vicinity of the trivial value $50\%$.
		Conventional wisdom attributes this failure to a limitation in the expressive capability of convolutional networks (\ie~to their inability to represent functions modeling long-range dependencies). 
		However, as can be seen, applying the dedicated regularization significantly improved performance, without any architectural modification.
		\app~\ref{app:experiments} provides further implementation details, as well as additional experiments: \emph{(i)} using ResNet34; and \emph{(ii)} showing similar improvements when the baseline network (“reg $0$'') is already regularized via standard techniques (weight decay or dropout).
	}
	\label{fig:long_range_and_reg_results}
	\vspace{-1.5mm}
\end{figure*}

Convolutional networks often struggle or completely fail to learn tasks that entail strong dependence between spatially distant regions of the input (patches of pixels in image classification or tokens in natural language processing tasks)~---~see, \eg,~\citet{wang2016temporal,linsley2018learning,mlynarski2019convolutional,hong2020graph, kim2020disentangling}.
Conventional wisdom attributes this failure to the local nature of the architecture, \ie~to its inability to express long-range dependencies (see, \eg,~\citet{cohen2017inductive,linsley2018learning,kim2020disentangling}).
This suggests that addressing the problem requires modifying the architecture.
Our theory reveals that there is also an implicit regularization at play, giving rise to the possibility of countering the locality of convolutional networks via \emph{explicit} regularization, without modifying their architecture.
In the current section we affirm this possibility, demonstrating that carefully designed regularization can greatly improve the performance of contemporary convolutional networks on tasks involving long-range dependencies.
For brevity, we defer some implementation details and experiments to \app~\ref{app:experiments}.

We conducted a series of experiments, using the ubiquitous ResNet18 and ResNet34 convolutional networks~\cite{he2016deep}, over two types of image classification datasets in which the distance between salient regions can be controlled.
The first type, referred to as "IsSameClass," comprises datasets we constructed, where the goal is to predict whether two randomly sampled CIFAR10~\cite{krizhevsky2009learning} images are of the same class.
Each input sample is a $32 \times 224$ image filled with zeros, in which the CIFAR10 images are placed (symmetrically around the center) at a predetermined distance from each other (to comply with ResNets, inputs were padded to have size $224 \times 224$).
By increasing the predetermined distance between CIFAR10 images, we produce datasets requiring stronger modeling of long-range dependencies.
The second type of datasets is taken from the Pathfinder challenge~\cite{linsley2018learning,kim2020disentangling,tay2021long}~---~a standard benchmark for modeling long-range dependencies.
In Pathfinder, each image contains two white circles and multiple dashed paths (curves) over a black background, and the goal is to predict whether the circles are connected by a path.  
The length of connecting paths is predetermined, allowing control over the (spatial) range of dependencies necessary to model.
Representative examples from IsSameClass and Pathfinder datasets are displayed in \fig~\ref{fig:long_range_datasets_samples}.

\fig~\ref{fig:long_range_and_reg_results} shows that when fitting IsSameClass and Pathfinder datasets, increasing the strength of long-range dependencies (\ie~the distance between images in IsSameClass, and the connecting path length in Pathfinder) leads to significant degradation in test accuracy, oftentimes resulting in performance no better than random guessing.
This phenomenon complies with existing evidence from~\citet{linsley2018learning,kim2020disentangling} showing failure of convolutional networks in learning tasks with long-range dependencies.
However, while \citet{linsley2018learning,kim2020disentangling} address the problem by modifying the architecture, we tackle it through explicit regularization (described in \subsect~\ref{sec:countering_locality:reg} below) designed to promote high separation ranks (\defin~\ref{def:sep_rank}), \ie~long-range dependencies between image regions.
As evident in \fig~\ref{fig:long_range_and_reg_results}, our regularization significantly improves test accuracy.
This implies that the tendency towards locality of modern convolutional networks may in large part be due to implicit regularization, and not an inherent limitation of expressive power as often believed.
Our findings showcase that deep learning architectures considered suboptimal for certain tasks may be greatly improved through a right choice of explicit regularization.
Theoretical understanding of implicit regularization may be key to discovering such regularizers.

\subsection{Explicit Regularization Promoting Long-Range Dependencies}
\label{sec:countering_locality:reg}

We describe below the explicit regularization applied in our experiments to counter the locality of convolutional networks.
We emphasize that this regularization is based on our theory, and merely serves as an example to how the performance of convolutional networks on tasks involving long-range dependencies can be improved without modifying their architecture.
Further evaluation and improvement of our regularization are regarded as promising directions for future work.

Denote by $f_\Theta ( \Xbf )$ the output of a neural network, where $\Theta$ stands for its learnable weights, and $\Xbf := ( \xbf^{(1)}, \ldots, \xbf^{(N)} )$ represents an input image, with each $\xbf^{(n)}$ standing for a pixel.
Suppose we are given a subset of indices $I \subset [N]$, with complement $J := [N] \setminus I$, and we would like to encourage the network to learn a function $f_\Theta$ that models strong dependence between $\Xbf_I := ( \xbf^{(i)} )_{i \in I}$ (pixels indexed by $I$) and $\Xbf_J := ( \xbf^{(j)} )_{j \in J}$ (those indexed by~$J$).
As discussed in \subsect~\ref{sec:low_htr_implies_locality:sep_rank}, a standard measure of such dependence is the separation rank, provided in \defin~\ref{def:sep_rank}.
If the separation rank of $f_\Theta$ with respect to $I$ is one, meaning no dependence between $\Xbf_I$ and~$\Xbf_J$ is modeled, then we may write $f_\Theta ( \Xbf ) = g (\Xbf_I) \cdot \bar{g} (\Xbf_J)$ for some functions $g$ and $\bar{g}$.
This implies that $\nabla_{\Xbf_I} f_\Theta ( \Xbf ) = \bar{g} (\Xbf_{J}) \cdot \nabla g (\Xbf_I)$, meaning that a change in $\Xbf_J$ (with $\Xbf_I$ held fixed) does not affect the direction of $\nabla_{\Xbf_I} f_\Theta (\Xbf)$, only its magnitude (and possibly its sign).
This observation suggests that, in order to learn a function $f_\Theta$ modeling strong dependence between $\Xbf_I$ and $\Xbf_J$, one may add a regularization term that promotes a change in the direction of $\nabla_{\Xbf_I} f_\Theta (\Xbf)$ whenever $\Xbf_J$ is altered (with $\Xbf_I$ held fixed).

The regularization applied in our experiments is of the type outlined above, with $I$ and $J$ chosen to promote long-range dependencies.
Namely, at each iteration of stochastic gradient descent we randomly choose disjoint subsets of indices $I$ and $J$ corresponding to contiguous (distinct) image regions.
Then, for each image $\Xbf$ in the iteration's batch, we let $\Xbf'$ be the result of replacing the pixels in $\Xbf$ indexed by $J$ with alternative values taken from a different image in the training set.
Finally, we compute $\abs{ \inprod{ \nabla_{\Xbf_I} f_\Theta (\Xbf) }{ \nabla_{\Xbf_I} f_\Theta (\Xbf') } } \cdot  \normnoflex{ \nabla_{\Xbf_I} f_\Theta (\Xbf) }^{-1}  \normnoflex{ \nabla_{\Xbf_I} f_\Theta (\Xbf') }^{-1}$~---~(absolute value of) cosine of the angle between $\nabla_{\Xbf_I} f_\Theta (\Xbf)$ and $\nabla_{\Xbf_I} f_\Theta (\Xbf')$~---~average it across the batch, multiply the average by a constant coefficient, and add the result to the minimized objective.\footnote{
	Each artificially generated image $\Xbf'$ is used only to compute the regularization term, not as an additional training instance incurring its own loss.
	Our proposed regularization is therefore fundamentally different from data augmentation.
}
For further details see \subapp~\ref{app:experiments:details:conv}.

	\section{Related Work} \label{sec:related}

A large and growing body of literature has theoretically investigated the implicit regularization brought forth by gradient-based optimization.
Works along this line have treated various models, including: linear predictors~\cite{soudry2018implicit,gunasekar2018implicit,nacson2019convergence,ji2019implicit,shachaf2021theoretical}; polynomially parameterized linear models with a single output~\cite{ji2019gradient,woodworth2020kernel,moroshko2020implicit,azulay2021implicit,haochen2021shape,pesme2021implicit,li2021happens,chou2021more}; shallow non-linear neural networks~\cite{hu2020surprising,vardi2021implicit,sarussi2021towards,mulayoff2021implicit,lyu2021gradient}; homogeneous networks~\cite{lyu2020gradient,vardi2021margin}; and ultra-wide networks~\cite{oymak2019overparameterized,chizat2020implicit}.
Arguably the most widely analyzed model is matrix factorization, whose study was extended to tensor factorization~\cite{gunasekar2017implicit,du2018algorithmic,li2018algorithmic,arora2019implicit,gidel2019implicit,mulayoff2020unique,blanc2020implicit,gissin2020implicit,razin2020implicit,chou2020gradient,eftekhari2021implicit,yun2021unifying,min2021explicit,li2021towards,razin2021implicit,milanesi2021implicit,ge2021understanding}.
Our work generalizes existing results for matrix and tensor factorizations (see \sect~\ref{sec:prelim}) to hierarchical tensor factorization~---~a considerably richer and more complex model.

Hierarchical tensor factorization was originally introduced in~\citet{hackbusch2009new}.
By virtue of its equivalence to different types of (non-linear) neural networks, it has been paramount to the study of expressiveness in deep learning~\cite{cohen2016expressive,sharir2016tensorial,cohen2016convolutional,cohen2017inductive,cohen2017analysis,sharir2018expressive,cohen2018boosting,levine2018benefits,levine2018deep,balda2018tensor,khrulkov2018expressive,khrulkov2019generalized,levine2019quantum}.
It is also used in different contexts, for example recovery of low (hierarchical tensor) rank tensors~\cite{da2015optimization,steinlechner2016riemannian,rauhut2017low,kargas2020nonlinear,kargas2021supervised}. 
To the best of our knowledge, this paper is the first to study implicit regularization of gradient-based optimization over hierarchical tensor factorization.

With regards to convolutional networks, theoretical investigations of their implicit regularization are scarce.
Existing works in this category treat linear~\cite{gunasekar2018implicit,jagadeesan2021inductive,kohn2021geometry} and homogeneous~\cite{nacson2019lexicographic,lyu2020gradient,ji2020directional} models.\footnote{
There have also been works studying implicit effects of explicit regularizers for convolutional networks~\cite{ergen2021implicit}, but these are outside the scope of our paper.
} None of these works have pointed out an implicit regularization towards local dependencies, as our theory does (\sects~\ref{sec:inc_rank_lrn} and \ref{sec:low_htr_implies_locality}).
Although the locality of convolutional networks is widely accepted, it is typically ascribed to expressive properties determined by their architecture (see, \eg,~\citet{cohen2017inductive,linsley2018learning,kim2020disentangling}).
Our work is the first to indicate that it also originates from implicit regularization.
As we demonstrate in \sect~\ref{sec:countering_locality}, this observation can have far reaching implications to the performance of convolutional networks in practice.
	
	\section{Summary}
\label{sec:summary}

Incremental matrix rank learning in matrix factorization~\cite{arora2019implicit,gidel2019implicit,gissin2020implicit,chou2020gradient,li2021towards} and incremental tensor rank learning in tensor factorization~\cite{razin2021implicit,ge2021understanding} were important discoveries on the path to explaining implicit regularization in deep learning.
The current paper takes an additional step along this path, establishing incremental hierarchical tensor rank learning in hierarchical tensor factorization.
It circumvents the complexity of the hierarchical tensor factorization model by introducing the notion of local components, and theoretically analyzing their evolution throughout optimization.
Experiments validate the theory.

While matrix factorization corresponds to linear neural networks and tensor factorization to certain shallow (depth two) non-linear convolutional neural networks~\cite{cohen2016expressive,razin2021implicit}, hierarchical tensor factorization is equivalent to a class of \emph{deep} non-linear convolutional neural networks~\cite{cohen2016expressive}.
It therefore jointly accounts for both non-linearity and depth~---~two critical aspects of deep learning.
For the class of convolutional networks equivalent to hierarchical tensor factorization, low hierarchical tensor rank translates to weak modeling of dependencies between spatially distant input regions~\cite{cohen2017inductive,levine2018benefits,levine2018deep}.
Our theory thus suggests an implicit regularization towards locality in convolutional networks.
While the locality of convolutional networks is widely accepted, it is typically ascribed to expressive properties determined by their architecture (see, \eg,~\citet{cohen2017inductive,linsley2018learning,kim2020disentangling}).
The fact that implicit regularization also plays a role indicates that it might be possible to counter this locality via explicit regularization. 
We verify this prospect empirically, demonstrating that explicit regularization designed to promote high hierarchical tensor rank vastly improves the performance of modern convolutional networks (ResNet18 and ResNet34 from~\citet{he2016deep}) on tasks with long-range dependencies.

Taken together, the theory and experiments presented in this paper bring forth the possibility that deep learning architectures considered suboptimal for certain tasks (\eg~convolutional networks for natural language processing tasks) may be greatly improved through a right choice of explicit regularization.
Theoretical understanding of implicit regularization may be key to discovering such regularizers.

	\ifdefined\NEURIPS
		\begin{ack}
			This work was supported by a Google Research Scholar Award, a Google Research Gift, the Yandex Initiative in Machine Learning, the Israel Science Foundation (grant 1780/21), Len Blavatnik and the Blavatnik Family Foundation, and Amnon and Anat Shashua.
NR is supported by the Apple Scholars in AI/ML and the Tel Aviv University Center for AI and Data Science (TAD) PhD fellowships.
		\end{ack}
	\else
		\newcommand{\ack}{This work was supported by a Google Research Scholar Award, a Google Research Gift, the Yandex Initiative in Machine Learning, the Israel Science Foundation (grant 1780/21), Len Blavatnik and the Blavatnik Family Foundation, and Amnon and Anat Shashua.
NR is supported by the Apple Scholars in AI/ML and the Tel Aviv University Center for AI and Data Science (TAD) PhD fellowships.}
	\fi
	\ifdefined\ARXIV
		\section*{Acknowledgments}
		This work was supported by a Google Research Scholar Award, a Google Research Gift, the Yandex Initiative in Machine Learning, the Israel Science Foundation (grant 1780/21), Len Blavatnik and the Blavatnik Family Foundation, and Amnon and Anat Shashua.
NR is supported by the Apple Scholars in AI/ML and the Tel Aviv University Center for AI and Data Science (TAD) PhD fellowships.
	\else
		\ifdefined\COLT
			\acks{This work was supported by a Google Research Scholar Award, a Google Research Gift, the Yandex Initiative in Machine Learning, the Israel Science Foundation (grant 1780/21), Len Blavatnik and the Blavatnik Family Foundation, and Amnon and Anat Shashua.
NR is supported by the Apple Scholars in AI/ML and the Tel Aviv University Center for AI and Data Science (TAD) PhD fellowships.}
		\else
			\ifdefined\CAMREADY
				\ifdefined\ICLR
					\newcommand*{\subsuback}{}
				\fi
				\ifdefined\NEURIPS
				\else
					\section*{Acknowledgments}
					This work was supported by a Google Research Scholar Award, a Google Research Gift, the Yandex Initiative in Machine Learning, the Israel Science Foundation (grant 1780/21), Len Blavatnik and the Blavatnik Family Foundation, and Amnon and Anat Shashua.
NR is supported by the Apple Scholars in AI/ML and the Tel Aviv University Center for AI and Data Science (TAD) PhD fellowships.
				\fi
			\fi
		\fi
	\fi

	\section*{References}
	{\small
		\ifdefined\ICML
			\bibliographystyle{icml2022}
		\else
			\bibliographystyle{plainnat}
		\fi
		\bibliography{refs}
	}

	\clearpage
	\appendix
	
	\onecolumn
	
	\ifdefined\ENABLEENDNOTES
		\theendnotes
	\fi
	

	\section{Hierarchical Tensor Factorization as Deep Non-Linear Convolutional Network}
\label{app:htf_cnn}

\begingroup
\setlength{\columnsep}{17pt}
\begin{wrapfigure}{r}{0.36\textwidth}
	\vspace{-4.6mm}
	\hspace*{-0.5mm}
	\includegraphics[width=0.31\textwidth]{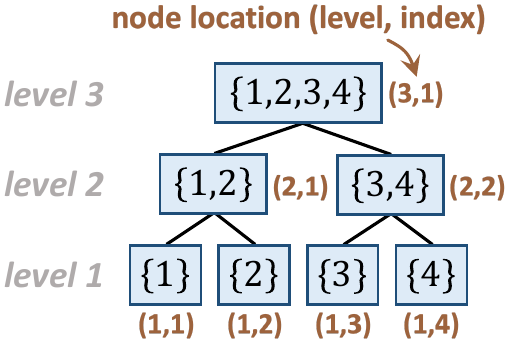}
	\vspace{-1mm}
	\caption{
		Perfect $P$-ary mode tree that combines adjacent indices, for order $N = 4$ and $P = 2$.
	}
	\vspace{-2mm}
	\label{fig:pary_mode_tree_with_locs}
	\vspace{-3mm}
\end{wrapfigure}
In this appendix, we formally state and prove a known correspondence between hierarchical tensor factorization and certain deep non-linear convolutional networks (\cf~\citet{cohen2016expressive}).
For conciseness, we assume the tensor order $N$ is a power of $P \in \N_{\geq 2}$ and the mode dimensions $D_1, \ldots, D_N$ are equal, and focus on the factorization induced by a perfect $P$-ary mode tree (\defin~\ref{def:mode_tree}) that combines nodes with adjacent indices.

Let $L := \log_P N$ denote the height of the mode tree, and associate each of its nodes with a respective location $(l, n)$, where $l \in [L + 1]$ is the level in the tree (numbered from leaves to root in ascending order), and $n \in [ N / P^{l - 1} ]$ is the index inside the level (see \fig~\ref{fig:pary_mode_tree_with_locs} for an illustration).
Adapting \eq~\eqref{eq:ht_end_tensor} to the current setting, the end tensor is computed as follows:
\be
\begin{split}
	&\text{for all $n \in [N]$ and $r \in [ R_{ 1 } ]$:} \\
	& \quad\underbrace{ \tensorpart{1, n}{r} }_{\text{order $1$}} := \weightmat{1, n}_{:, r} \text{\,,} \\[1mm]
	&\text{for all $l \in \{ 2, \ldots, L \}, n \in [ N / P^{l - 1}]$, and $r \in [ R_{ l } ]$ (traverse interior nodes of $\htmodetree$ from leaves to root, non-inclusive):} \\
	&\quad \underbrace{\tensorpart{l, n}{r}}_{\text{order $P^{l - 1}$}} := \sum\nolimits_{r' = 1}^{R_{l - 1}} \weightmat{l, n}_{r', r} \left [ \tenp_{p = (n - 1) \cdot P + 1}^{n \cdot P} \tensorpart{l - 1, p }{r'} \right ] \text{\,,} \\[0mm]
	&\underbrace{\tensorend}_{\text{order $N$}} := \sum\nolimits_{r' = 1}^{R_{L}} \weightmat{L + 1, 1}_{r', 1} \left [ \tenp_{p = 1}^{P} \tensorpart{L, p }{r'} \right ]
	\text{\,,}
\end{split}
\label{eq:ht_pary_end_tensor}
\ee
where $\big ( \weightmat{l, n} \in \R^{R_{l - 1}, R_{l}} \big )_{l \in [L + 1], n \in [N / P^{l - 1}]}$ are the factorization's weight matrices, $R_{L + 1} = 1$, and $R_{0} := D_1 = \cdots =~D_N$.

The deep non-linear convolutional network corresponding to the above factorization (illustrated in \fig~\ref{fig:tf_htf_as_convnet} (bottom)) has $L$ hidden layers, the $l$'th one comprising a locally connected linear operator with $R_{l}$ channels followed by channel-wise product pooling with window size $P$ (multiplicative non-linearity).
Denoting by $\bigl ( \hbf^{(l - 1, 1)}, \ldots,  \hbf^{(l - 1, N / P^{l - 1})} \bigr ) \in \R^{R_{l - 1}} \times \cdots \times \R^{R_{l - 1}}$ the output of the $l - 1$'th hidden layer, where $\bigl ( \hbf^{ (0, 1) }, \ldots,  \hbf^{ (0, N) } \bigr ) := \bigl ( \xbf^{ (1) }, \ldots, \xbf^{ (N) } \bigr )$ is the network's input, the locally connected operator of the $l$'th layer computes $\big ( \weightmat{l, n} \big )^\top \hbf^{(l - 1, n)}$ for each index $n \in [N / P^{l - 1}]$.
We refer to this operator as “$1 \times 1$ conv'' in appeal to the case of weight sharing, where $\weightmat{l, 1} = \cdots = \weightmat{l, N / P^{l - 1}}$.
Following the locally connected operator, for each $n \in [N / P^l]$ and $r \in [R_{l}]$, the pooling operator computes $\prod_{p = (n - 1) \cdot P + 1}^{n \cdot P}  \big [ \big ( \weightmat{l, p} \big )^\top \hbf^{(l - 1, p)} \big ]_r$, thereby producing $\big ( \hbf^{(l, 1)}, \ldots, \hbf^{(l, N / P^l)} \big )$.
After passing the input through all hidden layers, a final linear layer, whose weights are $\weightmat{L + 1, 1}$, yields the scalar output of the network $\bigl ( \weightmat{L + 1, 1} \bigr )^\top \hbf^{(L, 1)}$.
Notice that the weight matrices of the hierarchical tensor factorization are exactly the learnable weights of the network, and $R_{l - 1}$~---~the number of local components (\defin~\ref{def:local_comp}) at nodes in level $l$ of the factorization~---~is the width of the network's $l - 1$'th hidden layer.

The above formulation of the network supports not only sequential inputs (\eg~audio and text), but also inputs arranged as multi-dimensional arrays (\eg~two-dimensional images).  
The choice of how to assign the indices $1, \ldots, N$ to input elements determines the geometry of pooling windows throughout the network~\cite{cohen2017inductive}.

\prop~\ref{prop:htf_cnn} below implies that we may view solution of a prediction task using the deep convolutional network described above as a hierarchical tensor factorization problem, and vice versa.
For example, solving tensor completion and certain sensing problems using hierarchical tensor factorization amounts to applying the corresponding network to a regression task.
\endgroup
\begin{proposition}[adapted from~\citet{cohen2016expressive}]
\label{prop:htf_cnn}
Let $f_\Theta : \times_{n = 1} \R^{D_n} \to \R$ be the function realized by the deep non-linear convolutional network described above, where $\Theta := \big ( \weightmat{l, n} \big )_{l \in [L + 1], n \in [N / P^{l - 1}]}$ stands for the network's weights.
Denote by $\tensorend$ the end tensor of the hierarchical tensor factorization specified in \eq~\eqref{eq:ht_pary_end_tensor}.
Then, for all $\xbf^{(1)} \in \R^{D_1}, \ldots, \xbf^{(N)} \in \R^{D_N}$:
\[
f_\Theta \big ( \xbf^{(1)}, \ldots, \xbf^{(N)} \big ) = \inprodbig{ \tenp_{n = 1}^N \xbf^{(n)}  }{ \tensorend }
\text{\,.}
\]
\end{proposition}
\begin{proof}[Proof sketch (proof in \subapp~\ref{app:proofs:htf_cnn})]
By induction over the layers of the network, we show that the output of the $l$'th convolutional layer (linear output layer for $l = L + 1$) at index $n$ and channel $r$ is $\inprodbig{ \tenp_{p = (n -1 ) \cdot P^{l - 1} + 1}^{n \cdot P^{l - 1}} \xbf^{(p)} }{ \tensorpart{l, n}{r} }$, where \smash{$\tensorpart{L + 1, 1}{1} := \tensorend$}, and all other $\tensorpart{l,n}{r}$ are the intermediate tensors formed when computing $\tensorend$ according to \eq~\eqref{eq:ht_pary_end_tensor}.
Since $f_\Theta \big ( \xbf^{(1)}, \ldots, \xbf^{(N)} \big )$ is the output of the $L + 1$'th layer at index $1$ and channel $1$, applying the inductive claim for $l = L + 1, n = 1$, and $r = 1$ concludes the proof.
\end{proof}

We conclude this appendix by noting that in the special case where $P = N$, if the weight matrix of the root node holds ones, the hierarchical tensor factorization reduces to a tensor factorization, and the corresponding convolutional network has a single hidden layer (with global product pooling) followed by a final summation layer. 
We thus obtain the equivalence between tensor factorization and a shallow non-linear convolutional network as a corollary of \prop~\ref{prop:htf_cnn}.

	\section{Evolution of Local Component Norms Under Arbitrary Initialization}
\label{app:dyn_arbitrary}

\thm~\ref{thm:loc_comp_norm_bal_dyn} in \subsect~\ref{sec:inc_rank_lrn:evolution} characterizes the evolution of local component norms in a hierarchical tensor factorization, under the assumption of unbalancedness magnitude zero at initialization.
\thm~\ref{thm:loc_comp_norm_unbal_dyn} below extends the characterization to account for arbitrary initialization. 
It establishes that if the unbalancedness magnitude at initialization is small --- as is the case under any near-zero initialization~---~local component norms approximately evolve per \thm~\ref{thm:loc_comp_norm_bal_dyn}.

\begin{theorem}
	\label{thm:loc_comp_norm_unbal_dyn}
	With the context and notations of \thm~\ref{thm:loc_comp_norm_bal_dyn}, assume unbalancedness magnitude $\epsilon \geq 0$ at initialization.
	Then, for any $\nu \in \interior (\htmodetree)$, $r \in [R_\nu]$, and time $t \geq 0$ at which $\htfcompnorm{\nu}{r} (t) > 0$:\footnote{
		Since norms are not differentiable at the origin, when $\htfcompnorm{\nu}{r} (t)$ is equal to zero it may not be differentiable with respect to time.
	}
	\begin{itemize}
		\item If $\inprodbig{- \nabla \htfendloss ( \tensorend (t) ) }{ \htfcomp{\nu}{r} (t) } \geq 0$, then:
		\be
		\begin{split}
			&\frac{d}{dt} \htfcompnorm{\nu}{r} (t) \leq \left ( \htfcompnorm{\nu}{r} (t)^{\frac{2}{  L_{\nu} } } + \epsilon \right )^{L_{\nu} - 1} \cdot L_{\nu} \inprodbig{- \nabla \htfendloss ( \tensorend (t) ) }{ \htfcomp{\nu}{r} (t) }
			\text{\,,} \\[1mm]
			&\frac{d}{dt} \htfcompnorm{\nu}{r} (t) \geq \frac{ \htfcompnorm{\nu}{r} (t)^2 }{ \htfcompnorm{\nu}{r} (t)^{ \frac{2}{  L_{\nu} } }  + \epsilon } \cdot  L_{\nu} \inprodbig{- \nabla \htfendloss ( \tensorend (t) ) }{ \htfcomp{\nu}{r} (t) }
			\text{\,;}
		\end{split}
		\label{eq:loc_comp_norm_unbal_pos_bound}
		\ee
		\item otherwise, if $\inprodbig{- \nabla \htfendloss ( \tensorend (t) ) }{ \htfcomp{\nu}{r} (t) } < 0$, then:
		\be
		\begin{split}
			&\frac{d}{dt} \htfcompnorm{\nu}{r} (t) \geq \left ( \htfcompnorm{\nu}{r} (t)^{\frac{2}{ L_{\nu} } } + \epsilon \right )^{  L_{\nu} - 1 } \cdot  L_{\nu} \inprodbig{- \nabla \htfendloss ( \tensorend (t) ) }{ \htfcomp{\nu}{r} (t) }
			\text{\,,} \\[1mm]
			&\frac{d}{dt} \htfcompnorm{\nu}{r} (t) \leq \frac{ \htfcompnorm{\nu}{r} (t)^2 }{ \htfcompnorm{\nu}{r} (t)^{ \frac{2}{ L_{\nu} } }  + \epsilon } \cdot  L_{\nu} \inprodbig{- \nabla \htfendloss ( \tensorend (t) ) }{ \htfcomp{\nu}{r} (t) }
			\text{\,.}
		\end{split}
		\label{eq:loc_comp_norm_unbal_neg_bound}
		\ee
	\end{itemize}
\end{theorem}
\begin{proof}[Proof sketch (proof in \subapp~\ref{app:proofs:loc_comp_norm_unbal_dyn})]
	The proof follows a line similar to that of \thm~\ref{thm:loc_comp_norm_bal_dyn}, except that here conservation of unbalancedness magnitude leads to $\normnoflex{ \wbf (t) }^2 \leq \htfcompnorm{\nu}{r} (t)^{ \frac{2}{ L_\nu } } + \epsilon$ for all $\wbf \in \localcomp (\nu, r)$.
	Applying this inequality to $\tfrac{d}{dt} \htfcompnorm{\nu}{r} (t) = \inprodbig{ - \nabla \htfendloss ( \tensorend (t)) }{ \htfcomp{\nu}{r} (t) } \sum\nolimits_{ \wbf \in \localcomp (\nu, r) } \prod\nolimits_{ \wbf' \in \localcomp (\nu, r) \setminus \{ \wbf \} } \norm{ \wbf' (t) }^2$ yields \eqs~\eqref{eq:loc_comp_norm_unbal_pos_bound} and~\eqref{eq:loc_comp_norm_unbal_neg_bound}.
\end{proof}

	\section{Hierarchical Tensor Rank as Measure of Long-Range Dependencies}
\label{app:ht_sep_rank}

\subsect~\ref{sec:low_htr_implies_locality} discusses the known fact by which the hierarchical tensor rank (\defin~\ref{def:ht_rank}) of a hierarchical tensor factorization measures the strength of long-range dependencies modeled by the equivalent convolutional network (see~\citet{cohen2017inductive,levine2018benefits,levine2018deep}).
For the convenience of the reader, the current appendix formally explains this fact.

Consider a hierarchical tensor factorization with mode tree $\htmodetree$ (\defin~\ref{def:mode_tree}), weight matrices \smash{$\Theta := \big ( \weightmat{\nu} \big )_{\nu \in \htmodetree}$}, and an equivalent convolutional network realizing a parametric input-output function $f_\Theta$.
As claimed in \sect~\ref{sec:htf} (and formally justified in \app~\ref{app:htf_cnn}), the function realized by the convolutional network takes the form $f_\Theta \big ( \xbf^{(1)}, \ldots, \xbf^{(N)} \big ) = \inprodbig{ \tenp_{n = 1}^N \xbf^{(n)}  }{ \tensorend }$, where $\tensorend$ stands for the end tensor of the factorization (\eq~\eqref{eq:ht_end_tensor}).
\prop~\ref{prop:matrank_eq_seprank} below establishes that for any subset of indices $I \subset [N]$, the matrix rank of~$\tensorend$'s matricization according to $I$ is equal to the separation rank (\defin~\ref{def:sep_rank}) of $f_\Theta$ with respect to $I$, \ie~$\rank \mat{\tensorend}{I} = \seprank ( f_\Theta ; I)$.
In particular, the hierarchical tensor rank of $\tensorend$ with respect to $\htmodetree$~---~$\big ( \rank \mat{\tensorend}{\nu} \big )_{\nu \in \htmodetree  \setminus \{ [N] \} }$~---~amounts to $\big ( \seprank ( f_\Theta ; \nu )\big )_{\nu \in \htmodetree \setminus \{ [N] \} }$.
In the canonical case where nodes in $\htmodetree$ hold adjacent indices, the separation ranks of  $f_\Theta$ with respect to them measure the dependencies modeled between distinct areas of the input, \ie~the non-local (long-range) dependencies.

\begin{proposition}[adaptation of Claim 1 in~\citet{cohen2017inductive}]
\label{prop:matrank_eq_seprank}
Consider a hierarchical tensor factorization with mode tree $\htmodetree$ (\defin~\ref{def:mode_tree}) and weight matrices $\Theta := \big ( \weightmat{\nu} \big )_{\nu \in \htmodetree}$, and denote its end tensor by $\tensorend$ (\eq~\eqref{eq:ht_end_tensor}). Let $f_\Theta : \times_{n = 1} \R^{D_n} \to \R$ be defined by $f_\Theta \big ( \xbf^{(1)}, \ldots, \xbf^{(N)} \big ) := \inprodbig{ \tenp_{n = 1}^N \xbf^{(n)}  }{ \tensorend }$.
Then, for all $I \subset [ N ]$:
\[
\rank \mat{\tensorend}{ I} = \seprank (f_\Theta ; I)
\text{\,.}
\]
\end{proposition}
\vspace{-3mm}
\begin{proof}[Proof sketch (proof in \subapp~\ref{app:proofs:matrank_eq_seprank})]
To prove that $\rank \mat{\tensorend }{ I } \geq \seprank (f_\Theta ; I)$, we derive a representation of $f_\Theta$ as a sum of $\rank \mat{\tensorend}{I}$ terms, each being a product between a function that operates over $( \xbf^{(i)} )_{i \in I}$ and another that operates over the remaining input variables.
For the converse, $\rank \mat{\tensorend}{I} \leq \seprank (f_\Theta ; I)$, we prove that for any grid tensor $\W$ of a function $f$, \ie~tensor holding the outputs of $f$ over a grid of inputs, it holds that $\rank \mat{\W}{I} \leq \seprank (f; I)$.
We conclude by showing that $\tensorend$ is a grid tensor of~$f_\Theta$.
\end{proof}

	\section{Further Experiments and Implementation Details}
\label{app:experiments}

\subsection{Further Experiments} \label{app:experiments:further}

\figs~\ref{fig:tc_o4_p2},~\ref{fig:ts_o4_p2}, and~\ref{fig:tc_o9_p3} supplement \fig~\ref{fig:mf_tf_htf_dynamics} by including, respectively: \emph{(i)} plots of additional local component norms and singular values during optimization in the experiment presented by \fig~\ref{fig:mf_tf_htf_dynamics} (right); \emph{(ii)} experiments with tensor sensing loss; and \emph{(iii)} experiments with different hierarchical tensor factorization orders and mode trees, as well as different ground truth hierarchical tensor ranks.
\fig~\ref{fig:long_range_and_reg_results_resnet34} portrays an experiment identical to that of \fig~\ref{fig:long_range_and_reg_results}, but with ResNet34 in place of ResNet18.  
\figs~\ref{fig:long_range_other_reg_results_resnet18} and~\ref{fig:long_range_other_reg_results_resnet34} extend \figs~\ref{fig:long_range_and_reg_results} and~\ref{fig:long_range_and_reg_results_resnet34}, respectively, by presenting results obtained with baseline networks that are already regularized using standard techniques (weight decay and dropout).

\subsection{Implementation Details} \label{app:experiments:details}

In this subappendix we provide implementation details omitted from our experimental reports (\fig~\ref{fig:mf_tf_htf_dynamics}, \sect~\ref{sec:countering_locality}, and \subapp~\ref{app:experiments:further}).
Source code for reproducing our results and figures, based on the PyTorch framework~\cite{paszke2017automatic},\ifdefined\CAMREADY
~can be found at \url{https://github.com/asafmaman101/imp_reg_htf}.
\else
~is attached as supplementary material and will be made publicly available.
\fi
All experiments were run on a single Nvidia RTX 2080 Ti GPU.

\subsubsection{Incremental Hierarchical Tensor Rank Learning (\figs~\ref{fig:mf_tf_htf_dynamics},~\ref{fig:tc_o4_p2},~\ref{fig:ts_o4_p2}, and~\ref{fig:tc_o9_p3})}
\label{app:experiments:details:inc_rank_lrn}

\textbf{\fig~\ref{fig:mf_tf_htf_dynamics} (left):} the minimized matrix completion loss was $\mfendloss ( \matrixend ) = \frac{1}{\abs{\Omega}} \sum\nolimits_{(i, j) \in \Omega} ( (\matrixend)_{i, j} - \Wbf^*_{i, j} )^2$, where $\Omega$ denotes a set of $2048$ observed entries chosen uniformly at random (without repetition) from a matrix rank $5$ ground truth $\Wbf^* \in \R^{64, 64}$.
We generated $\Wbf^*$ by computing $\Wbf^{* (1) } \Wbf^{* (2) }$, with each entry of $\Wbf^{* (1) } \in \R^{64, 5}$ and $\Wbf^{* (2)} \in \R^{5, 64}$ drawn independently from the standard normal distribution, and subsequently normalizing the result to be of Frobenius norm $64$ (square root of its number of entries).
Reconstruction error with respect to $\Wbf^*$ is based on normalized Frobenius distance, \ie~for a solution $\matrixend$ it is $\norm{ \matrixend - \Wbf^* } / \norm{ \Wbf^* }$.
The matrix factorization applied to the task was of depth $3$ and had hidden dimensions $64$ between its layers so that its rank was unconstrained.
Standard deviation for initialization was set to $0.001$.

\textbf{\fig~\ref{fig:mf_tf_htf_dynamics} (middle):} the minimized tensor completion loss was $\tfendloss ( \tftensorend ) = \frac{1}{\abs{\Omega}} \sum\nolimits_{(d_1, d_2, d_3) \in \Omega} ( (\tftensorend)_{d_1, d_2, d_3} - \W^*_{d_1,  d_2, d_3} )^2$, where $\Omega$ denotes a set of $2048$ observed entries chosen uniformly at random (without repetition) from a tensor rank $5$ ground truth $\W^* \in \R^{16, 16, 16}$.
We generated $\W^*$ by computing $\sum_{r = 1}^5 \Wbf^{* (1)}_{:, r} \tenp \Wbf^{* (2)}_{:, r} \tenp \Wbf^{* (3)}_{:, r}$, with each entry of $\Wbf^{* (1)}, \Wbf^{* (2)},$ and $\Wbf^{* (3)} \in \R^{16, 5}$ drawn independently from the standard normal distribution, and subsequently normalizing the result to be of Frobenius norm $64$ (square root of its number of entries).
Reconstruction error with respect to $\W^*$ is based on normalized Frobenius distance, \ie~for a solution $\tftensorend$ it is $\norm{ \tftensorend - \W^* } / \norm{ \W^* }$.
The tensor factorization applied to the task had $R = 256$ components so that its tensor rank was unconstrained.\footnote{
For any $D_1, \ldots, D_N \in \N$, setting $R = ( \prod_{n = 1}^N D_n) / \max \{ D_n \}_{n = 1}^N$ suffices for expressing all tensors in $\R^{D_1, \ldots, D_N}$ (Lemma~3.41 in~\citet{hackbusch2012tensor}).
} Standard deviation for initialization was set to~$0.001$.

\begin{figure*}[t!]
	\vspace{-1mm}
	\begin{center}
		\hspace{-3.8mm}
		\includegraphics[width=0.95\textwidth]{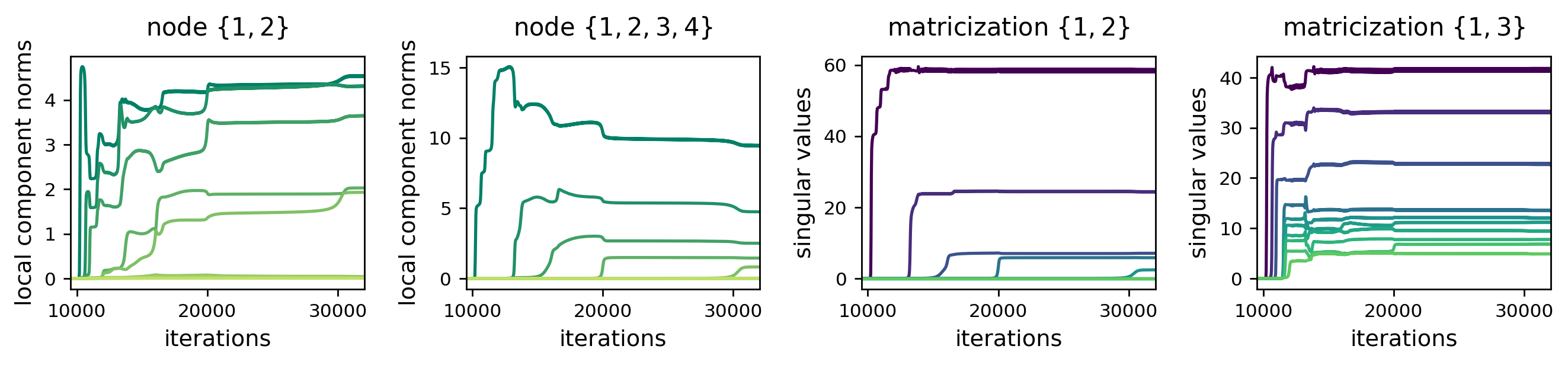}
	\end{center}
	\vspace{-4mm}
	\caption{
		Dynamics of gradient descent over order four hierarchical tensor factorization with a perfect binary mode tree (on tensor completion task)~---~incremental learning leads to low hierarchical tensor rank. 
		For the hierarchical tensor factorization experiment in \fig~\ref{fig:mf_tf_htf_dynamics} (right), plots present the evolution of additional quantities during optimization.
		\textbf{Left and second to left:} top $10$ local component norms at nodes $\{ 1, 2 \}$ and $\{ 1, 2, 3, 4\}$ (respectively) in the mode tree (the latter also appears in \fig~\ref{fig:mf_tf_htf_dynamics} (right)).
		\textbf{Second to right and right:} top $10$ singular values of the end tensor's matricizations according to $\{ 1, 2\}$ and $\{ 1, 3 \}$ (respectively).
		The former corresponds to a node in the mode tree, meaning its rank is part of the end tensor's hierarchical tensor rank, whereas the latter does not.
		\textbf{All:} notice that, in line with our analysis (\sect~\ref{sec:inc_rank_lrn}), local component norms move slower when small and faster when large, creating an incremental process that leads to low hierarchical tensor rank solutions.
		Moreover, the singular values of the end tensor's matricizations according to nodes in the mode tree exhibit a similar behavior, whereas those of matricizations according to index sets outside the mode tree do not.
		The rank of a matricization lower bounds the (non-hierarchical) tensor rank (Remark~6.21 in~\citet{hackbusch2012tensor}). Thus, while the hierarchical tensor rank of the obtained solution is low, its tensor rank is high.
		For further implementation details, such as loss definition and factorization size, see \subapp~\ref{app:experiments:details}.
	}
	\label{fig:tc_o4_p2}
\end{figure*}

\begin{figure*}[t!]
	\vspace{-2.5mm}
	\begin{center}
		\hspace{-3.8mm}
		\includegraphics[width=0.95\textwidth]{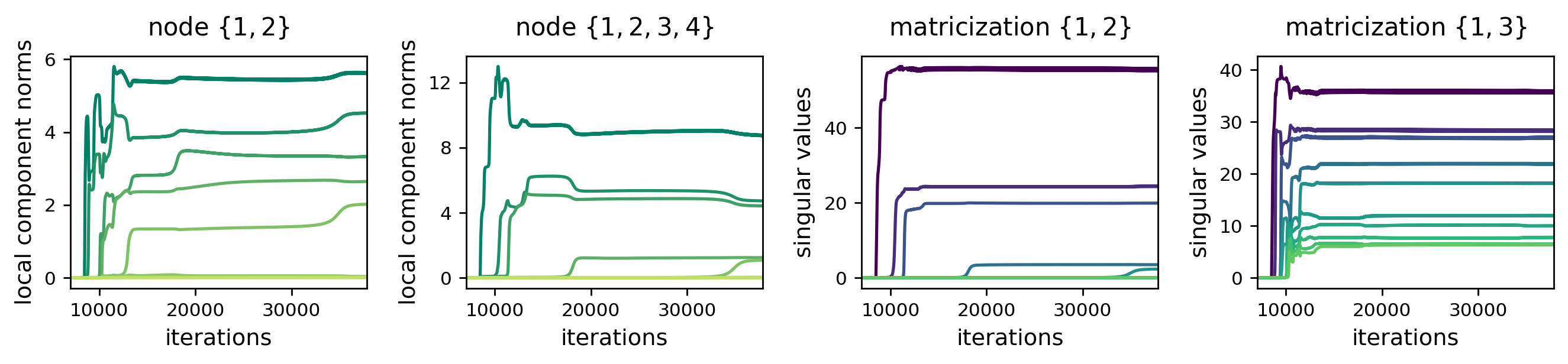}
	\end{center}
	\vspace{-4mm}
	\caption{
		Dynamics of gradient descent over order four hierarchical tensor factorization with a perfect binary mode tree (on tensor sensing task)~---~incremental learning leads to low hierarchical tensor rank.
		This figure is identical to \fig~\ref{fig:tc_o4_p2}, except that the minimized mean squared error was based on random linear measurements (instead of randomly chosen entries).
		For further implementation details, such as loss definition and factorization size, see \subapp~\ref{app:experiments:details}.
	}
	\label{fig:ts_o4_p2}
\end{figure*}

\begin{figure*}[t!]
	\vspace{-2.5mm}
	\begin{center}
		\hspace{-3.8mm}
		\includegraphics[width=0.95\textwidth]{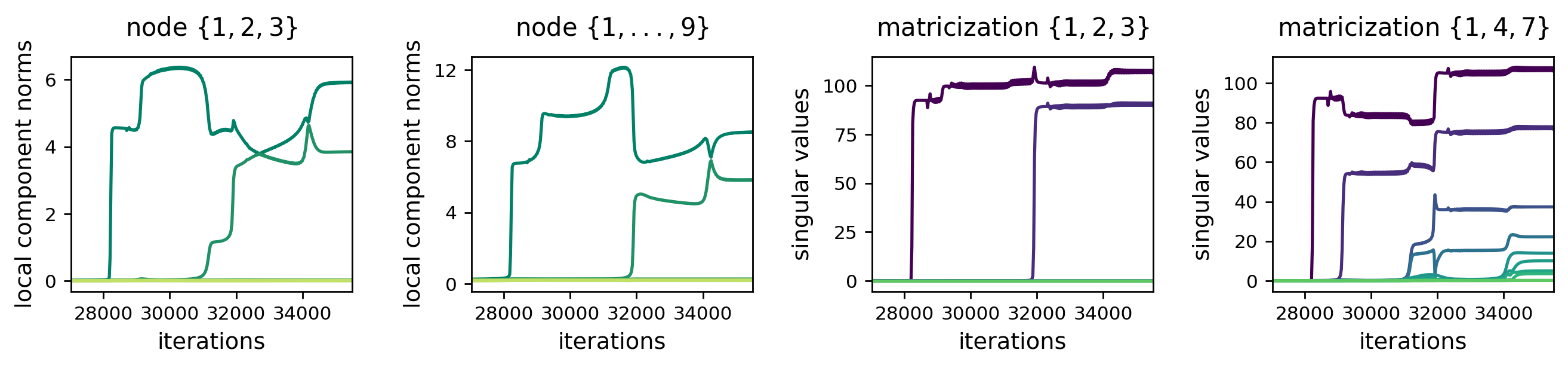}
	\end{center}
	\vspace{-4mm}
	\caption{
		Dynamics of gradient descent over order nine hierarchical tensor factorization with a perfect ternary mode tree~---~incremental learning leads to low hierarchical tensor rank.
		This figure is identical to \fig~\ref{fig:tc_o4_p2}, except that: \emph{(i)} the hierarchical tensor factorization employed had order nine and complied with a perfect ternary mode tree; and \emph{(ii)} the ground truth tensor was of hierarchical tensor rank $(2, \ldots, 2)$ (\defin~\ref{def:ht_rank}).
		For further implementation details, such as loss definition and factorization size, see \subapp~\ref{app:experiments:details}.
	}
	\label{fig:tc_o9_p3}
	\vspace{-2mm}
\end{figure*}

\textbf{\fig~\ref{fig:mf_tf_htf_dynamics} (right):} the minimized tensor completion loss was $\htfendloss ( \tensorend ) = \frac{1}{\abs{\Omega}} \sum\nolimits_{(d_1, \ldots, d_4) \in \Omega} ( (\tensorend)_{d_1, \ldots, d_4} - \W^*_{d_1,  \ldots, d_4} )^2$, where $\Omega$ denotes a set of $2048$ observed entries chosen uniformly at random (without repetition) from a hierarchical tensor rank $(5, 5, 5, 5, 5, 5)$ ground truth $\W^* \in \R^{8, 8, 8, 8}$.
We generated $\W^*$ according to \eq~\eqref{eq:ht_end_tensor} using a perfect binary mode tree $\htmodetree$ over $[4]$ and weight matrices $\big ( \Wbf^{*(\nu)} \big )_{\nu \in \htmodetree}$, where $\Wbf^{* (\nu) } \in \R^{8, 5}$ for $\nu \in \{ \{1\}, \ldots, \{4\} \}$, $\Wbf^{* (\nu) } \in \R^{5, 5}$ for $\nu \in \interior (\htmodetree) \setminus \{ [4] \}$, and $\Wbf^{* ( [4] ) }  \in \R^{5, 1}$.
We sampled the entries of $\big ( \Wbf^{* (\nu) } \big )_{\nu \in \htmodetree}$ independently from the standard normal distribution, and subsequently normalized the ground truth to be of Frobenius norm $64$ (square root of its number of entries).
Reconstruction error with respect to $\W^*$ is based on normalized Frobenius distance, \ie~for a solution $\tensorend$ it is $\norm{ \tensorend - \W^* } / \norm{ \W^* }$.
The hierarchical tensor factorization applied to the task had $512$ local components at all interior nodes due to computational and memory considerations (increasing the number of local components had no substantial impact on the dynamics).
Standard deviation for initialization was set to~$0.01$.

\textbf{\fig~\ref{fig:tc_o4_p2}:} plots correspond to the same experiment presented in \fig~\ref{fig:mf_tf_htf_dynamics} (right).

\textbf{\fig~\ref{fig:ts_o4_p2}:} implementation details are identical to those of \fig~\ref{fig:mf_tf_htf_dynamics} (right), except that the following tensor sensing loss was minimized: $\htfendloss ( \tensorend ) = \sum_{i = 1}^{2048} \big ( \inprodbig{ \tenp_{n = 1}^4 \xbf^{(i, n)} }{ \tensorend} - \inprodbig{ \tenp_{n = 1}^4 \xbf^{(i, n)} }{\W^*} \big )^2$, where the entries of $\big ( ( \xbf^{(i, 1)}, \ldots, \xbf^{(i, 4)} ) \in \R^8 \times \cdots \times \R^8 \big )_{i = 1}^{2048}$ were sampled independently from a zero-mean Gaussian distribution with standard deviation $4096^{- 1 / 8}$ (ensures each measurement tensor $\tenp_{n = 1}^4 \xbf^{(i, n)}$ has expected square Frobenius norm $1$).

\textbf{\fig~\ref{fig:tc_o9_p3}:} implementation details are identical to those of \fig~\ref{fig:mf_tf_htf_dynamics} (right), except that: \emph{(i)} the ground truth tensor was of order $9$ with modes of dimension $3$, Frobenius norm $\sqrt{19683}$ (square root of its number of entries), hierarchical tensor rank $(2, \ldots, 2)$, and was generated according to a perfect ternary mode tree; \emph{(ii)} reconstruction was based on $9840$ entries chosen uniformly at random; \emph{(iii)} the hierarchical tensor factorization applied to the task had $100$ local components at all  interior nodes; and \emph{(iv)} standard deviation for initialization was set to $0.1$.

\textbf{All:} using sample sizes smaller than those specified above led to similar results, up until a point where solutions found had fewer non-zero singular values, components, or local components (at all nodes) than the ground truths.
Gradient descent was initialized randomly by sampling each weight in the factorization independently from a zero-mean Gaussian distribution, and was run until the loss remained under $5 \cdot 10^{-5}$ for $100$ iterations in a row.
For each figure, experiments were carried out with initialization standard deviations $0.1, 0.05, 0.01, 0.005, 0.001,$ and $0.0005$.
Reported are representative runs striking a balance between the potency of the incremental learning effect and run time.
Reducing standard deviations further did not yield a significant change in the dynamics, yet resulted in longer optimization times due to vanishing gradients around the origin.

To facilitate more efficient experimentation, we employed an adaptive learning rate scheme, where at each iteration a base learning rate is divided by the square root of an exponential moving average of squared gradient norms.
That is, for base learning rate $\eta = 10^{-2}$ and weighted average coefficient $\beta = 0.99$, at iteration~$t$ the learning rate was set to $\eta_t = \eta / (\sqrt{\gamma_t / (1 - \beta^t)} + 10^{-6})$, where \smash{$\gamma_t = \beta \cdot \gamma_{t-1} + (1 - \beta) \cdot \normnoflex{ \nicefrac{\partial}{\partial \Theta} \phi ( \Theta (t) ) }^2$}, $\gamma_0 = 0$, $\phi$ stands for any one of $\mfobj, \tfobj$, or $\htfobj$, and $\Theta$ denotes the corresponding factorization's weights.
We emphasize that only the learning rate (step size) is affected by this scheme, not the direction of movement.
Comparisons between the scheme and optimization with a fixed learning rate showed no significant difference in terms of the dynamics, while run times of the former were considerably shorter.

\subsubsection{Countering Locality of Convolutional Networks via Regularization (\figs~\ref{fig:long_range_and_reg_results},~\ref{fig:long_range_and_reg_results_resnet34},~\ref{fig:long_range_other_reg_results_resnet18}, and~\ref{fig:long_range_other_reg_results_resnet34})} 
\label{app:experiments:details:conv}

In all experiments, we randomly initialized the ResNet18 and ResNet34 networks according to the default PyTorch~\cite{paszke2017automatic} implementation.
The (regularized) binary cross-entropy loss was minimized via stochastic gradient descent with learning rate $0.01$, momentum coefficient $0.9$, and batch size $64$ (for ResNet34 we used a batch size of $32$ and accumulated gradients over two batches due to GPU memory considerations).
Optimization proceeded until perfect training accuracy was attained for $20$ consecutive epochs or $150$ epochs elapsed (runs without regularization always reached perfect training accuracy).
For each dataset and model combination, runs were carried out using the regularization described in \subsect~\ref{sec:countering_locality:reg} with coefficients $0, 0.1, 0.5, 1, 3, 6, 9,$ and $10$.
Values lower than those reported in \figs~\ref{fig:long_range_and_reg_results} and~\ref{fig:long_range_and_reg_results_resnet34} had no noticeable impact, whereas higher values typically did not allow fitting the training data.
\tab~\ref{tab:other_reg_hyperparams} specifies the hyperparameters used for the different regularizations in the experiments of~\figs~\ref{fig:long_range_other_reg_results_resnet18} and~\ref{fig:long_range_other_reg_results_resnet34}.
Dropout layers shared the same probability hyperparameter, and were inserted before blocks expanding the number of channels, \ie~before the first convolutional layers with $128$, $256$, and $512$ output channels (the default ResNet18 and ResNet34 implementations do not include dropout).

\begin{figure*}[t]
	\vspace{-1mm}
	\begin{center}
		\hspace{-3mm}
		\includegraphics[width=0.868\textwidth]{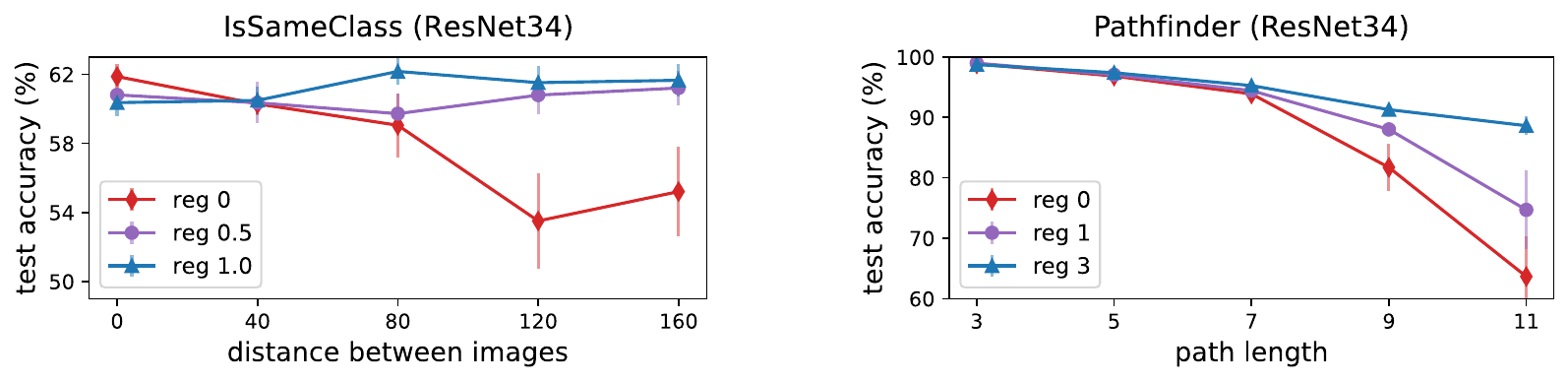}
	\end{center}
	\vspace{-4mm}
	\caption{	
		Dedicated explicit regularization can counter the locality of convolutional networks, significantly improving performance on tasks with long-range dependencies.
		This figure is identical to \fig~\ref{fig:long_range_and_reg_results}, except that: \emph{(i)} experiments were carried out using a randomly initialized ResNet34 (as opposed to ResNet18); and \emph{(ii)} it includes evaluation over a Pathfinder dataset with path length $11$, since up until path length $9$ an unregularized network still obtained non-trivial performance.
		For further details see \subapp~\ref{app:experiments:details:conv}.
	}
	\label{fig:long_range_and_reg_results_resnet34}
	\vspace{-2mm}
\end{figure*}

\begin{figure*}[t!]
	\vspace{0mm}
	\begin{center}
		\hspace{-3mm}
		\includegraphics[width=0.98\textwidth]{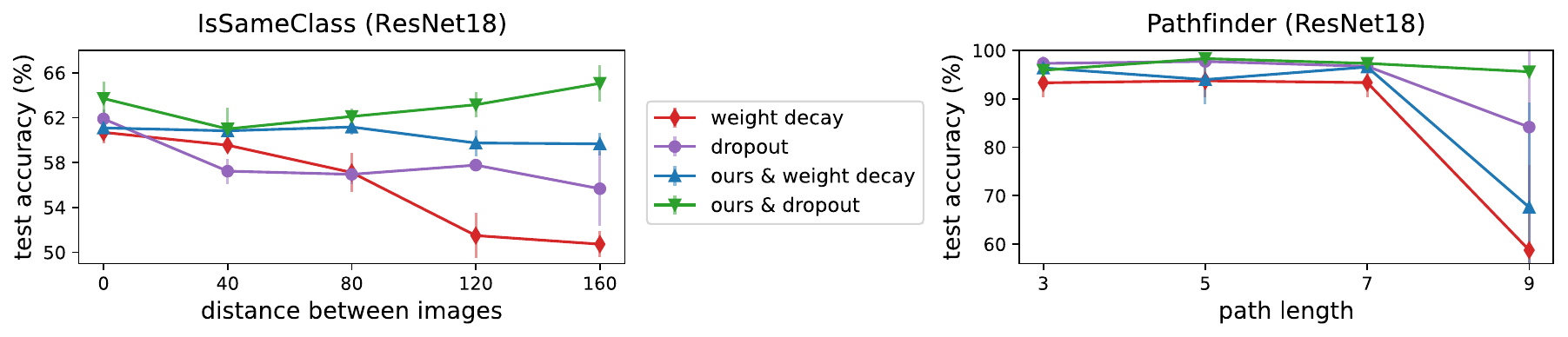}
	\end{center}
	\vspace{-4mm}
	\caption{	
		Dedicated explicit regularization can counter the locality of convolutional networks (regularized via standard techniques), significantly improving performance on tasks with long-range dependencies.
		This figure is identical to \fig~\ref{fig:long_range_and_reg_results}, except that instead of applying our regularizer (\subsect~\ref{sec:countering_locality:reg}) to a baseline unregularized network, the baseline networks here were regularized using either weight decay or dropout, and are compared to the results obtained when applying our regularization in addition to them.
		\fig~\ref{fig:long_range_and_reg_results} shows that the test accuracy obtained by an unregularized network substantially deteriorates when increasing the (spatial) range of dependencies required to be modeled.
		From the plots above it is evident that, even when employing standard regularization techniques such as weight decay or dropout, a similar degradation in performance occurs.
		As was the case for unregularized networks, applying our dedicated regularization, in addition to these techniques, significantly improved performance.
		In particular, for the combination of our regularization and dropout, the test accuracy was high across all datasets.
		For further details such as regularization hyperparameters, see \subapp~\ref{app:experiments:details:conv}.
	}
	\label{fig:long_range_other_reg_results_resnet18}
	\vspace{-2mm}
\end{figure*}

\begin{figure*}[t!]
	\vspace{0mm}
	\begin{center}
		\hspace{-3mm}
		\includegraphics[width=0.98\textwidth]{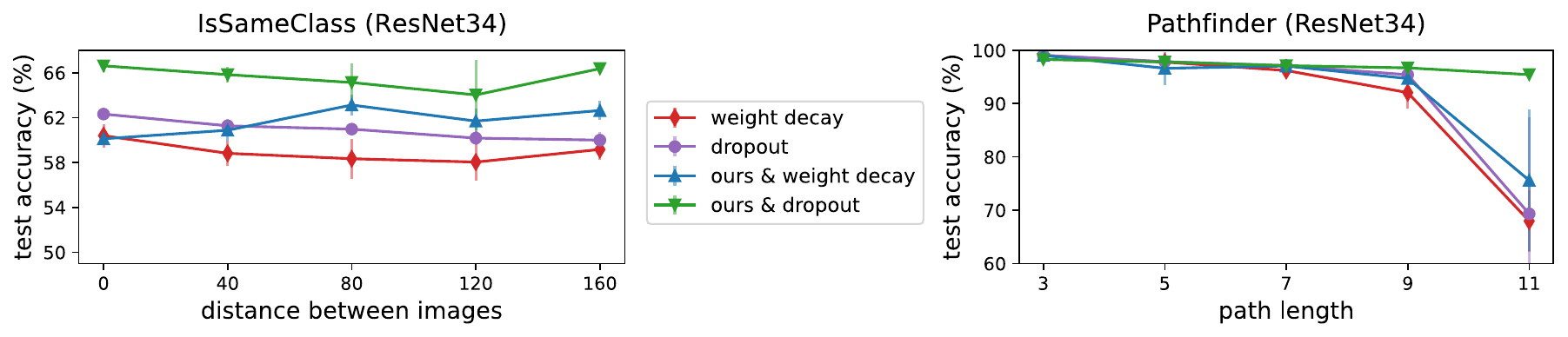}
	\end{center}
	\vspace{-4mm}
	\caption{	
		Dedicated explicit regularization can counter the locality of convolutional networks (regularized via standard techniques), significantly improving performance on tasks with long-range dependencies.
		This figure is identical to \fig~\ref{fig:long_range_other_reg_results_resnet18},  except that: \emph{(i)} experiments were carried out using a randomly initialized ResNet34 (as opposed to ResNet18); and \emph{(ii)} it includes evaluation over a Pathfinder dataset with path length $11$, since up until path length $9$ networks regularized using weight decay or dropout still obtained non-trivial performance.
		For further details such as regularization hyperparameters, see \subapp~\ref{app:experiments:details:conv}.	
	}
	\label{fig:long_range_other_reg_results_resnet34}
	\vspace{-2mm}
\end{figure*}

At each stochastic gradient descent iteration, the subset of indices $I$ and $J$ used for computing the regularized objective were sampled as follows.
For IsSameClass datasets, we set $I$ to be the indices marking either the left or right CIFAR10 image uniformly at random, and then let $J$ be the indices corresponding to the remaining CIFAR10 image.
For Pathfinder datasets, $I$ and $J$ were set to non-overlapping $2 \times 2$ patches chosen uniformly across the input.
In order to prevent additional computational overhead, alternative values for pixels indexed by $J$ were taken from other images in the batch (as opposed to from the whole training set).
Specifically, we used a permutation without fixed points to shuffle the pixel patches indexed by~$J$ across the batch.

IsSameClass datasets consisted of $5000$ training and $10000$ test samples.
Each sample was generated by first drawing uniformly at random a label from $\{0, 1\}$ and an image from CIFAR10.
Then, depending on the chosen label, another image was sampled either from the same class (for label $1$) or from all other classes (for label $0$).
Lastly, the CIFAR10 images were placed at a predetermined horizontal distance from each other around the center of a $224 \times 224$ image filled with zeros.
For example, when the horizontal distance is $0$, the CIFAR10 images are adjacent, and when it is $160$, they reside in opposite borders of the $224 \times 224$ input.
Pathfinder datasets consisted of $10000$ training and $10000$ test samples.
Given a path length, the corresponding dataset was generated according to the protocol of~\citet{linsley2018learning}, with hyperparameters: circle radius $3$, paddle length $5$, paddle thickness $2$, inner paddle margin $3$, and continuity $1.8$.
See~\citet{linsley2018learning} for additional information regarding the data generation process.
As to be expected, when running a subset of all experiments using larger training set sizes (for both IsSameClass and Pathfinder datasets), we observed improved generalization across the board.
Nevertheless, the addition of training samples did not alleviate the degradation in test accuracy observed for larger horizontal distances and path lengths, nor did it affect the beneficial impact of our regularization.
That is, the trends observed in \figs~\ref{fig:long_range_and_reg_results},~\ref{fig:long_range_and_reg_results_resnet34},~\ref{fig:long_range_other_reg_results_resnet18}, and~\ref{fig:long_range_other_reg_results_resnet34} remained intact up to a certain shift upwards.

\begin{table*}[t]
	\caption{
		Hyperparameters for the regularizations employed in the experiments of~\figs~\ref{fig:long_range_other_reg_results_resnet18} and~\ref{fig:long_range_other_reg_results_resnet34}.
		For every model and dataset type combination, table reports the weight decay coefficient and dropout probability used when applied individually, as well as when combined with our regularization (described in \subsect~\ref{sec:countering_locality:reg}), whose coefficients are also specified.
		These hyperparameters were tuned on the datasets with largest spatial range between salient regions of the input.
		That is, for each model separately, their values on IsSameClass datasets were set to those achieving the best test accuracy over a dataset with $160$ pixels between CIFAR10 images.
		Similarly, their values on Pathfinder datasets were set to those achieving the best test accuracy over a dataset with connecting path length $9$ for ResNet18 and path length $11$ for ResNet34.
		For further details see the captions of \figs~\ref{fig:long_range_other_reg_results_resnet18} and~\ref{fig:long_range_other_reg_results_resnet34}, as well as \subapp~\ref{app:experiments:details:conv}.
	}
	\label{tab:other_reg_hyperparams}
	\vspace{-2mm}
	\begin{center}
		\begin{small}
			\begin{tabular}{lcccc}
				\toprule
				& \multicolumn{2}{c}{ResNet18} & \multicolumn{2}{c}{ResNet34} \\
				\cmidrule(lr){2-3}\cmidrule(lr){4-5}
				& IsSameClass & Pathfinder & IsSameClass & Pathfinder \\
				\midrule
				Weight Decay & $0.001$ & $0.01$ & $0.01$ & $0.001$ \\
				Dropout & $0.6$ & $0.5$ & $0.3$ & $0.2$ \\
				Ours \& Weight Decay & $1$ ~\&~ $0.001$ & $0.1$ ~\&~ $0.01$ & $1$ ~\&~ $0.0001$ & $0.1$ ~\&~ $0.001$ \\
				Ours \& Dropout & $1$ ~\&~ $0.5$ & $0.1$ ~\&~ $0.4$ & $1$ ~\&~ $0.5$ & $0.5$ ~\&~ $0.3$ \\
				\bottomrule
			\end{tabular}
		\end{small}
	\end{center}
	\vskip -0.15in
\end{table*}

\section{Deferred Proofs}
\label{app:proofs}

\subsection{Additional Notation}
\label{app:proofs:notation}

Before delving into the proofs, we introduce the following notation.

\vspace{-2mm}

\paragraph*{General}
A colon is used to indicate a range of entries in a mode, \eg~$\Wbf_{i, :} \in \R^{D'}$ and $\Wbf_{:, j} \in \R^D$ are the $i$'th row and $j$'th column of $\Wbf \in \R^{D, D'}$, respectively, and $\Wbf_{:i, :j} \in \R^{i, j}$ is the sub-matrix of $\Wbf$ consisting of its first $i$ rows and $j$ columns.
For $\W \in \R^{D_1, \ldots, D_N}$, we let $\vectnoflex{\W} \in \R^{\prod_{n = 1}^N D_n}$ be its arrangement as a vector.
The tensor and Kronecker products are denoted by $\tenp$ and $\kronp$, respectively.

\vspace{-2mm}

\paragraph*{Hierarchical tensor factorization}
For a mode tree $\htmodetree$ over $[N]$ (\defin~\ref{def:mode_tree}), we denote the set of nodes in the sub-tree of~$\htmodetree$ whose root is $\nu \in \htmodetree$ by $\subtree (\nu) \subset \htmodetree$.
The sets of left and right siblings of $\nu \in \htmodetree$ are denoted by $\lsib (\nu)$ and $\rsib (\nu)$, respectively.
For $\nu \in \htmodetree$, we let $\tensorpart{\nu}{:}$ be the tensor obtained by stacking $\big ( \tensorpart{\nu}{r} \big )_{r = 1}^{R_{ \parent (\nu) } }$ into a single tensor, \ie~$\tensorpart{\nu}{:}_{:, \ldots, :, r} = \tensorpart{\nu}{r}$ for all $r \in [ R_{\parent (\nu)} ]$.
Given weight matrices $\big ( \weightmat{\nu} \in \R^{R_\nu, R_{ \parent (\nu) } } \big )_{\nu \in \htmodetree}$, the function mapping them to the end tensor they produce according to \eq~\eqref{eq:ht_end_tensor} is denoted by $\tensorendmap \big ( \big ( \weightmat{ \nu } \big )_{ \nu \in \htmodetree } \big )$.
For $\nu \in \htmodetree$, with slight abuse of notation we let $\tensorendmap \big ( \big ( \weightmat{ \nu' } \big )_{ \nu' \in \htmodetree \setminus \subtree (\nu)}, \tensorpart{\nu}{:} \big )$ be the function mapping $\big ( \tensorpart{\nu}{r} \big )_{r = 1}^{ R_{ \parent(\nu) } }$ and weight matrices outside of $\subtree (\nu)$ to the end tensor they produce.

\subsection{Useful Lemmas}
\label{app:proofs:useful_lemmas}

\subsubsection{Technical}
\label{app:proofs:useful_lemmas:technical}

\begin{lemma}
\label{lem:tenp_eq_kronp}
For any $\U \in \R^{D_1, \ldots, D_N}, \V \in \R^{ H_1, \ldots, H_K }$, and $I \subset [N + K]$:
\[
\mat{ \U \tenp \V }{ I}= \mat{ \U }{ I \cap [N] } \kronp \mat{ \V }{ I - N \cap [K]}
\text{\,,}
\]
where $I - N := \{ i - N : i \in I \}$.
\end{lemma}
\begin{proof}
The identity follows directly from the definitions of the tensor and Kronecker products.
\end{proof}

\begin{lemma}
\label{lem:kronp_mixed_prod_row_vec}
For any $\Ubf \in \R^{D_1, D_2}, \Vbf \in \R^{D_2, D_3}$, and $\wbf \in \R^{D_4}$, the following holds:
\[
\left ( \Ubf \Vbf \right ) \kronp \wbf^\top = \Ubf \left ( \Vbf \kronp \wbf^\top \right ) \quad , \quad \wbf^\top \kronp \left ( \Ubf \Vbf \right ) = \Ubf \left ( \wbf^\top \kronp \Vbf \right )
\text{\,.}
\]
\end{lemma}
\begin{proof}
According to the mixed-product property of the Kronecker product, for any matrices $\Abf, \Abf', \Bbf, \Bbf'$ for which $\Abf \Abf'$ and $\Bbf \Bbf'$ are defined, it holds that $( \Abf \Abf') \kronp ( \Bbf \Bbf') = ( \Abf \kronp \Bbf)(\Abf' \kronp \Bbf')$.
Thus:
\[
\big ( \Ubf \Vbf \big ) \kronp \wbf^\top = \big ( \Ubf \Vbf \big ) \kronp \big ( 1 \cdot \wbf^\top \big) = \big ( \Ubf \kronp 1)(\Vbf \kronp \wbf^\top \big ) = \Ubf \big ( \Vbf \kronp \wbf^\top \big )
\text{\,,}
\]
where $1$ is treated as the $1$-by-$1$ identity matrix.
Similarly:
\[
\wbf^\top \kronp \big ( \Ubf \Vbf \big ) = \big (1 \cdot \wbf^\top \big ) \kronp \big ( \Ubf \Vbf \big ) = \big (1 \kronp \Ubf \big ) \big ( \wbf^\top \kronp \Vbf  \big ) = \Ubf \big ( \wbf^\top \kronp \Vbf \big)
\text{\,.}
\]
\end{proof}

\subsubsection{Hierarchical Tensor Factorization}
\label{app:proofs:useful_lemmas:htf}

Suppose that use a hierarchical tensor factorization with mode tree $\htmodetree$, weight matrices $\big ( \weightmat{\nu} \in \R^{R_\nu, R_{ \parent (\nu) } } \big )_{\nu \in \htmodetree}$, and end tensor $\tensorend \in \R^{D_1, \ldots, D_N}$ (\eqs~\eqref{eq:ht_end_tensor}) to minimize $\htfobj$ (\eq~\eqref{eq:htf_obj}) via gradient flow (\eq~\eqref{eq:gf_htf}).
Under this setting, we prove the following technical lemmas.

\begin{lemma}
\label{lem:ht_multilinear}
The functions $\tensorendmap \big ( \big ( \weightmat{ \nu } \big )_{ \nu \in \htmodetree } \big )$ and $\tensorendmap \big ( \big ( \weightmat{ \nu' } \big )_{ \nu' \in \htmodetree \setminus \subtree (\nu)}, \tensorpart{\nu}{:} \big )$, for $\nu \in \htmodetree$, defined in \subapp~\ref{app:proofs:notation}, are multilinear.
\end{lemma}
\begin{proof}
We begin by proving that $\tensorendmap \big ( \big ( \weightmat{ \nu } \big )_{ \nu \in \htmodetree } \big )$ is multilinear. 
Fix $\nu \in \htmodetree$, and let $\weightmat{\nu}, \Ubf^{(\nu)} \in \R^{ R_\nu, R_{\parent (\nu)}}$, and $\alpha > 0$.

\paragraph*{Homogeneity} Denote by $\big ( \U^{ (\nu', r) }_\alpha \big )_{\nu' \in \htmodetree, r \in [ R_{ \parent (\nu') } ] }$ the intermediate tensors produced when computing the end tensor $\tensorendmap \big ( \big ( \weightmat{ \nu' } \big )_{ \nu' \in \htmodetree \setminus \{ \nu \} }, \alpha \cdot \weightmat{\nu} \big )$ according to \eq~\eqref{eq:ht_end_tensor} (there denoted $\big (\tensorpart{\nu'}{r} \big )_{\nu', r }$).
If $\nu$ is a leaf node, then $\U^{(\nu, r)}_\alpha = \alpha \cdot \weightmat{\nu}_{:, r} = \alpha \cdot \tensorpart{\nu}{r}$ for all $r \in [ R_{ \parent(\nu) } ]$.
Otherwise, if $\nu$ is an interior node, a straightforward computation leads to the same conclusion, \ie~for all $r \in [ R_{\parent (\nu) } ]$:
\[
\begin{split}
\U^{(\nu, r)}_\alpha & = \pi_\nu \left ( \sum\nolimits_{r' = 1}^{R_\nu} \alpha \cdot \weightmat{\nu}_{r', r} \left [ \tenp_{\nu_c \in \children ( \nu )} \tensorpart{ \nu_c }{ r' } \right ] \right ) \\
& =  \alpha \cdot \pi_\nu \left ( \sum\nolimits_{r' = 1}^{R_\nu} \weightmat{\nu}_{r', r} \left [ \tenp_{\nu_c \in \children ( \nu )} \tensorpart{ \nu_c }{ r' } \right ] \right )  \\
& = \alpha \cdot \tensorpart{\nu}{r}
\text{\,,}
\end{split}
\]
where the second equality is by the linearity of $\pi_v$ (recall it is merely a reordering of the tensor entries).
Moving on to the parent of $\nu$, multilinearity of the tensor product implies that for all $r \in [ R_{ \parent (\parent ( \nu ))} ]$:
\[
\begin{split}	
\U^{( \parent ( \nu ), r )}_\alpha & =  \pi_{ \parent (\nu) } \left ( \sum\nolimits_{r' = 1}^{R_{ \parent (\nu) } } \weightmat{ \parent (\nu)  }_{r', r} \left [ \left ( \tenp_{\nu_c \in \lsib (\nu)} \tensorpart{ \nu_c }{ r' } \right ) \tenp \U^{(\nu, r')}_{\alpha} \tenp \left ( \tenp_{\nu_c \in \rsib (\nu)} \tensorpart{\nu_c}{r'} \right ) \right ] \right ) \\
& =  \pi_{ \parent (\nu) } \left ( \sum\nolimits_{r' = 1}^{R_{ \parent (\nu) } } \weightmat{ \parent (\nu)  }_{r', r} \left [ \left ( \tenp_{\nu_c \in \lsib (\nu)} \tensorpart{ \nu_c }{ r' } \right ) \tenp \big ( \alpha \cdot \tensorpart{\nu}{r'} \big ) \tenp \left ( \tenp_{\nu_c \in \rsib (\nu)} \tensorpart{\nu_c}{r'} \right ) \right ] \right ) \\
& = \alpha \cdot \pi_{ \parent (\nu) } \left ( \sum\nolimits_{r' = 1}^{R_{ \parent (\nu) } } \weightmat{ \parent (\nu)  }_{r', r} \left [ \tenp_{\nu_c \in \children (\parent (\nu))} \tensorpart{ \nu_c }{ r' } \right ] \right ) \\
& = \alpha \cdot \tensorpart{\parent ( \nu )}{r}
\text{\,.}
\end{split}
\]
An inductive claim over the path from $\nu$ to the root $[N]$ therefore yields:
\be
\tensorendmap \big  ( \big ( \weightmat{ \nu' } \big )_{ \nu' \in \htmodetree \setminus \{ \nu \} }, \alpha \cdot \weightmat{\nu} \big ) = \alpha \cdot \tensorendmap \big  ( \big ( \weightmat{ \nu' } \big )_{ \nu' \in \htmodetree } \big )
\text{\,.}
\label{eq:H_multi_hom}
\ee

\paragraph*{Additivity}
We let $\big ( \U^{ (\nu', r) } \big )_{\nu' \in \htmodetree, r \in [ R_{ \parent (\nu') } ] }$ and $\big ( \U^{ (\nu', r) }_+ \big )_{\nu' \in \htmodetree, r \in [ R_{ \parent (\nu') } ] }$ denote the intermediate tensors produced when computing $\tensorendmap \big  ( \big ( \weightmat{ \nu' } \big )_{ \nu' \in \htmodetree \setminus \{ \nu \} }, \Ubf^{ (\nu) } \big )$ and $\tensorendmap \big  ( \big ( \weightmat{ \nu' } \big )_{ \nu' \in \htmodetree \setminus \{ \nu \} }, \weightmat{\nu} + \Ubf^{ (\nu) } \big )$ according to \eq~\eqref{eq:ht_end_tensor}, respectively.
If $\nu$ is a leaf node, we have that $\U^{(\nu, r)}_+ = \weightmat{\nu}_{:, r} + \Ubf^{(\nu)}_{:, r} =  \tensorpart{\nu}{r} + \U^{ ( \nu, r ) }$ for all $r \in [ R_{\parent (\nu) } ]$.
Otherwise, if $\nu$ is an interior node, we arrive at the same conclusion, \ie~for all $ r \in [ R_{\parent (\nu) } ]$:
\[
\begin{split}
\U^{(\nu, r)}_+ & =  \pi_\nu \left ( \sum\nolimits_{r' = 1}^{R_\nu} \left ( \weightmat{\nu}_{r', r} + \Ubf^{(\nu)}_{r', r} \right ) \left [ \tenp_{\nu_c \in \children ( \nu )} \tensorpart{ \nu_c }{ r' } \right ] \right ) \\
& = \pi_\nu \left ( \sum\nolimits_{r' = 1}^{R_\nu} \weightmat{\nu}_{r', r} \left [ \tenp_{\nu_c \in \children ( \nu )} \tensorpart{ \nu_c }{ r' } \right ] \right )  + \pi_\nu \left ( \sum\nolimits_{r' = 1}^{R_\nu} \Ubf^{(\nu)}_{r', r} \left [ \tenp_{\nu_c \in \children ( \nu )} \tensorpart{ \nu_c }{ r' } \right ] \right ) \\
& = \tensorpart{\nu}{r} + \U^{ ( \nu, r ) }
\text{\,,}
\end{split}
\]
where the second equality is by the linearity of $\pi_\nu$.
Then, for any $r \in [ R_{ \parent ( \parent ( \nu ) ) } ]$:
\[
\begin{split}
\U^{ ( \parent (\nu) , r)}_+ & = \pi_{ \parent (\nu) } \left ( \sum\nolimits_{r' = 1}^{R_{ \parent (\nu) } } \weightmat{ \parent (\nu)  }_{r', r} \left [ \left ( \tenp_{\nu_c \in \lsib (\nu)} \tensorpart{ \nu_c }{ r' } \right ) \tenp \U^{(\nu, r')}_{+} \tenp \left ( \tenp_{\nu_c \in \rsib (\nu)} \tensorpart{\nu_c}{r'} \right ) \right ] \right ) \\
& =  \pi_{ \parent (\nu) } \left ( \sum\nolimits_{r' = 1}^{R_{ \parent (\nu) } } \weightmat{ \parent (\nu)  }_{r', r} \left [ \left ( \tenp_{\nu_c \in \lsib (\nu)} \tensorpart{ \nu_c }{ r' } \right ) \tenp \big ( \tensorpart{\nu}{r'} + \U^{(\nu, r')} \big ) \tenp \left ( \tenp_{\nu_c \in \rsib (\nu)} \tensorpart{\nu_c}{r'} \right ) \right ] \right ) \\
& = \pi_{ \parent (\nu) } \left ( \sum\nolimits_{r' = 1}^{R_{ \parent (\nu) } } \weightmat{ \parent (\nu)  }_{r', r} \left [ \tenp_{\nu_c \in \children (\parent (\nu))} \tensorpart{ \nu_c }{ r' } \right ] \right ) + \\
& \hspace{4.5mm} \pi_{ \parent (\nu) } \left ( \sum\nolimits_{r' = 1}^{R_{ \parent (\nu) } } \weightmat{ \parent (\nu)  }_{r', r} \left [ \left ( \tenp_{\nu_c \in \lsib (\nu)} \tensorpart{ \nu_c }{ r' } \right ) \tenp \U^{(\nu, r')} \tenp \left ( \tenp_{\nu_c \in \rsib (\nu)} \tensorpart{\nu_c}{r'} \right ) \right ] \right ) \\
& = \tensorpart{ \parent (\nu) }{ r } + \U^{ ( \parent (\nu) , r)}	
\text{\,,}
\end{split}
\]
where the penultimate equality is by multilinearity of the tensor product as well as linearity of $\pi_{\parent (\nu)}$.
An induction over the path from $\nu$ to the root thus leads to:
\be
\tensorendmap \big  ( \big ( \weightmat{ \nu' } \big )_{ \nu' \in \htmodetree \setminus \{ \nu \} }, \weightmat{\nu} + \Ubf^{ (\nu) } \big ) = \tensorendmap \big  ( \big ( \weightmat{ \nu' } \big )_{ \nu' \in \htmodetree } \big ) + \tensorendmap \big  ( \big ( \weightmat{ \nu' } \big )_{ \nu' \in \htmodetree \setminus \{ \nu \} }, \Ubf^{ (\nu) } \big )
\text{\,.}
\label{eq:H_multi_add}
\ee

\medskip

\eqs~\eqref{eq:H_multi_hom} and~\eqref{eq:H_multi_add} establish that $\tensorendmap \big ( \big ( \weightmat{ \nu } \big )_{ \nu \in \htmodetree } \big )$ is multilinear.
The proof for $\tensorendmap \big ( \big ( \weightmat{ \nu' } \big )_{ \nu' \in \htmodetree \setminus \subtree (\nu)}, \tensorpart{\nu}{:} \big )$ follows by analogous derivations.

\end{proof}

\begin{lemma}
\label{lem:zero_inter_tensor}
Suppose there exists $\nu \in \htmodetree$ such that $\tensorpart{\nu}{r} = 0$ for all $r \in [ R_{ \parent (\nu) } ]$, where $\tensorpart{\nu}{r}$ is as defined in \eq~\eqref{eq:ht_end_tensor}.
Then, $\tensorend = 0$.
\end{lemma}
\begin{proof}
Since $\tensorpart{\nu}{r} = 0$ for all $r \in [R_{ \parent (\nu ) }]$, for any $r' \in [R_{ \parent ( \parent (\nu ) )} ]$ we have that:
\[
\begin{split}
\tensorpart{\parent (\nu )}{r'} & =  \pi_{ \parent (\nu) } \left ( \sum\nolimits_{r = 1}^{R_{ \parent (\nu) } } \weightmat{ \parent (\nu)  }_{r, r'} \left [ \left ( \tenp_{\nu_c \in \lsib (\nu)} \tensorpart{ \nu_c }{ r } \right ) \tenp \tensorpart{\nu}{r} \tenp \left ( \tenp_{\nu_c \in \rsib (\nu)} \tensorpart{\nu_c}{r} \right ) \right ] \right ) \\
& =  \pi_{ \parent (\nu) } \left ( \sum\nolimits_{r = 1}^{R_{ \parent (\nu) } } \weightmat{ \parent (\nu)  }_{r, r'} \left [ \left ( \tenp_{\nu_c \in \lsib (\nu)} \tensorpart{ \nu_c }{ r } \right ) \tenp 0 \tenp \left ( \tenp_{\nu_c \in \rsib (\nu)} \tensorpart{\nu_c}{r} \right ) \right ] \right ) \\
&  = 0
\text{\,.}
\end{split}
\]
Thus, the claim readily follows by an induction up the path from $\nu$ to the root $[N]$.
\end{proof}

\begin{lemma}
	\label{lem:tensorpart_norm_bound}
	For any $\nu \in \interior (\htmodetree)$ and $r \in [ R_{ \parent (\nu) } ]$:
	\[
	\norm{ \tensorpart{\nu}{r} } \leq \norm{ \weightmat{\nu}_{:, r} } \cdot \prod\nolimits_{\nu_c \in \children (\nu) } \norm{ \tensorpart{\nu_c}{:} }
	\text{\,,}
	\]
	where $\tensorpart{\nu_c}{:}$, for $\nu_c \in \children (\nu)$, is the tensor obtained by stacking $\big ( \tensorpart{\nu_c}{r'} \big )_{r' = 1}^{R_{\nu} }$ into a single tensor, \ie~$\tensorpart{\nu_c}{:}_{:, \ldots, :, r'} = \tensorpart{\nu_c}{r'}$ for all $r' \in [ R_{\nu} ]$.
\end{lemma}
\begin{proof}
	By the definition of $\tensorpart{\nu}{r}$ (\eq~\eqref{eq:ht_end_tensor}) we have that:
	\[
	\begin{split}
		\norm{ \tensorpart{\nu}{r} } &= \norm{ \pi_\nu \left ( \sum\nolimits_{r' = 1}^{R_\nu} \weightmat{\nu}_{r', r} \left [ \tenp_{\nu_c \in \children ( \nu )} \tensorpart{ \nu_c }{ r' } \right ] \right ) } \\[1mm]
		& =  \norm{ \sum\nolimits_{r' = 1}^{R_\nu} \weightmat{\nu}_{r', r} \left [ \tenp_{\nu_c \in \children ( \nu )} \tensorpart{ \nu_c }{ r' } \right ] }
		\text{\,,}
	\end{split}
	\]
	where the second equality is due to the fact that $\pi_\nu$ merely reorders entries of a tensor, and therefore does not alter its Frobenius norm.
	Vectorizing each $\tenp_{\nu_c \in \children (\nu)} \tensorpart{\nu_c}{r'}$, we may write $\norm{ \tensorpart{\nu}{r} }$ as the Frobenius norm of a matrix-vector product:
	\[
	\norm{ \tensorpart{\nu}{r} } = \norm{ \left ( \vectbig{ \tenp_{\nu_c \in \children (\nu)} \tensorpart{\nu_c}{1} } , \ldots,\vectbig{ \tenp_{\nu_c \in \children (\nu)} \tensorpart{\nu_c}{R_\nu} }  \right ) \weightmat{\nu}_{:, r} }
	\text{\,.}
	\]
	Hence, sub-multiplicativity of the Frobenius norm gives:
	\be
	\norm{ \tensorpart{\nu}{r} } \leq \norm{ \weightmat{\nu}_{:, r} } \cdot \norm{  \left ( \vectbig{ \tenp_{\nu_c \in \children (\nu)} \tensorpart{\nu_c}{1} } , \ldots,\vectbig{ \tenp_{\nu_c \in \children (\nu)} \tensorpart{\nu_c}{R_\nu} }  \right ) }
	\text{\,.}
	\label{eq:tensorpart_norm_interm_bound}
	\ee
	Notice that:
	\[
	\begin{split}
		\norm{  \left ( \vectbig{ \tenp_{\nu_c \in \children (\nu)} \tensorpart{\nu_c}{1} } , \ldots,\vectbig{ \tenp_{\nu_c \in \children (\nu)} \tensorpart{\nu_c}{R_\nu} }  \right ) }^2 & = \sum\nolimits_{r' = 1}^{R_\nu} \norm{ \tenp_{\nu_c \in \children (\nu)} \tensorpart{\nu_c}{r'} }^2 \\
		& = \sum\nolimits_{r' = 1}^{R_\nu} \prod\nolimits_{\nu_c \in \children (\nu)} \norm{ \tensorpart{\nu_c}{r'} }^2 \\
		& \leq \prod\nolimits_{\nu_c \in \children (\nu)} \left ( \sum\nolimits_{r' = 1}^{R_\nu} \norm{ \tensorpart{\nu_c}{r'} }^2 \right ) \\
		& =  \prod\nolimits_{\nu_c \in \children (\nu)} \norm{ \tensorpart{\nu_c}{:} }^2
		\text{\,,}
	\end{split}
	\]
	where the second transition is by the fact that the norm of a tensor product is equal to the product of the norms,
	and the inequality is due to $\prod\nolimits_{\nu_c \in \children (\nu)} \big ( \sum\nolimits_{r' = 1}^{R_\nu} \normbig{ \tensorpart{\nu_c}{r'} }^2 \big )$ being a sum of non-negative elements which includes $\sum\nolimits_{r' = 1}^{R_\nu} \prod\nolimits_{\nu_c \in \children (\nu)} \normbig{ \tensorpart{\nu_c}{r'} }^2$.
	Taking the square root of both sides in the equation above and plugging it into \eq~\eqref{eq:tensorpart_norm_interm_bound} completes the proof.
\end{proof}

\begin{lemma}
\label{lem:stacked_tensorpart_norm_bound}
For any $\nu \in \interior (\htmodetree)$:
\[
\norm{ \tensorpart{\nu}{:} } \leq \norm{ \weightmat{\nu} } \cdot  \prod\nolimits_{\nu_c \in \children (\nu) } \norm{ \tensorpart{\nu_c}{:} }
\text{\,,}
\]
where $\tensorpart{\nu_c}{:}$, for $\nu_c \in \children (\nu)$, is the tensor obtained by stacking $\big ( \tensorpart{\nu_c}{r} \big )_{r = 1}^{R_{ \nu } }$ into a single tensor, \ie~$\tensorpart{\nu_c}{:}_{:, \ldots, :, r} = \tensorpart{\nu_c}{r}$ for all $r \in [ R_{ \nu } ]$.
\end{lemma}
\begin{proof}
We may explicitly write $\norm{ \tensorpart{\nu}{:} }^2$ as follows:
\be
\norm{ \tensorpart{\nu}{:} }^2 = \sum\nolimits_{r = 1}^{ R_{\parent (\nu) } } \norm{ \tensorpart{\nu}{r} }^2
\text{\,.}
\label{eq:stacked_tensorpart_norm}
\ee
For each $r \in [ R_{ \parent (\nu) } ]$, by \lem~\ref{lem:tensorpart_norm_bound} we know that:
\[
\norm{ \tensorpart{\nu}{r} }^2 \leq \norm{ \weightmat{\nu}_{:, r} }^2 \cdot \prod\nolimits_{\nu_c \in \children (\nu)} \norm{ \tensorpart{\nu_c}{:} }^2
\text{\,.}
\]
Thus, going back to \eq~\eqref{eq:stacked_tensorpart_norm} we arrive at:
\[
\norm{ \tensorpart{\nu}{:} }^2 \leq  \sum\nolimits_{ r = 1 }^{R_{ \parent (\nu) } } \norm{ \weightmat{\nu}_{:, r} }^2 \cdot \prod\nolimits_{\nu_c \in \children (\nu)} \norm{ \tensorpart{\nu_c}{:} }^2 = \norm{ \weightmat{\nu} }^2 \cdot \prod\nolimits_{\nu_c \in \children (\nu)} \norm{ \tensorpart{\nu_c}{:} }^2
\text{\,.}
\]
Taking the square root of both sides concludes the proof.
\end{proof}

\begin{lemma}
\label{lem:ht_weightmat_grad}
For any  $\nu \in \htmodetree$ and $\Delta \in \R^{R_\nu, R_{\parent (\nu)}}$:
\[
\inprod{ \frac{ \partial }{ \partial \weightmat{\nu} } \htfobj \big ( \big ( \weightmat{\nu'} \big )_{\nu' \in \htmodetree} \big ) }{ \Delta } = \inprod{ \nabla \htfendloss ( \tensorend ) }{ \tensorendmap \Big  ( \big ( \weightmat{ \nu' } \big )_{ \nu' \in \htmodetree \setminus \{ \nu \} }, \Delta \Big ) }
\text{\,.}
\]
\end{lemma}
\begin{proof}
We treat $\big ( \weightmat{\nu'} \big )_{\nu '\in \htmodetree \setminus{ \{ \nu \}} }$ as fixed, and with slight abuse of notation consider:
\[
\htfobj^{ ( \nu ) } \big ( \weightmat{\nu}  \big ) := \htfobj \big ( \big ( \weightmat{\nu'} \big )_{\nu' \in \htmodetree} \big )
\text{\,.}
\]
For $\Delta \in \R^{R_\nu, R_{ \parent (\nu) } }$, by multilinearity of $\tensorendmap$ (\lem~\ref{lem:ht_multilinear}) we have that:
\[
\begin{split}
\htfobj^{ ( \nu ) } \big ( \weightmat{\nu}  + \Delta \big ) & = \htfendloss \left ( \tensorendmap \Big  ( \big ( \weightmat{ \nu' } \big )_{ \nu' \in \htmodetree \setminus \{ \nu \} },  \weightmat{\nu}  + \Delta \Big ) \right ) \\
& =  \htfendloss \left ( \tensorend + \tensorendmap \Big  ( \big ( \weightmat{ \nu' } \big )_{ \nu' \in \htmodetree \setminus \{ \nu \} }, \Delta \Big ) \right ) 
\text{\,.}
\end{split}
\]
According to the  first order Taylor approximation of $\htfendloss$ we may write:
\[
\begin{split}
\htfobj^{ ( \nu ) } \big ( \weightmat{\nu}  + \Delta \big ) & =  \htfendloss \left ( \tensorend  \right ) + \inprod{ \nabla \htfendloss \left ( \tensorend \right )  }{ \tensorendmap \Big  ( \big ( \weightmat{ \nu' } \big )_{ \nu' \in \htmodetree \setminus \{ \nu \} }, \Delta \Big ) } + o \left ( \norm{ \Delta } \right ) \\
& = \htfobj^{ ( \nu ) } \big ( \weightmat{\nu}  \big ) + \inprod{ \nabla \htfendloss \left ( \tensorend \right )  }{ \tensorendmap \Big  ( \big ( \weightmat{ \nu' } \big )_{ \nu' \in \htmodetree \setminus \{ \nu \} }, \Delta \Big ) } + o \left ( \norm{ \Delta } \right ) 
\text{\,.}
\end{split}
\]
The term $\inprodbig{ \nabla \htfendloss \left ( \tensorend \right )  }{ \tensorendmap \big  ( ( \weightmat{ \nu' } )_{ \nu' \in \htmodetree \setminus \{ \nu \} }, \Delta \big ) }$ is a linear function of $\Delta$.
Therefore, uniqueness of the linear approximation of $\htfobj^{ (\nu) }$ at $\weightmat{\nu}$ implies:
\[
\inprod{ \frac{ d }{ d \weightmat{\nu} } \htfobj^{(\nu)} \big ( \weightmat{\nu} \big ) }{ \Delta } = \inprod{ \nabla \htfendloss ( \tensorend ) }{ \tensorendmap \Big  ( \big ( \weightmat{ \nu' } \big )_{ \nu' \in \htmodetree \setminus \{ \nu \} }, \Delta \Big ) }
\text{\,.}
\]	
Noticing that $\frac{ \partial }{ \partial \weightmat{\nu} } \htfobj \big ( \big ( \weightmat{\nu'} \big )_{\nu' \in \htmodetree} \big ) =  \frac{ d }{ d \weightmat{\nu} } \htfobj^{(\nu)} \big ( \weightmat{\nu} \big )$ completes the proof.
\end{proof}

\begin{lemma}
\label{lem:ht_weightvec_grad}
For any $\nu \in \interior(\htmodetree), r \in [R_{ \nu }]$, and $ \Delta \in \R^{ R_{ \parent ( \nu )} }$:
\be
\inprodbigg{ \frac{ \partial }{ \partial \weightmat{\nu}_{r, :} } \htfobj \big ( \big ( \weightmat{\nu'} \big )_{\nu' \in \htmodetree} \big ) }{ \Delta^\top } = \inprod{ \nabla \htfendloss ( \tensorend ) }{ \tensorendmap \Big  ( \big ( \weightmat{ \nu' } \big )_{ \nu' \in \htmodetree \setminus \{ \nu  \} }, \padrow_r \left ( \Delta^\top \right ) \Big ) }
\text{\,,}
\label{eq:ht_row_weightvec_grad}
\ee
where $\padrow_r \left ( \Delta^\top \right ) \in \R^{ R_{ \nu  }, R_{ \parent ( \nu ) }}$ is the matrix whose $r$'th row is $\Delta^\top$, and all the rest are zero.
Furthermore, for any $\nu_c \in \children (\nu)$ and $\Delta \in \R^{R_{\nu_c}}$:
\be
\inprodbigg{ \frac{ \partial }{ \partial \weightmat{\nu_c}_{:, r} } \htfobj \big ( \big ( \weightmat{\nu'} \big )_{\nu' \in \htmodetree} \big ) }{ \Delta } = \inprod{ \nabla \htfendloss ( \tensorend ) }{ \tensorendmap \Big ( \big ( \weightmat{ \nu' } \big )_{ \nu' \in \htmodetree \setminus \{ \nu_c \} }, \padcol_r ( \Delta ) \Big ) }
\text{\,,}
\label{eq:ht_col_weightvec_grad}
\ee
where $\padcol_r ( \Delta ) \in \R^{ R_{\nu_c}, R_{ \nu }}$ is the matrix whose $r$'th column is $\Delta$, and all the rest are zero.
\end{lemma}
\begin{proof}
\eqs~\eqref{eq:ht_row_weightvec_grad} and~\eqref{eq:ht_col_weightvec_grad} are direct implications of Lemma~\ref{lem:ht_weightmat_grad} since:
\[
\begin{split}
 \inprodbigg{ \frac{ \partial }{ \partial \weightmat{\nu}_{r, :} } \htfobj \big ( \big ( \weightmat{\nu'} \big )_{\nu' \in \htmodetree} \big ) }{ \Delta^\top } & = \inprodbigg{ \frac{ \partial }{ \partial \weightmat{\nu} } \htfobj \Big ( \big ( \weightmat{\nu'} \big )_{\nu' \in \htmodetree} \Big ) }{  \padrow_r \left (\Delta^\top \right )}
\end{split}
\text{\,,}
\]
and
\[
\begin{split}
	\inprodbigg{ \frac{ \partial }{ \partial \weightmat{\nu_c}_{:, r} } \htfobj \big ( \big ( \weightmat{\nu'} \big )_{\nu' \in \htmodetree} \big ) }{ \Delta } & = \inprodbigg{ \frac{ \partial }{ \partial \weightmat{\nu_c} } \htfobj \Big ( \big ( \weightmat{\nu'} \big )_{\nu' \in \htmodetree} \Big ) }{  \padcol_r ( \Delta ) }
	\text{\,.}
\end{split}
\]
\end{proof}

\begin{lemma}
	\label{lem:ht_weightvec_grad_zero_comp}
	Let $\nu \in \interior(\htmodetree)$, $\nu_c \in \children (\nu)$, and $r \in [R_{ \nu }]$.
	If both $\weightmat{\nu}_{r, :} = 0$ and $\weightmat{\nu_c}_{:, r} = 0$, then:
	\be
	\frac{ \partial }{ \partial \weightmat{\nu}_{r, :} } \htfobj \big ( \big ( \weightmat{\nu'} \big )_{\nu' \in \htmodetree} \big ) = 0
	\text{\,,}
	\label{eq:ht_row_weightvec_grad_zero}
	\ee
	and
	\be
	\frac{ \partial }{ \partial \weightmat{\nu_c}_{:, r} } \htfobj \big ( \big ( \weightmat{\nu'} \big )_{\nu' \in \htmodetree} \big ) = 0
	\text{\,.}
	\label{eq:ht_col_weightvec_grad_zero}
	\ee
\end{lemma}
\begin{proof}
We show that $\tensorendmap \big  ( \big ( \weightmat{ \nu' } \big )_{ \nu' \in \htmodetree \setminus \{ \nu  \} }, \padrow_r \left ( \Delta^\top \right ) \big ) = 0$ for all $\Delta \in \R^{\parent(\nu)}$.
\eq~\eqref{eq:ht_row_weightvec_grad_zero} then follows from \eq~\eqref{eq:ht_row_weightvec_grad} in \lem~\ref{lem:ht_weightvec_grad}.
Fix some $\Delta \in \R^{\parent(\nu)}$ and let $\big ( \U^{ (\nu', r' ) } \big )_{\nu' \in \htmodetree, r' \in [ R_{ \parent (\nu') } ] }$ be the intermediate tensors produced when computing $\tensorendmap \big  ( \big ( \weightmat{ \nu' } \big )_{ \nu' \in \htmodetree \setminus \{ \nu  \} }, \padrow_r \left ( \Delta^\top \right ) \big )$ according to \eq~\eqref{eq:ht_end_tensor} (there denoted $\big ( \tensorpart{\nu'}{ r' } \big )_{\nu', r'}$).
For any $\bar{r} \in [ R_{\parent (\nu)} ]$ we have that:
\[
\U^{(\nu, \bar{r})} = \pi_{\nu} \left ( \sum\nolimits_{r' = 1}^{R_\nu} \padrow_r \big ( \Delta^\top \big )_{r', \bar{r}} \left [  \tenp_{\nu' \in \children (\nu)} \tensorpart{\nu'}{ r' }  \right ] \right ) =  \pi_\nu \left ( \Delta_{ \bar{r} } \left [  \tenp_{\nu' \in \children (\nu)} \tensorpart{\nu'}{r}  \right ] \right )
\text{\,.}
\]
The fact that $\weightmat{\nu_c}_{:, r} = 0$ implies that $\tensorpart{\nu_c}{r} := \pi_{\nu_c} \big ( \sum\nolimits_{r' = 1}^{R_{\nu_c}} \weightmat{\nu_c}_{r', r} \big [ \tenp_{\nu' \in \children ( \nu )} \tensorpart{ \nu' }{ r' } \big ] \big ) = 0$, and so for every $r' \in [ R_{\parent(\nu)} ]$:
\[
\U^{(\nu, \bar{r})} = \pi_\nu \left ( \Delta_{ \bar{r} } \left [ \left ( \tenp_{\nu' \in \lsib (\nu_c)} \tensorpart{\nu'}{r} \right ) \tenp 0 \tenp \left ( \tenp_{\nu' \in \rsib (\nu_c)} \tensorpart{\nu'}{r}  \right ) \right ] \right ) = 0
\text{\,.}
\]
\lem~\ref{lem:zero_inter_tensor} then gives $\tensorendmap \big  ( ( \weightmat{ \nu' } )_{ \nu' \in \htmodetree \setminus \{ \nu  \} }, \padrow_r \left ( \Delta^\top \right ) \big ) = 0$, completing this part of the proof.

\medskip

Next, we show that $\tensorendmap \big ( \big ( \weightmat{ \nu' } \big )_{ \nu' \in \htmodetree \setminus \{ \nu_c \} }, \padcol_r ( \Delta ) \big ) = 0$ for all $\Delta \in \R^{\nu_c}$.
\eq~\eqref{eq:ht_col_weightvec_grad} in \lem~\ref{lem:ht_weightvec_grad} then yields \eq~\eqref{eq:ht_col_weightvec_grad_zero}.
Fix some $\Delta \in \R^{\nu_c}$ and let $\big ( \V^{ (\nu', r') } \big )_{\nu' \in \htmodetree, r' \in [ R_{ \parent (\nu') } ] }$ be the intermediate tensors produced when computing $\tensorendmap \big ( \big ( \weightmat{ \nu' } \big )_{ \nu' \in \htmodetree \setminus \{ \nu_c \} }, \padcol_r ( \Delta ) \big )$ according to \eq~\eqref{eq:ht_end_tensor}.
For any $\bar{r} \in [ R_{\nu} ] \setminus \{ r\}$:
\[
\U^{(\nu_c, \bar{r})} = \pi_{\nu_c} \left ( \sum\nolimits_{ r' = 1}^{R_{\nu_c}}  \padcol_r \big ( \Delta \big )_{r' , \bar{r}}  \left [  \tenp_{\nu' \in \children (\nu_c)} \tensorpart{\nu'}{ r' }  \right ] \right ) = \pi_{\nu_c} \left ( \sum\nolimits_{ r' = 1}^{R_{\nu_c}} 0 \cdot  \left [  \tenp_{\nu' \in \children (\nu_c)} \tensorpart{\nu'}{ r' }  \right ] \right ) = 0
\text{\,.}
\]
Thus, for any $\hat{r} \in [ R_{ \parent( \nu) } ]$ we may write $\U^{(\nu, \hat{r})} = \pi_{\nu} \big ( \weightmat{\nu}_{r, \hat{r}} \big [  \tenp_{\nu' \in \children (\nu)} \U^{(\nu', r)}  \big ] \big )$.
Since $\weightmat{\nu}_{r, :} = 0$, we get that $\U^{(\nu, \hat{r})} = 0$ for all $\hat{r} \in [ R_{ \parent( \nu) } ]$, which by \lem~\ref{lem:zero_inter_tensor} leads to $\tensorendmap \big ( \big ( \weightmat{ \nu' } \big )_{ \nu' \in \htmodetree \setminus \{ \nu_c \} }, \padcol_r ( \Delta ) \big ) = 0$.
\end{proof}

\begin{lemma}
\label{lem:tensorend_comp_eq_zeroing_cols_or_rows}
For any $\nu \in \interior (\htmodetree)$, $\nu_c \in \children (\nu)$, and $r \in [ R_{\nu}]$, the following hold:
\be
\tensorendmap \big  ( \big ( \weightmat{ \nu' } \big )_{ \nu' \in \htmodetree \setminus \{ \nu  \} }, \padrow_r \big ( \weightmat{\nu}_{r, :} \big ) \big ) = \htfcompnorm{\nu}{r} \cdot \htfcomp{ \nu }{ r }
\text{\,,}
\label{eq:tensorend_zero_rows_loc_comp}
\ee
and
\be
\tensorendmap \big ( \big ( \weightmat{ \nu' } \big )_{ \nu' \in \htmodetree \setminus \{ \nu_c \} }, \padcol_r \big (  \weightmat{ \nu_c }_{:, r} \big ) \big ) = \htfcompnorm{\nu}{r} \cdot \htfcomp{ \nu }{ r }
\text{\,,}
\label{eq:tensorend_zero_cols_loc_comp}
\ee
where $\padrow_r \left ( \Delta^\top \right ) \in \R^{ R_{ \nu  }, R_{ \parent ( \nu ) }}$ is the matrix whose $r$'th row is $\Delta^\top$, and all the rest are zero,
$\padcol_r ( \Delta ) \in \R^{ R_{\nu_c}, R_{ \nu }}$ is the matrix whose $r$'th column is $\Delta$, and all the rest are zero, and $\htfcomp{\nu}{r}$ is as defined in \thm~\ref{thm:loc_comp_norm_bal_dyn}.
\end{lemma}

\begin{proof}
Starting with \eq~\eqref{eq:tensorend_zero_rows_loc_comp}, let $\big ( \U^{ (\nu', r') } \big )_{\nu' \in \htmodetree, r' \in [ R_{ \parent (\nu') } ] }$ be the intermediate tensors formed when computing $\tensorendmap \big ( \big ( \weightmat{ \nu' } \big )_{ \nu' \in \htmodetree \setminus \{ \nu  \} }, \padrow_r \big ( \weightmat{\nu}_{r, :} \big ) \big )$ according to \eq~\eqref{eq:ht_end_tensor} (there denoted $\big ( \tensorpart{\nu'}{ r' } \big )_{\nu', r' }$).
Clearly, for any $\nu' \in \subtree (\nu) \setminus \{ \nu \}$~---~a node in the subtree of $\nu$ which is not $\nu$~---~it holds that $\U^{(\nu', r')} = \tensorpart{\nu'}{r'}$ for all $r' \in [ R_{\parent (\nu')} ]$.
Thus, for all $r' \in [ R_{ \parent(\nu) } ]$ we have that:
\be
\U^{(\nu, r')} = \pi_{\nu} \left ( \sum\nolimits_{\bar{r} = 1}^{R_\nu} \padrow_r \big ( \weightmat{\nu}_{r, :} \big )_{\bar{r}, r'} \left [  \tenp_{\nu' \in \children (\nu)} \tensorpart{\nu'}{\bar{r}}  \right ] \right ) =  \pi_{\nu} \left ( \weightmat{\nu}_{r, r'} \left [  \tenp_{\nu' \in \children (\nu)} \tensorpart{\nu'}{r}  \right ] \right )
\text{\,.}
\label{eq:tensorpart_U}
\ee
If $\htfcompnorm{\nu}{r} = \normbig{ \weightmat{\nu}_{r, :} \tenp \big ( \tenp_{\nu' \in \children (\nu)} \weightmat{\nu'}_{:, r} \big ) } = \normbig{\weightmat{\nu}_{r, :}} \prod_{\nu' \in \children (\nu)} \normbig{\weightmat{\nu'}_{:, r}} = 0$, then either $\weightmat{\nu}_{r, :} = 0$ or $\weightmat{\nu'}_{:, r} = 0$ for some $\nu' \in \children (\nu)$.
We claim that in both cases $\U^{(\nu, r')} = 0$ for all $r' \in [R_{\parent(\nu)}]$.
Indeed, if $\weightmat{\nu}_{r, :} = 0$ this immediately follows from \eq~\eqref{eq:tensorpart_U}.
On the other hand, if $\weightmat{\nu'}_{:, r} = 0$ for some $\nu' \in \children (\nu)$, then $\tensorpart{\nu'}{r} = 0$, which combined with \eq~\eqref{eq:tensorpart_U} also implies that $\U^{(\nu, r')} = 0$ for all $r' \in [R_{\parent(\nu)}]$.
Hence, \lem~\ref{lem:zero_inter_tensor} establishes \eq~\eqref{eq:tensorend_zero_cols_loc_comp} for the case of $\htfcompnorm{\nu}{r} = 0$:
\[
\tensorendmap \Big  ( \big ( \weightmat{ \nu' } \big )_{ \nu' \in \htmodetree \setminus \{ \nu  \} }, \padrow_r \big ( \weightmat{\nu}_{r, :} \big ) \Big ) = 0 = \htfcompnorm{\nu}{r} \cdot \htfcomp{\nu}{r}
\text{\,.}
\]
Now, suppose that $\htfcompnorm{\nu}{r} \neq 0$ and let $\U^{(\nu, :)}$ be the tensor obtained by stacking $\big ( \U^{(\nu, r') } \big )_{r' = 1}^{R_{ \parent (\nu) }}$ into a single tensor, \ie~$\U^{(\nu, :)}_{:, \ldots, :, r'} = \U^{(\nu, r')}$ for all $r' \in [R_{\parent(\nu)}]$.
Multilinearity of $\tensorendmap \big ( \big ( \weightmat{ \nu' } \big )_{ \nu' \in \htmodetree \setminus \subtree (\nu) }, \U^{(\nu, :)} \big )$ (\lem~\ref{lem:ht_multilinear}) leads to:
\be
\begin{split}
	\tensorendmap \Big  ( \big ( \weightmat{ \nu' } \big )_{ \nu' \in \htmodetree \setminus \{ \nu  \} }, \padrow_r \big ( \weightmat{\nu}_{r, :} \big ) \Big ) & = \tensorendmap \Big ( \big ( \weightmat{ \nu' } \big )_{ \nu' \in \htmodetree \setminus \subtree (\nu) }, \U^{(\nu, :)} \Big ) \\
	& = \htfcompnorm{\nu}{r} \cdot \tensorendmap \Big ( \big ( \weightmat{ \nu' } \big )_{ \nu' \in \htmodetree \setminus \subtree (\nu) }, \big ( \htfcompnorm{\nu}{r} \big )^{-1} \U^{(\nu, :)} \Big ) 
	\text{\,.}
\end{split}
\label{eq:tensorend_zero_cols_loc_comp_intermid_I}
\ee
From \eq~\eqref{eq:tensorpart_U} we know that $\big ( \htfcompnorm{\nu}{r} \big )^{-1} \U^{(\nu, :)}_{:,\ldots, :, r'} = \pi_\nu \big ( \big ( \htfcompnorm{\nu}{r} \big )^{-1} \weightmat{\nu}_{r, r'} \big [  \tenp_{\nu' \in \children (\nu)} \tensorpart{\nu'}{r}  \big ] \big )$ for all $r' \in [R_{\parent(\nu)}]$.
Thus, by the definition of $\htfcomp{\nu}{r}$ we may conclude that:
\be
\tensorendmap \Big ( \big ( \weightmat{ \nu' } \big )_{ \nu' \in \htmodetree \setminus  \subtree (\nu) }, \big ( \htfcompnorm{\nu}{r} \big )^{-1} \U^{(\nu, :)} \Big ) = \htfcomp{\nu}{r}
\text{\,.}
\label{eq:tensorend_zero_cols_loc_comp_intermid_II}
\ee
Combining \eqs~\eqref{eq:tensorend_zero_cols_loc_comp_intermid_I} and~\eqref{eq:tensorend_zero_cols_loc_comp_intermid_II} yields \eq~\eqref{eq:tensorend_zero_cols_loc_comp}, completing this part of the proof.
	
\medskip

Turning our attention to \eq~\eqref{eq:tensorend_zero_cols_loc_comp}, let $\big ( \V^{ (\nu', r') } \big )_{\nu' \in \htmodetree, r' \in [ R_{ \parent (\nu') } ] }$ be the intermediate tensors produced when computing $\tensorendmap \big ( \big ( \weightmat{ \nu' } \big )_{ \nu' \in \htmodetree \setminus \{ \nu_c \} }, \padcol_r \big (  \weightmat{ \nu_c }_{:, r} \big ) \big )$ according to \eq~\eqref{eq:ht_end_tensor} (there denoted $\big ( \tensorpart{\nu'}{ r' } \big )_{\nu', r' }$).
Clearly, for any $\nu' \in \subtree ( \nu ) \setminus \{ \nu, \nu_c \}$~---~a node in the subtree of $\nu$ which is not $\nu$ nor $\nu_c$~---~it holds that $\V^{(\nu', r')} = \tensorpart{\nu'}{r'}$ for all $r' \in [R_{\parent (\nu')}]$.
Thus, $\V^{ (\nu_c, r) } = \pi_{\nu_c} \big ( \sum\nolimits_{\bar{r} = 1}^{R_{\nu_c}} \padcol_r \big ( \weightmat{\nu_c}_{:, r}\big )_{\bar{r}, r} \big [  \tenp_{\nu' \in \children (\nu_c)} \tensorpart{\nu'}{\bar{r}}  \big ] \big ) = \tensorpart{\nu_c}{r}$, whereas for any $r' \in [ R_{\nu} ] \setminus \{r\}$:
\[
\begin{split}
\V^{ (\nu_c, r') } & = 
\pi_{\nu_c} \left ( \sum\nolimits_{\bar{r} = 1}^{R_{\nu_c}} \padcol_r \big ( \weightmat{\nu_c}_{:, r}\big )_{\bar{r}, r'} \left [  \tenp_{\nu' \in \children (\nu_c)} \tensorpart{\nu'}{\bar{r}}  \right ] \right ) \\
& = \pi_{\nu} \left ( \sum\nolimits_{\bar{r} = 1}^{R_{\nu_c}} 0 \cdot \left [  \tenp_{\nu' \in \children (\nu_c)} \tensorpart{\nu'}{\bar{r}}  \right ] \right ) \\
& = 0
\text{\,.}
\end{split}
\]
Putting it all together, for any $r' \in [R_{\parent(\nu)}]$ we may write:
\[
\V^{(\nu, r')} = \pi_{\nu} \left ( \sum\nolimits_{\bar{r} = 1}^{R_\nu} \weightmat{\nu}_{\bar{r}, r'} \left [  \tenp_{\nu' \in \children (\nu)} \U^{(\nu', \bar{r})}  \big ] \right ) = \pi_{\nu} \left ( \weightmat{\nu}_{r, r'} \right [  \tenp_{\nu' \in \children (\nu)} \tensorpart{\nu'}{r}  \big ] \right )
\text{\,.}
\]
From this point, following steps analogous to those used for proving \eq~\eqref{eq:tensorend_zero_rows_loc_comp} based on \eq~\eqref{eq:tensorpart_U} yields \eq~\eqref{eq:tensorend_zero_cols_loc_comp}.
\end{proof}

\begin{lemma}
\label{lem:weightvec_sq_norm_time_deriv}
For any $\nu \in \interior (\htmodetree)$, $\nu_c \in \children (\nu)$, and $r \in [ R_{\nu}]$:
\[
\frac{d}{dt} \norm { \weightmat{\nu}_{r, :} (t) }^2 = 2 \htfcompnorm{\nu}{r} (t) \inprod{ - \nabla \htfendloss ( \tensorend (t)) }{ \htfcomp{\nu}{r} (t) } = \frac{d}{dt} \norm{ \weightmat{\nu_c}_{:,r} (t) }^2 
\text{\,,}
\]
where $\htfcomp{\nu}{r} (t)$ is as defined in \thm~\ref{thm:loc_comp_norm_bal_dyn}.
\end{lemma}
\begin{proof}
Differentiating $\normbig{ \weightmat{\nu_c}_{:, r} (t) }^2$ with respect to time we get:
\[
\frac{ d }{ dt } \norm{\weightmat{ \nu_c}_{:, r} (t) }^2 = 2 \inprod{ \weightmat{ \nu_c }_{:, r} (t) }{ \frac{d}{dt} \weightmat{\nu_c}_{:, r} (t) } = - 2 \inprod{  \weightmat{ \nu_c }_{:, r} (t) }{ \frac{ \partial }{ \partial \weightmat{\nu_c}_{:, r} } \htfobj \big ( \big ( \weightmat{\nu'} (t) \big )_{\nu' \in \htmodetree} \big ) }
\text{\,.}
\]
By \eq~\eqref{eq:ht_col_weightvec_grad} from \lem~\ref{lem:ht_weightvec_grad} we have that:
\[
\frac{ d }{ dt } \norm{\weightmat{ \nu_c}_{:, r} (t) }^2 = - 2 \inprod{ \nabla \htfendloss ( \tensorend (t) ) }{ \tensorendmap \Big ( \big ( \weightmat{ \nu' } (t) \big )_{ \nu' \in \htmodetree \setminus \{ \nu_c \} }, \padcol_r \big (  \weightmat{ \nu_c }_{:, r} (t) \big ) \Big ) }
\text{\,.}
\]
Then, applying \eq~\eqref{eq:tensorend_zero_cols_loc_comp} from \lem~\ref{lem:tensorend_comp_eq_zeroing_cols_or_rows} concludes:
\[
\frac{d}{dt} \norm{ \weightmat{\nu_c}_{:,r} (t) }^2 = 2 \htfcompnorm{\nu}{r} (t) \inprod{ - \nabla \htfendloss ( \tensorend (t)) }{ \htfcomp{\nu}{r} (t) }
\text{\,.}
\]

A similar argument yields the desired result for $\normbig{ \weightmat{\nu}_{r, :} (t) }^2$.
Differentiating with respect to time we obtain:
\[
\frac{d}{dt} \norm { \weightmat{\nu}_{r, :} (t) }^2 = 2 \inprod{ \weightmat{ \nu }_{r, :} (t) }{ \frac{d}{dt} \weightmat{\nu}_{r, :} (t) } = - 2 \inprod{  \weightmat{ \nu }_{r, :} (t) }{ \frac{ \partial }{ \partial \weightmat{\nu}_{r, :} } \htfobj \big ( \big ( \weightmat{\nu'} (t) \big )_{\nu' \in \htmodetree} \big ) }
\text{\,.}
\]
By \eq~\eqref{eq:ht_row_weightvec_grad} from \lem~\ref{lem:ht_weightvec_grad} we may write:
\[
\frac{d}{dt} \norm { \weightmat{\nu}_{r, :} (t) }^2 = - 2 \inprod{ \nabla \htfendloss ( \tensorend (t) ) }{ \tensorendmap \Big ( \big ( \weightmat{ \nu' } (t) \big )_{ \nu' \in \htmodetree \setminus \{ \nu \} }, \padrow_r \big (  \weightmat{ \nu }_{r, :} (t) \big ) \Big ) }
\text{\,.}
\]
Lastly, applying \eq~\eqref{eq:tensorend_zero_rows_loc_comp} from \lem~\ref{lem:tensorend_comp_eq_zeroing_cols_or_rows} completes the proof:
\[
\frac{d}{dt} \norm { \weightmat{\nu}_{r, :} (t) }^2 = 2 \htfcompnorm{\nu}{r} (t) \inprod{ - \nabla \htfendloss ( \tensorend (t)) }{ \htfcomp{\nu}{r} (t) }
\text{\,.}
\]
\end{proof}

\begin{lemma}
\label{lem:bal_zero_stays_zero}
Let $\nu \in \interior (\htmodetree)$ and $r \in [ R_\nu ]$.
If there exists a time $t_0 \geq 0$ at which $\wbf (t_0) = 0$ for all $\wbf \in \localcomp (\nu, r)$, then:
\[
\wbf (t) = 0 \quad , t \geq 0 ~,~ \wbf \in \localcomp (\nu, r)
\text{\,,}
\]
\ie~$\wbf (t)$ is identically zero for all $\wbf \in \localcomp( \nu, r)$.
\end{lemma}
\begin{proof}
Standard existence and uniqueness theorems (\eg~Theorem 2.2 in~\citet{teschl2012ordinary}) imply that the system of differential equations governing gradient flow over $\htfobj$ (\eq~\eqref{eq:gf_htf}) has a unique solution that passes through $\big ( \weightmat{\nu'} (t_0) \big )_{\nu' \in \htmodetree}$ at time $t_0$.
It therefore suffices to show that there exist $\big (\widebar{\Wbf}^{(\nu')} (t) \big )_{\nu' \in \htmodetree}$ satisfying \eq~\eqref{eq:gf_htf} such that $\widebar{\Wbf}^{(\nu')} (t_0) = \weightmat{\nu'} (t_0)$ for all $\nu' \in \htmodetree$, for which $\widebar{\Wbf}^{ (\nu) }_{r, :} (t)$ and $\big ( \widebar{\Wbf}^{(\nu_c)}_{:, r} (t) \big )_{\nu_c \in \children (\nu)}$ are zero for all $t \geq 0$ (recall that $\localcomp (\nu, r)$ consists of $\weightmat{\nu}_{r, :}$ and $\big ( \weightmat{\nu_c}_{:, r} \big )_{\nu_c \in \children (\nu)}$).

We denote by $\Theta_{\nu, r} (t)$ all factorization weights at time $t \geq 0$, except for those in $\localcomp (\nu, r)$, \ie:
\[
\Theta_{\nu , r } (t) := \big ( \Wbf^{(\nu')} (t)  \big )_{\nu' \in \htmodetree \setminus \left ( \{ \nu \} \cup \children (\nu) \right ) } \cup \big ( \Wbf^{(\nu_c)}_{:, r'} (t) \big )_{\nu_c \in \children (\nu), r' \in [R_\nu] \setminus \{ r\} } \cup \big ( \Wbf^{(\nu)}_{r', :} (t) \big )_{r' \in [R_\nu] \setminus \{ r \}}
\text{\,.}
\]
We construct $\big (\widebar{\Wbf}^{(\nu')} (t) \big )_{\nu' \in \htmodetree}$ as follows.
First, let $\widebar{\Wbf}^{ (\nu) }_{r, :} (t) := 0$ and $\widebar{\Wbf}^{(\nu_c)}_{:, r} (t) := 0$ for all $\nu_c \in \children (\nu)$ and $t \geq 0$.
Then, considering $\Wbf^{ (\nu) }_{r, :} (t)$ and $\big ( \Wbf^{(\nu_c)}_{:, r} (t) \big )_{\nu_c \in \children (\nu)}$ as fixed to zero, we denote by $\widebar{\phi}_{HT} ( \Theta_{\nu, r} (t) )$ the induced objective over all other weights, and let
\[
\widebar{\Theta}_{\nu , r } (t) := \big ( \widebar{\Wbf}^{(\nu')} (t)  \big )_{\nu' \in \htmodetree \setminus \left ( \{ \nu \} \cup \children (\nu) \right ) } \cup \big ( \widebar{\Wbf}^{(\nu_c)}_{:, r'} (t) \big )_{\nu_c \in \children (\nu) , r' \in [R_\nu] \setminus \{ r\}} \cup \big ( \widebar{\Wbf}^{(\nu)}_{r', :} (t) \big )_{r' \in [R_\nu] \setminus \{ r \}}
\text{\,}
\]
be a gradient flow path over $\widebar{\phi}_{HT}$ satisfying $\widebar{\Theta}_{\nu, r} (t_0) =\Theta_{\nu, r} (t_0)$.
By definition, it holds that $\widebar{ \Wbf }^{(\nu')} (t_0) = \weightmat{\nu'} (t_0)$ for all $\nu' \in \htmodetree$.
Thus, it remains to show that $\big ( \widebar{\Wbf}^{(\nu')} (t) \big )_{\nu' \in \htmodetree}$ obey the differential equations defining gradient flow over $\htfobj$ (\eq~\eqref{eq:gf_htf}).
To see it is so, notice that since $\widebar{\Wbf}^{ (\nu) }_{r, :} (t)$ and $\big ( \widebar{\Wbf}^{(\nu_c)}_{:, r} (t) \big )_{\nu_c \in \children (\nu)}$ are identically zero, by the definition of $\widebar{\phi}_{HT}$ we have that:
\be
\frac{d}{dt} \widebar{\Theta}_{\nu, r} (t) = - \frac{ d }{ d \Theta_{\nu, r} } \widebar{\phi}_{HT} \left ( \widebar{\Theta}_{\nu, r} (t) \right ) = - \frac{ \partial }{ \partial \Theta_{\nu, r} } \htfobj \big (  \big ( \widebar{\Wbf}^{(\nu')} (t) \big )_{\nu' \in \htmodetree} \big ) 
\text{\,.}
\label{eq:bar_all_but_local_comp_time_deriv}
\ee
Furthermore, by \lem~\ref{lem:ht_weightvec_grad_zero_comp} we obtain:
\be
\frac{d}{dt} \widebar{\Wbf}^{ (\nu) }_{r, :} (t) = 0 = - \frac{ \partial }{ \partial \weightmat{\nu}_{r, :} } \htfobj \big ( \big ( \widebar{\Wbf}^{(\nu')} (t) \big )_{\nu' \in \htmodetree} \big )
\text{\,,}
\label{eq:bar_row_zero_time_deriv}
\ee
and for all $\nu_c \in \children (\nu)$:
\be
\frac{d}{dt} \widebar{\Wbf}^{ (\nu_c) }_{:, r} (t) = 0 = - \frac{ \partial }{ \partial \weightmat{\nu_c}_{:, r} } \htfobj \big ( \big ( \widebar{\Wbf}^{(\nu')} (t) \big )_{\nu' \in \htmodetree} \big )
\text{\,.}
\label{eq:bar_col_zero_time_deriv}
\ee
Combining \eqs~\eqref{eq:bar_all_but_local_comp_time_deriv},~\eqref{eq:bar_row_zero_time_deriv}, and~\eqref{eq:bar_col_zero_time_deriv}, completes the proof:
\[
\frac{d}{dt} \widebar{ \Wbf }^{(\nu')} (t) = - \frac{ \partial }{ \partial \weightmat{\nu'} } \htfobj \big ( \big ( \widebar{\Wbf}^{(\bar{\nu})} (t) \big )_{\bar{\nu} \in \htmodetree} \big ) \quad , \nu' \in \htmodetree
\text{\,.}
\]
\end{proof}

\subsection{Proof of Lemma~\ref{lem:loc_comp_ht_rank_bound}}
\label{app:proofs:loc_comp_ht_rank_bound}
The proof follows a line similar to that of \thm~7 in~\citet{cohen2018boosting}, extending from binary to arbitrary trees its upper bound on $( \rank \mat{\tensorend }{ \nu} )_{\nu \in \htmodetree}$.

Towards deriving a matricized form of \eq~\eqref{eq:ht_end_tensor}, we define the notion of \emph{index set reduction}.
The reduction of $\nu \in \htmodetree$ onto $\nu' \in \htmodetree$, whose elements are denoted by $i_1 < \cdots < i_{ \abs{\nu'} }$, is defined by:
\[
\nu |_{\nu'} := \left \{ n \in [ \abs{ \nu' } ] : i_n \in \nu \cap \nu' \right \}
\text{\,.}
\]

Now, fix $\nu \in \interior (\htmodetree)$ and $\nu_c \in \children (\nu)$.
By \lem~\ref{lem:tenp_eq_kronp} and the linearity of the matricization operator, we may write the computation of $\mat{ \tensorend }{ \nu_c}$ based on \eq~\eqref{eq:ht_end_tensor} as follows:
\[
\begin{split}
	&\text{For $\bar{\nu} \in \left \{ \{1\}, \ldots, \{ N\} \right \}$ (traverses leaves of $\htmodetree$):} \\
	&\quad \tensorpart{\bar{\nu}}{r} := \weightmat{\bar{\nu}}_{:, r} \quad , r \in [ R_{ \parent ( \bar{\nu} ) } ] \text{\,,} \\[2.5mm]
	&\text{for $\bar{\nu} \in \interior (\htmodetree) \setminus \left \{ [N] \right \}$ (traverses interior nodes of $\htmodetree$ from leaves to root, non-inclusive):} \\
	&\quad \mat{ \tensorpart{\bar{\nu}}{r} }{ \nu_c |_{ \bar{\nu} }} :=  \Qbf^{(\bar{\nu})} \left ( \sum\nolimits_{r' = 1}^{R_{\bar{\nu}}} \weightmat{\bar{\nu}}_{r', r} \left [ \kronp_{\nu' \in \children ( \bar{\nu} )} \mat{ \tensorpart{ \nu' }{ r' } }{ \nu_c |_{\nu'} } \right ] \right ) \widebar{\Qbf}^{(\bar{\nu})} \quad ,r \in [ R_{\parent (\bar{\nu} )} ] \text{\,,} \\[2.5mm]
	& \mat{ \tensorend }{ \nu_c} = \Qbf^{([N])} \left ( \sum\nolimits_{r' = 1}^{R_{ [N] } } \weightmat{ [N] }_{r', 1} \left [ \kronp_{ \nu' \in \children ( [N] )} \mat{ \tensorpart{\nu'}{r'} }{ \nu_c |_{\nu'}  } \right ] \right ) \widebar{\Qbf}^{([N])}
	\text{\,,}
\end{split}
\]
where $\Qbf^{(\bar{\nu})}$ and $\widebar{\Qbf}^{ (\bar{\nu})}$, for $\bar{\nu} \in \interior (\htmodetree)$, are permutation matrices rearranging the rows and columns, respectively, to accord with an ascending order of $\bar{\nu}$, \ie~they fulfill the role of $\pi_{\bar{\nu}}$ in \eq~\eqref{eq:ht_end_tensor}.
For $r \in [R_{ \parent (\nu) }]$, let us focus on $\matbig{\tensorpart{\nu}{r} }{ \nu_c |_\nu}$.
Since $\nu_c |_{\nu_c} = [\abs{ \nu_c } ]$ and $\nu_c |_{\nu'} = \emptyset$ for all $\nu' \in \children (\nu) \setminus \{ \nu_c \}$, we have that:
\[
\begin{split}
& \mat{\tensorpart{\nu}{r} }{ \nu_c |_\nu} \\
& \hspace{4.5mm} = \Qbf^{(\nu)} \left ( \sum_{r' = 1}^{R_{\nu}} \weightmat{\nu}_{r', r} \left [ \left ( \kronp_{\nu' \in \lsib ( \nu_c )} \mat{ \tensorpart{ \nu' }{ r' } }{ \emptyset} \right ) \kronp \mat{ \tensorpart{\nu_c}{r'} }{ [\abs{ \nu_c } ]} \kronp \left ( \kronp_{\nu' \in \rsib ( \nu_c )} \mat{ \tensorpart{ \nu' }{ r' } }{ \emptyset} \right ) \right ] \right ) \widebar{\Qbf}^{ (\nu)} 
\text{\,.}
\end{split}
\]
Notice that $\matbig{ \tensorpart{ \nu' }{ r' } }{ \emptyset}$ is a row vector, whereas $\matbig{ \tensorpart{\nu_c}{r'} }{ [\abs{ \nu_c } ]}$ is a column vector, for each $\nu' \in \children (\nu) \setminus \{ \nu_c\}$ and $r' \in [R_\nu]$.
Commutativity of the Kronecker product between a row and column vectors therefore leads to:
\[
\begin{split}
\mat{\tensorpart{\nu}{r} }{ \nu_c |_\nu} & = \Qbf^{(\nu)} \left ( \sum\nolimits_{r' = 1}^{R_{\nu}} \weightmat{\nu}_{r', r} \left [ \mat{ \tensorpart{\nu_c}{r'} }{ [\abs{ \nu_c } ]} \kronp \left ( \kronp_{\nu' \in \children ( \nu ) \setminus \{ \nu_c \}} \mat{ \tensorpart{ \nu' }{ r' } }{ \emptyset} \right ) \right ] \right ) \widebar{\Qbf}^{ (\nu)} \\
& = \Qbf^{(\nu)} \left ( \sum\nolimits_{r' = 1}^{R_{\nu}} \weightmat{\nu}_{r', r} \left [ \mat{ \tensorpart{\nu_c}{r'} }{ [\abs{ \nu_c } ]} \left ( \kronp_{\nu' \in \children ( \nu ) \setminus \{ \nu_c \}} \mat{ \tensorpart{ \nu' }{ r' } }{ \emptyset} \right ) \right ] \right ) \widebar{\Qbf}^{ (\nu)}
\text{\,,}
\end{split}
\]
where the second equality is by the fact that for any column vector $\ubf$ and row vector $\vbf$ it holds that $\ubf \kronp \vbf = \ubf \vbf$.
Defining $\Bbf^{(\nu_c)}$ to be the matrix whose column vectors are $\matbig{ \tensorpart{\nu_c}{1} }{ [\abs{ \nu_c } ]}, \ldots, \matbig{ \tensorpart{\nu_c}{ R_\nu } }{ [\abs{ \nu_c } ]}$, we can express the term between $\Qbf^{(\nu)}$ and $\widebar{\Qbf}^{ (\nu)}$ in the equation above as a $\Bbf^{(\nu_c)} \Abf^{(\nu, r)}$, where $\Abf^{(\nu, r)}$ is defined to be the matrix whose rows are $\weightmat{\nu}_{1, r} \big ( \kronp_{\nu' \in \children ( \nu ) \setminus \{ \nu_c \}} \matbig{ \tensorpart{ \nu' }{ 1 } }{ \emptyset} \big ), \ldots, \weightmat{\nu}_{R_\nu, r} \big ( \kronp_{\nu' \in \children ( \nu ) \setminus \{ \nu_c \}} \matbig{ \tensorpart{ \nu' }{ R_\nu } }{ \emptyset} \big )$.
That is:
\be
\mat{\tensorpart{\nu}{r}}{\nu_c |_\nu} = \Qbf^{(\nu)} \Bbf^{(\nu_c)} \Abf^{(\nu, r)} \widebar{\Qbf}^{ (\nu)}
\text{\,.}
\label{eq:tensorpart_nu_mat_B}
\ee
The proof proceeds by propagating  $\Bbf^{(\nu_c)}$ and the left permutation matrices up the tree, until reaching a representation of $\mat{\tensorend }{ \nu_c}$ as a product of matrices that includes $\Bbf^{(\nu_c)}$.
Since $\Bbf^{(\nu_c)}$ has $R_\nu$ columns, this will imply that the rank of $\mat{\tensorend}{\nu_c}$ is at most $R_\nu$, as required.

We begin with the propagation step from $\nu$ to $\parent (\nu)$.
For $r \in [ R_{ \parent ( \parent ( \nu ) )} ]$, we examine:
\[
\begin{split}
& \big ( \Qbf^{ (\parent (\nu))} \big )^{-1} \matbig{ \tensorpart{ \parent (\nu) }{r} }{\nu_c |_{ \parent (\nu) } } \big ( \widebar{\Qbf}^{  ( \parent (\nu) ) } \big )^{-1} \\
& \hspace{1mm} = \sum\nolimits_{r' = 1}^{ R_{ \parent (\nu) } }\weightmat{\parent (\nu)}_{r', r} \left [ \left ( \kronp_{\nu' \in \lsib ( \nu )} \mat{ \tensorpart{ \nu' }{ r' } }{\nu_c |_{\nu'}} \right ) \kronp \mat{ \tensorpart{\nu}{r'} }{\nu_c |_{\nu}} \kronp \left ( \kronp_{\nu' \in \rsib ( \nu )} \mat{ \tensorpart{ \nu' }{ r' } }{\nu_c |_{\nu'}} \right )\right ]
\text{\,.}
\end{split}
\]
Plugging in \eq~\eqref{eq:tensorpart_nu_mat_B} while noticing that $\nu_c |_{\nu'} = \emptyset$ for any $\nu'$ which is not an ancestor of $\nu_c$, we arrive~at:
\be
\sum\nolimits_{r' = 1}^{ R_{ \parent (\nu) } }\weightmat{\parent (\nu)}_{r', r} \left [ \left ( \kronp_{\nu' \in \lsib ( \nu )} \mat{ \tensorpart{ \nu' }{ r' } }{\emptyset} \right ) \kronp \left ( \Qbf^{(\nu)} \Bbf^{(\nu_c)} \Abf^{(\nu, r')} \widebar{\Qbf}^{ (\nu)} \right ) \kronp \left ( \kronp_{\nu' \in \rsib ( \nu )} \mat{ \tensorpart{ \nu' }{ r' } }{\emptyset} \right ) \right ]
\text{\,.}
\label{eq:tensorpart_pa_nu_mat_no_Q}
\ee
Let $r' \in [ R_{ \parent (\nu) } ]$.
Since $\matbig{ \tensorpart{ \nu' }{ r' } }{\emptyset}$ is a row vector for any $\nu' \in \children ( \parent (\nu) ) \setminus \{ \nu \}$, so are $\kronp_{\nu' \in \lsib ( \nu )} \matbig{ \tensorpart{ \nu' }{ r' } }{\emptyset}$ and $\kronp_{\nu' \in \rsib ( \nu )} \matbig{ \tensorpart{ \nu' }{ r' } }{\emptyset}$.
Applying \lem~\ref{lem:kronp_mixed_prod_row_vec} twice we therefore have that:
\[
\begin{split}
& \left ( \kronp_{\nu' \in \lsib ( \nu )} \mat{ \tensorpart{ \nu' }{ r' } }{\emptyset} \right ) \kronp \left ( \Qbf^{(\nu)} \Bbf^{(\nu_c)} \Abf^{(\nu, r')} \widebar{\Qbf}^{ (\nu)} \right ) \kronp \left ( \kronp_{\nu' \in \rsib ( \nu )} \mat{ \tensorpart{ \nu' }{ r' } }{\emptyset} \right ) \\[1mm]
& \quad\quad = \left ( \Qbf^{(\nu)} \Bbf^{(\nu_c)} \left [ \left ( \kronp_{\nu' \in \lsib ( \nu )} \mat{ \tensorpart{ \nu' }{ r' } }{\emptyset} \right ) \kronp \left ( \Abf^{(\nu, r')} \widebar{\Qbf}^{ (\nu)} \right ) \right ] \right ) \kronp \left ( \kronp_{\nu' \in \rsib ( \nu )} \mat{ \tensorpart{ \nu' }{ r' } }{\emptyset} \right ) \\[1mm]
& \quad\quad = \Qbf^{(\nu)} \Bbf^{(\nu_c)} \left [ \left ( \kronp_{\nu' \in \lsib ( \nu )} \mat{ \tensorpart{ \nu' }{ r' } }{\emptyset} \right ) \kronp \left ( \Abf^{(\nu, r')} \widebar{\Qbf}^{ (\nu)} \right )  \kronp \left ( \kronp_{\nu' \in \rsib ( \nu )} \mat{ \tensorpart{ \nu' }{ r' } }{\emptyset} \right ) \right ]
\text{\,.}
\end{split}
\]
Going back to \eq~\eqref{eq:tensorpart_pa_nu_mat_no_Q}, we obtain:
\be
\begin{split}
& \sum\nolimits_{r' = 1}^{ R_{ \parent (\nu) } }\weightmat{\parent (\nu)}_{r', r} \left [ \left (  \kronp_{\nu' \in \lsib ( \nu )} \mat{ \tensorpart{ \nu' }{ r' } }{\emptyset} \right ) \kronp \left ( \Qbf^{(\nu)} \Bbf^{(\nu_c)} \Abf^{(\nu, r')} \widebar{\Qbf}^{ (\nu)} \right ) \kronp \left ( \kronp_{\nu' \in \rsib ( \nu )} \mat{ \tensorpart{ \nu' }{ r' } }{\emptyset} \right ) \right ] \\[1mm]
& \hspace{0.5mm} =  \sum\nolimits_{r' = 1}^{ R_{ \parent (\nu) } }\weightmat{\parent (\nu)}_{r', r}  \left ( \Qbf^{(\nu)} \Bbf^{(\nu_c)} \left [ \left ( \kronp_{\nu' \in \lsib ( \nu )} \mat{ \tensorpart{ \nu' }{ r' } }{\emptyset} \right ) \kronp \left ( \Abf^{(\nu, r')} \widebar{\Qbf}^{ (\nu)} \right )  \kronp \left ( \kronp_{\nu' \in \rsib ( \nu )} \mat{ \tensorpart{ \nu' }{ r' } }{\emptyset} \right ) \right ] \right ) \\[1mm]
& \hspace{0.5mm} =  \Qbf^{(\nu)} \Bbf^{(\nu_c)} \left ( \sum\nolimits_{r' = 1}^{ R_{ \parent (\nu) } }\weightmat{\parent (\nu)}_{r', r} \left [ \left ( \kronp_{\nu' \in \lsib ( \nu )} \mat{ \tensorpart{ \nu' }{ r' } }{\emptyset} \right ) \kronp \left ( \Abf^{(\nu, r')} \widebar{\Qbf}^{ (\nu)} \right )\kronp \left ( \kronp_{\nu' \in \rsib ( \nu )} \mat{ \tensorpart{ \nu' }{ r' } }{\emptyset} \right ) \right ]  \right )
\text{.}
\end{split}
\label{eq:tensorpart_pa_nu_mat_no_Q_B_out}
\ee
For brevity, we denote the matrix multiplying $\Qbf^{(\nu)} \Bbf^{(\nu_c)}$ from the right in the equation above by $\Abf^{(\parent (\nu), r)}$, \ie:
\[
\Abf^{(\parent (\nu), r)} := \sum\nolimits_{r' = 1}^{ R_{ \parent (\nu) } }\weightmat{\parent (\nu)}_{r', r} \left [ \left ( \kronp_{\nu' \in \lsib ( \nu )} \mat{ \tensorpart{ \nu' }{ r' } }{\emptyset} \right ) \kronp \left ( \Abf^{(\nu, r')} \widebar{\Qbf}^{ (\nu)} \right )\kronp \left ( \kronp_{\nu' \in \rsib ( \nu )} \mat{ \tensorpart{ \nu' }{ r' } }{\emptyset} \right ) \right ]
\text{\,.}
\]
Recalling that the expression in \eq~\eqref{eq:tensorpart_pa_nu_mat_no_Q_B_out} is of $\big ( \Qbf^{ (\parent (\nu))} \big )^{-1} \matbig{ \tensorpart{ \parent (\nu) }{r} }{\nu_c |_{ \parent (\nu) } } \big ( \widebar{\Qbf}^{  ( \parent (\nu) ) } \big )^{-1}$ completes the propagation step:
\[
\mat{ \tensorpart{ \parent (\nu) }{r} }{\nu_c |_{ \parent (\nu) } } = \Qbf^{( \parent (\nu) ) }  \Qbf^{(\nu)} \Bbf^{(\nu_c)} \Abf^{ ( \parent (\nu), r)} \widebar{\Qbf}^{ ( \parent (\nu) )}
\text{\,.}
\]

Continuing this process, we propagate $\Bbf^{(\nu_c)}$, along with the left permutation matrices, upwards in the tree until reaching the root.
This brings forth the following representation of $\mat{ \tensorend }{\nu_c}$:
\[
\mat{ \tensorend }{\nu_c} =  \Qbf^{([N])} \Qbf \Bbf^{ (\nu_c )} \Abf^{ ([N])} \widebar{\Qbf}^{ ( [N] )}
\text{\,,}
\]
for appropriate $\Qbf$ and $\Abf^{([N])}$ encompassing the propagated permutation matrices and the “remainder'' of the decomposition, respectively.
Since $\Bbf^{(\nu_c)}$ has $R_\nu$ columns, we may conclude:
 \[
\rank \mat{\tensorend}{\nu_c} \leq \rank \, \Bbf^{(\nu_c)} \leq R_\nu
\text{\,.}
\]
\qed

\subsection{Proof of Lemma~\ref{lem:loc_comp_sq_norm_diff_invariant}}
\label{app:proofs:loc_comp_sq_norm_diff_invariant}
For any $\nu \in \interior (\htmodetree)$, $r \in [R_\nu]$, and $\wbf, \wbf' \in \localcomp (\nu, r)$, \lem~\ref{lem:weightvec_sq_norm_time_deriv} implies that:
\[
\frac{d}{dt} \normbig{ \wbf (t) }^2 = 2 \htfcompnorm{\nu}{r} (t) \inprod{ - \nabla \htfendloss ( \tensorend (t)) }{ \htfcomp{\nu}{r} (t) } = \frac{d}{dt} \normbig{ \wbf' (t) }^2 
\text{\,,}
\]
where $\htfcomp{\nu}{r} (t)$ is as defined in \thm~\ref{thm:loc_comp_norm_bal_dyn}.
For $t \geq 0$, integrating both sides with respect to time leads to:
\[
\normbig{ \wbf (t) }^2 - \normbig{\wbf (0) }^2 = \normbig{ \wbf' (t) }^2 - \normbig{ \wbf' (0) }^2
\text{\,.}
\]
Rearranging the equality above yields the desired result.
\qed

\subsection{Proof of Theorem~\ref{thm:loc_comp_norm_bal_dyn}}
\label{app:proofs:loc_comp_norm_bal_dyn}
Let $t \geq 0$.

First, suppose that $\htfcompnorm{\nu}{r} (t) = 0$.
Since the unbalancedness magnitude at initialization is zero, from \lem~\ref{lem:loc_comp_sq_norm_diff_invariant} we know that $\norm{ \wbf (t) } = \norm{ \wbf' (t) }$ for any $\wbf, \wbf' \in \localcomp (\nu, r)$.
Hence, the fact that $\htfcompnorm{\nu}{r} (t) = 0$ implies that $\norm{ \wbf (t) } = 0$ for all $\wbf \in \localcomp (\nu, r)$.
\lem~\ref{lem:bal_zero_stays_zero} then establishes that $\htfcompnorm{\nu}{r} (t')$ is identically zero through time, in which case both sides of \eq~\eqref{eq:loc_comp_norm_bal_dyn} are equal to zero.

We now move to the case where $\htfcompnorm{\nu}{r} (t) > 0$.
Since $\htfcompnorm{\nu}{r} (t) = \normnoflex{\tenp_{\wbf \in \localcomp (\nu, r)} \wbf (t) } = \prod\nolimits_{ \wbf \in \localcomp (\nu, r) } \norm{ \wbf (t) }$ (the norm of a tensor product is equal to the product of the norms), by the product rule of differentiation we have that:
\[
\frac{d}{dt} \htfcompnorm{\nu}{r} (t)^2 = \sum\nolimits_{ \wbf \in \localcomp (\nu, r) } \frac{d}{dt} \norm{ \wbf (t) }^2 \cdot \prod\nolimits_{ \wbf' \in \localcomp (\nu, r) \setminus \{ \wbf \} } \norm{ \wbf' (t) }^2
\text{\,.}
\]
Applying \lem~\ref{lem:weightvec_sq_norm_time_deriv} then leads to:
\[
\begin{split}
	\frac{d}{dt} \htfcompnorm{\nu}{r} (t)^2 & = \sum\nolimits_{ \wbf \in \localcomp (\nu, r) } 2 \htfcompnorm{\nu}{r} (t) \inprod{ - \nabla \htfendloss ( \tensorend (t)) }{ \htfcomp{\nu}{r} (t) } \cdot \prod\nolimits_{ \wbf' \in \localcomp (\nu, r) \setminus \{ \wbf \} } \norm{ \wbf' (t) }^2 \\
	& = 2 \htfcompnorm{\nu}{r} (t) \inprod{ - \nabla \htfendloss ( \tensorend (t)) }{ \htfcomp{\nu}{r} (t) } \sum\nolimits_{ \wbf \in \localcomp (\nu, r) } \prod\nolimits_{ \wbf' \in \localcomp (\nu, r) \setminus \{ \wbf \} } \norm{ \wbf' (t) }^2
	\text{\,.}
\end{split}
\]
From the chain rule we know that $\frac{d}{dt} \htfcompnorm{\nu}{r} (t)^2 = 2 \htfcompnorm{\nu}{r} (t) \cdot \frac{d}{dt} \htfcompnorm{\nu}{r} (t)$ (note that $\frac{d}{dt} \htfcompnorm{\nu}{r} (t)$ surely exists because $\htfcompnorm{\nu}{r} (t) > 0$).
Thus:
\be
\begin{split}
	\frac{d}{dt} \htfcompnorm{\nu}{r} (t) & = \frac{1}{2} \htfcompnorm{\nu}{r} (t)^{-1} \frac{d}{dt} \htfcompnorm{\nu}{r} (t)^2 \\
	& = \inprod{ - \nabla \htfendloss ( \tensorend (t)) }{ \htfcomp{\nu}{r} (t) } \sum\nolimits_{ \wbf \in \localcomp (\nu, r) } \prod\nolimits_{ \wbf' \in \localcomp (\nu, r) \setminus \{ \wbf \} } \norm{ \wbf' (t) }^2
	\text{\,.}
\end{split}
\label{eq:localcompnorm_time_deriv_interm_bal}
\ee
According to \lem~\ref{lem:loc_comp_sq_norm_diff_invariant}, the unbalancedness magnitude remains zero through time, and so $\norm{ \wbf (t) } = \norm{ \wbf' (t) }$ for any $\wbf, \wbf' \in \localcomp (\nu, r)$.
Recalling that $L_\nu := \children (\nu) + 1$ is the number of weight vectors in a local component at $\nu$, this implies that for each $\wbf \in \localcomp (\nu, r)$:
\be
\norm{ \wbf (t) }^2 =  \norm{ \wbf (t) }^{L_\nu \cdot \frac{2}{L_\nu} } = \left ( \prod\nolimits_{\wbf' \in \localcomp (\nu, r)} \norm{ \wbf' (t) } \right )^{\frac{2}{L_\nu}} =  \htfcompnorm{\nu}{r} (t)^{ \frac{2}{ L_\nu } }
\text{\,.}
\label{eq:sq_weightvec_norm_eq_local_comp_norm}
\ee
Plugging \eq~\eqref{eq:sq_weightvec_norm_eq_local_comp_norm} into \eq~\eqref{eq:localcompnorm_time_deriv_interm_bal} completes the proof.
\qed

\subsection{Proof of Proposition~\ref{prop:low_rank_dist_bound}}
\label{app:proofs:low_rank_dist_bound}
We begin by establishing the following key lemma, which upper bounds the distance between the end tensor $\tensorend$ and the one obtained after setting a local component to zero.

\begin{lemma}
\label{lem:dist_from_pruned_local_comp_end_tensor}
Let $\nu \in \interior ( \htmodetree )$ and $r \in [R_\nu]$.
Denote by $\widebar{\W}_{HT}^{(\nu, r)}$ the end tensor obtained by pruning the $(\nu, r)$'th local component, \ie~by setting the $r$'th row of $\weightmat{\nu}$ and the $r$'th columns of $\big ( \weightmat{\nu_c} \big )_{\nu_c \in \children (\nu)}$ to zero.
Then:
\[
\norm{ \tensorend - \widebar{\W}_{HT}^{(\nu, r)} } \leq \htfcompnorm{\nu}{r} \cdot \prod\nolimits_{\nu' \in \htmodetree \setminus ( \{ \nu \} \cup \children (\nu) )} \norm{ \weightmat{\nu'} }
\text{\,.}
\]
\end{lemma}
\begin{proof}
Let $\big ( \widebar{ \Wbf }^{(\nu')} \big )_{\nu' \in \htmodetree}$ be the weight matrices corresponding to $\widebar{\W}_{HT}^{(\nu, r)}$, \ie~$\widebar{\Wbf}^{(\nu)}$ is the weight matrix obtained by setting the $r$'th row of $\weightmat{\nu}$ to zero, $\big ( \widebar{\Wbf}^{(\nu_c)} \big )_{\nu_c \in \children (\nu)}$ are the weight matrices obtained by setting the $r$'th columns of $\big ( \weightmat{\nu_c} \big )_{\nu_c \in \children (\nu) }$ to zero, and $\widebar{ \Wbf }^{(\nu')} = \weightmat{\nu'}$ for all $\nu' \in \htmodetree \setminus ( \{ \nu \} \cup \children (\nu))$.
Accordingly, we denote by $\big ( \widebar{ \W }^{(\nu', r')} \big )_{\nu' \in \htmodetree, r' \in [R_{ \parent (\nu') }]}$ the intermediate tensors produced when computing $\widebar{\W}_{HT}^{(\nu, r)}$ according to \eq~\eqref{eq:ht_end_tensor} (there denoted $\big ( \tensorpart{\nu'}{r'} \big )_{\nu', r' }$).

By definition, $\tensorendmap \big ( \big ( \weightmat{ \nu' } \big )_{ \nu' \in \htmodetree \setminus \subtree (\nu)}, \tensorpart{\nu}{:} \big ) = \tensorend$ and $\tensorendmap \big ( \big ( \widebar{\Wbf}^{ (\nu') } \big )_{ \nu' \in \htmodetree \setminus \subtree (\nu)}, \widebar{\W}^{(\nu, :)} \big ) = \widebar{\W}_{HT}^{(\nu, r)}$.
Since $\tensorendmap$ is multilinear (\lem~\ref{lem:ht_multilinear}) and $\widebar{ \Wbf }^{(\nu')} = \weightmat{\nu'}$ for all $\nu' \in \htmodetree \setminus \subtree (\nu)$, we have that:
\[
\begin{split}
\norm{ \tensorend - \widebar{\W}_{HT}^{(\nu, r)} } & = \norm{ \tensorendmap \big ( ( \weightmat{ \nu' } )_{ \nu' \in \htmodetree \setminus \subtree (\nu)}, \tensorpart{\nu}{:} \big ) - \tensorendmap \big ( ( \widebar{\Wbf}^{ (\nu') } )_{ \nu' \in \htmodetree \setminus \subtree (\nu)}, \widebar{\W}^{(\nu, :)} \big ) } \\
& = \norm{ \tensorendmap \big ( ( \weightmat{ \nu' } )_{ \nu' \in \htmodetree \setminus \subtree (\nu)}, \tensorpart{\nu}{:} -  \widebar{\W}^{(\nu, :)} \big ) }
\text{\,.}
\end{split}
\]
Heading from the root downwards, subsequent applications of \lem~\ref{lem:stacked_tensorpart_norm_bound} over all nodes in the mode tree, except those belonging to the sub-tree whose root is $\nu$, then yield:
\be
\norm{ \tensorend - \widebar{\W}_{HT}^{(\nu, r)} } \leq \norm{  \tensorpart{\nu}{:} -  \widebar{\W}^{(\nu, :)} } \cdot \prod\nolimits_{\nu' \in \htmodetree \setminus \subtree (\nu)} \norm{ \weightmat{\nu'} }
\text{\,.}
\label{eq:tensorend_trimmed_comp_upper_bound_intermid}
\ee
Notice that for any $r' \in [R_{ \parent (\nu) } ]$:
\[
\begin{split}
\left ( \tensorpart{\nu}{:} -  \widebar{\W}^{(\nu, :)} \right )_{:, \ldots, :, r'} & = \sum\nolimits_{\bar{r} \in [R_\nu]} \weightmat{\nu}_{\bar{r}, r'} \tenp_{ \nu_c \in \children (\nu) } \tensorpart{\nu_c}{\bar{r}} - \sum\nolimits_{\bar{r} \in [R_\nu] \setminus \{ r \}} \weightmat{\nu}_{\bar{r}, r'} \tenp_{ \nu_c \in \children (\nu) } \tensorpart{\nu_c}{\bar{r}} \\
& = \weightmat{\nu}_{r, r'} \tenp_{ \nu_c \in \children (\nu) } \tensorpart{\nu_c}{r}
\text{\,.}
\end{split}
\]
Thus, a straightforward computation shows:
\[
\begin{split}
\norm{  \tensorpart{\nu}{:} -  \widebar{\W}^{(\nu, :)} }^2 & = \sum\nolimits_{r' = 1}^{ R_{ \parent (\nu') } } \norm{  \weightmat{\nu}_{r, r'} \tenp_{ \nu_c \in \children (\nu) } \tensorpart{\nu_c}{r} }^2 \\
& = \sum\nolimits_{r' = 1}^{ R_{ \parent (\nu') } } \left ( \weightmat{\nu}_{r, r'} \right )^2 \cdot \prod\nolimits_{ \nu_c \in \children (\nu) } \norm{ \tensorpart{\nu_c}{r} }^2 \\
& = \norm{ \weightmat{\nu}_{r, :} }^2 \cdot \prod\nolimits_{ \nu_c \in \children (\nu) } \norm{ \tensorpart{\nu_c}{r} }^2
\text{\,,}
\end{split}
\]
where the second equality is by the fact that the norm of a tensor product is equal to the product of the norms.
From \lem~\ref{lem:tensorpart_norm_bound} we get that $\normbig{ \tensorpart{\nu_c}{r} } \leq \normbig{ \weightmat{ \nu_c }_{:, r} } \cdot \prod\nolimits_{ \nu' \in \children (\nu_c) } \normbig{ \tensorpart{\nu'}{:} }$ for all $\nu_c \in \children (\nu)$, which leads to:
\[
\begin{split}
\norm{  \tensorpart{\nu}{:} -  \widebar{\W}^{(\nu, :)} }^2 & \leq \norm{ \weightmat{\nu}_{r, :} }^2 \cdot \prod\nolimits_{ \nu_c \in \children (\nu) } \left ( \norm{ \weightmat{ \nu_c }_{:, r} }^2 \cdot \prod\nolimits_{ \nu' \in \children (\nu_c) } \norm{ \tensorpart{\nu'}{:} }^2 \right ) \\
& = \left ( \htfcompnorm{\nu}{r} \right )^2 \cdot \prod\nolimits_{ \nu_c \in \children (\nu), \nu' \in \children (\nu_c) } \norm{ \tensorpart{\nu'}{:} }^2
\text{\,.}
\end{split}
\]
Taking the square root of both sides and plugging the inequality above into \eq~\eqref{eq:tensorend_trimmed_comp_upper_bound_intermid}, we arrive at:
\[
\norm{ \tensorend - \widebar{\W}_{HT}^{(\nu, r)} } \leq \htfcompnorm{\nu}{r} \cdot \prod\nolimits_{\nu_c \in \children (\nu), \nu' \in \children (\nu_c) } \norm{ \tensorpart{\nu'}{:} } \cdot \prod\nolimits_{\nu' \in \htmodetree \setminus \subtree (\nu)} \norm{ \weightmat{\nu'} }
\text{\,.}
\]
Applying \lem~\ref{lem:stacked_tensorpart_norm_bound} iteratively over the sub-trees whose roots are $\children (\nu_c)$ gives:
\[
\prod\nolimits_{\nu_c \in \children (\nu), \nu' \in \children (\nu_c) } \norm{ \tensorpart{\nu'}{:} } \leq \prod\nolimits_{\nu' \in \subtree (\nu) \setminus ( \{ \nu \} \cup \children (\nu) )} \norm{ \weightmat{\nu' } }
\text{\,,}
\]
concluding the proof.
\end{proof}

\medskip

With \lem~\ref{lem:dist_from_pruned_local_comp_end_tensor} in hand, we are now in a position to prove \prop~\ref{prop:low_rank_dist_bound}.
Let
\[
\S := \left \{ (\nu, r) : \nu \in \interior (\htmodetree) , r \in \{ R'_{\nu} + 1, \ldots, R_\nu \right \}
\text{\,,}
\]
and denote by $\widebar{\W}_{HT}^{\S}$ the end tensor obtained by pruning all local components in $\S$, \ie~by setting to zero the $r$'th row of $\weightmat{\nu}$ and the $r$'th column of $\weightmat{\nu_c}$ for all $(\nu, r) \in \S$ and $\nu_c \in \children (\nu)$.
As can be seen from \eq~\eqref{eq:ht_end_tensor}, we may equivalently discard these weight vectors instead of setting them to zero.
Doing so, we arrive at a representation of $\widebar{\W}_{HT}^{\S}$ as the end tensor of $\big ( \widebar{ \Wbf }^{(\nu)} \in \R^{R'_\nu, R'_{ \parent (\nu) } } \big )_{\nu \in \htmodetree}$, where $R'_{ \parent ([N])} = 1$, $R'_{ \{n\}} = D_n$ for $n \in [N]$, and $\widebar{ \Wbf }^{(\nu)} = \weightmat{\nu}_{:R'_\nu, :R'_{\parent (\nu)}}$ for all $\nu \in \htmodetree$.
Hence, \lem~\ref{lem:loc_comp_ht_rank_bound} implies that for any $\nu \in \htmodetree$ the rank of $\mat{ \widebar{\W}_{HT}^{\S} }{\nu}$ is at most $R'_{ \parent (\nu) }$.
This means that it suffices to show that:
\be
\norm{ \tensorend - \widebar{\W}_{HT}^{\S} } \leq \epsilon
\text{\,.}
\label{eq:tensorend_prunedtensorend_dist}
\ee
For $i \in [ \abs{\S} ]$, let $\S_i \subset \S$ be the set comprising the first $i$ local components in $\S$ according to an arbitrary order.
Adding and subtracting $\widebar{\W}_{HT}^{\S_i}$ for all $i \in [ \abs{\S} - 1]$, and applying the triangle inequality, we have:
\[
\norm{ \tensorend - \widebar{\W}_{HT}^{\S} } \leq \sum\nolimits_{i = 0}^{\abs{\S} - 1} \norm{ \widebar{\W}_{HT}^{\S_i} - \widebar{\W}_{HT}^{\S_{i + 1}} }
\text{\,,}
\]
where $\widebar{\W}_{HT}^{\S_0} := \tensorend$.
Upper bounding each term in the sum according to \lem~\ref{lem:dist_from_pruned_local_comp_end_tensor}, while noticing that pruning a local component can only decrease the norms of weight matrices and other local components in the factorization, we obtain:
\[
\begin{split}
\norm{ \tensorend - \widebar{\W}_{HT}^{\S} } & \leq \sum\nolimits_{\nu \in \interior (\htmodetree) } \sum\nolimits_{r = R'_\nu + 1}^{R_\nu} \htfcompnorm{\nu}{r} \cdot \prod\nolimits_{\nu' \in \htmodetree \setminus \left ( \{ \nu \} \cup \children (\nu) \right )}  \normbig{ \weightmat{\nu'} } \\
& \leq \sum\nolimits_{\nu \in \interior (\htmodetree) } B^{\abs{\htmodetree} - 1 - \abs{\children (\nu)}} \cdot \sum\nolimits_{r = R'_\nu + 1}^{R_\nu} \htfcompnorm{\nu}{r}
\text{\,,}
\end{split}
\]
where the latter inequality is by recalling that $B = \max_{\nu \in \htmodetree} \normnoflex{ \weightmat{\nu}}$.
Since for all $\nu \in \interior (\htmodetree)$ we have that $\sum\nolimits_{r = R'_\nu + 1}^{R_\nu} \htfcompnorm{\nu}{r} \leq \epsilon \cdot (\abs{\htmodetree} - N)^{-1} B^{\abs{\children (\nu)} + 1 - \abs{\htmodetree} }$, \eq~\eqref{eq:tensorend_prunedtensorend_dist} readily follows.
\qed

\subsection{Proof of Proposition~\ref{prop:htr_multiple_minima}}
\label{app:proofs:htr_multiple_minima}
Consider the tensor completion problem defined by the set of observed entries
\[
\Omega = \left \{ (1, \ldots, 1, 1, 1, 1), (1, \ldots, 1, 1, 1, 2), (1, \ldots, 1, 2, 2, 1), (1, \ldots, 1, 2, 2, 2) \right \}
\]
and ground truth $\W^* \in \R^{D_1, \ldots, D_N}$, whose values at those locations are:
\be
\W^*_{1, \ldots, 1, 1, :2, :2} = \begin{bmatrix} 1 & 0 \\ ? & ? \end{bmatrix} ~~ , ~~ \W^*_{1, \ldots, 1, 2, :2, :2} = \begin{bmatrix} ? & ? \\ 0 & 1 \end{bmatrix} 
\text{\,,}
\label{eq:ht_multiple_minima_observations}
\ee
where $?$ stands for an unobserved entry.
We define two solutions for the tensor completion problem, $\W$ and $\W'$ in $\R^{D_1, \ldots, D_N}$, as follows:
\[
\W_{1, \ldots, 1, 1, :2, :2} := \begin{bmatrix} 1 & 0 \\ 1 & 0 \end{bmatrix} ~~ , ~~ \W_{1, \ldots, 1, 2, :2, :2} := \begin{bmatrix} 0 & 1 \\ 0 & 1 \end{bmatrix}  ~ \quad , \quad ~ \W'_{1, \ldots, 1, 1, :2, :2} := \begin{bmatrix} 1 & 0 \\ 0 & 1 \end{bmatrix} ~~ , ~~ \W'_{1, \ldots, 1, 2, :2, :2} := \begin{bmatrix} 1 & 0 \\ 0 & 1 \end{bmatrix} 
\text{\,,}
\text{\,}
\]
and the remaining entries of $\W$ and $\W'$ hold zero.
Clearly, $\L (\W) = \L (\W') = 0$.

Fix a mode tree $\htmodetree$ over $[N]$.
Since $\W$ and $\W'$ fit the observed entries their hierarchical tensor ranks with respect to $\htmodetree$, $ ( \rank \mat{ \W }{\nu} )_{\nu \in \htmodetree \setminus \{ [N] \}}$ and $ ( \rank \mat{ \W' }{\nu} )_{\nu \in \htmodetree \setminus \{ [N] \}}$, are in $\RR_\htmodetree$.
We prove that neither $ ( \rank \mat{ \W }{\nu} )_{\nu \in \htmodetree \setminus \{ [N] \}} \leq  ( \rank \mat{\W'}{\nu} )_{\nu \in \htmodetree \setminus \{ [N] \}}$ nor $( \rank \mat{ \W' }{\nu} )_{\nu \in \htmodetree \setminus \{ [N] \}} \leq  ( \rank \mat{\W}{\nu} )_{\nu \in \htmodetree \setminus \{ [N] \}}$ (with respect to the standard product partial order), by examining the matrix ranks of the matricizations of $\W$ and $\W'$ according to $\{ N - 2 \} \in \htmodetree$ and $\{ N - 1\} \in \htmodetree$ (recall that any mode tree has leaves $\{1\}, \ldots, \{N\}$).
For $\{ N - 2\}$, we have that $\rank \mat{\W}{\{N - 2\}} = 2$ whereas $\rank \mat{\W'}{\{ N - 2 \}} = 1$.
To see it is so, notice that:
\[
	\mat{\W}{ \{N - 2\} }_{:2, :4} = \begin{bmatrix}
		1 & 0 & 1 & 0 \\
		0 & 1 & 0 & 1
	\end{bmatrix} 
	\quad , \quad
	\mat{\W'}{\{N - 2\}}_{:2, :4} = \begin{bmatrix}
		1 & 0 & 0 & 1 \\
		1 & 0 & 0 & 1
	\end{bmatrix} 
	\text{\,,}
\]
and all other entries of $\mat{\W}{ \{N - 2 \} }$ and $\mat{\W'}{ \{N - 2\} }$ hold zero.
This means that $( \rank \mat{ \W }{\nu} )_{\nu \in \htmodetree \setminus \{ [N] \}} \leq  ( \rank \mat{\W'}{\nu})_{\nu \in \htmodetree \setminus \{ [N] \}}$ does not hold.
On the other hand, for $\{N - 1\}$ we have that $\rank \mat{\W}{\{N - 1\}} = 1$ while $\rank \mat{\W'}{\{N - 1\}} = 2$, because:
\[
\mat{\W}{ \{N - 1\}}_{:2, :4} = \begin{bmatrix}
	1 & 0 & 0 & 1 \\
	1 & 0 & 0 & 1
\end{bmatrix} 
\quad , \quad
\mat{\W'}{\{ N - 1\}}_{:2, :4} = \begin{bmatrix}
	1 & 0 & 1 & 0 \\
	0 & 1 & 0 & 1
\end{bmatrix} 
\text{\,,}
\]
and the remaining entries of $\mat{\W}{\{ N - 1\}}$ and $\mat{\W'}{\{ N - 1 \}}$ hold zero.
This implies that $( \rank \mat{ \W' }{\nu} )_{\nu \in \htmodetree \setminus \{ [N] \}} \leq  ( \rank \mat{\W}{\nu} )_{\nu \in \htmodetree \setminus \{ [N] \}}$ does not hold, and so the hierarchical tensor ranks of $\W$ and $\W'$ are incomparable, \ie~neither is smaller than or equal to the other.

It remains to show that there exists no $( R''_\nu )_{ \nu \in \htmodetree \setminus \{ [N] \} } \in \RR_\htmodetree \setminus \left \{ ( \rank \mat{\W}{\nu} )_{\nu \in \htmodetree \setminus \{ [N] \}},  ( \rank \mat{\W'}{\nu} )_{\nu \in \htmodetree \setminus \{ [N] \}} \right \}$ satisfying $( R''_\nu )_{ \nu \in \htmodetree \setminus \{ [N] \} } \leq ( \rank \mat{ \W }{\nu} )_{\nu \in \htmodetree \setminus \{ [N] \}}$ or $( R''_\nu )_{ \nu \in \htmodetree \setminus \{ [N] \} } \leq ( \rank \mat{ \W' }{\nu} )_{\nu \in \htmodetree \setminus \{ [N] \}}$.
Assume by way of contradiction that there exists such $( R''_\nu )_{ \nu \in \htmodetree \setminus \{ [N] \} }$, and let $\W'' \in \R^{D_1, \ldots, D_N}$ be a solution of this hierarchical tensor rank.
We now prove that $( R''_\nu )_{ \nu \in \htmodetree \setminus \{ [N] \} } \leq ( \rank \mat{ \W }{\nu} )_{\nu \in \htmodetree \setminus \{ [N] \}}$ entails a contradiction.
Since $( R''_\nu )_{ \nu \in \htmodetree \setminus \{ [N] \} }$ is not equal to the hierarchical tensor rank of $\W$, there exists $\nu \in \htmodetree \setminus \{[N]\}$ for which $\rank \mat{\W''}{\nu} = R''_\nu < \rank \mat{\W}{\nu}$.
Let us examine the possible cases:
\begin{itemize}
	\item If $\nu$ does not contain $N - 2, N - 1,$ and $N$, then $\rank \mat{\W}{\nu} = 1$ as all rows but the first of this matricization are zero.
	In this case $\mat{\W''}{\nu}  = R''_\nu = 0$, implying that $\W''$ is the zero tensor, in contradiction to it fitting the (non-zero) observed entries from \eq~\eqref{eq:ht_multiple_minima_observations}.
	
	\item If $\nu$ contains $N$ but not $N-2$ and $N-1$, then $\rank \mat{\W}{\nu} = 2$ since:
	\[
	\mat{\W}{\nu}_{:2, :4} = \begin{bmatrix}
		1 & 1 & 0 & 0 \\
		0 & 0 & 1 & 1
	\end{bmatrix} 
	\text{\,,}
	\]
	and all other entries of $\mat{\W}{\nu}$ hold zero.
	In this case $\mat{\W''}{\nu}  = R''_\nu < 2$.
	However, the fact that $\W''$ fits the observed entries from \eq~\eqref{eq:ht_multiple_minima_observations} leads to a contradiction, as $\mat{\W''}{\nu}$ must contain at least two linearly independent columns.
	To see it is so, notice that:
	\[
	\mat{\W^*}{\nu}_{:2, :4} = \begin{bmatrix}
		1 & ? & ? & 0 \\
		0 & ? & ? & 1
	\end{bmatrix} 
	\text{\,,}
	\]
	where recall that $?$ stands for an unobserved entry.
	
	\item If $\nu$ contains $N - 1$ but not $N - 2$ and $N$, then $\rank \mat{\W}{\nu} = 1$ since:
	\[
	\mat{\W}{\nu}_{:2, :4} = \begin{bmatrix}
		1 & 0 & 0 & 1 \\
		1 & 0 & 0 & 1
	\end{bmatrix} 
	\text{\,,}
	\]
	and all other entries of $\mat{\W}{\nu}$ hold zero.
	In this case $\mat{\W''}{\nu}  = R''_\nu = 0$, which means that $\W''$ is the zero tensor, in contradiction to it fitting the (non-zero) observed entries from \eq~\eqref{eq:ht_multiple_minima_observations}.
	
	\item If $\nu$ contains $N - 2$ but not $N - 1$ and $N$, then $\rank \mat{\W}{\nu} = 2$ since:
	\[
	\mat{\W}{\nu}_{:2, :4} = \begin{bmatrix}
		1 & 0 & 1 & 0 \\
		0 & 1 & 0 & 1
	\end{bmatrix} 
	\text{\,,} 
	\]
	and all other entries of $\mat{\W}{\nu}$ hold zero.
	In this case $\mat{\W''}{\nu}  = R''_\nu  < 2$.
	Noticing that $\mat{\W''}{\{N - 2\}}_{:2, :4} = \mat{\W''}{\nu}_{:2, :4}$, and that entries of $\mat{\W''}{\{N - 2\}}$ outside its top $2$-by-$4$ submatrix hold zero, we get that $\mat{\W''}{\{N - 2\}} = \mat{\W''}{\nu} < 2$.
	Furthermore, from the assumption that $( R''_\nu )_{ \nu \in \htmodetree \setminus \{ [N] \} } \leq ( \rank \mat{ \W }{\nu} )_{\nu \in \htmodetree \setminus \{ [N] \}}$ and the previous three cases, we know that $R''_{\{n\}} = \mat{\W}{\{n\}} = 1$ for all $n \in [N - 3]$, $R''_{\{N\}} = \mat{\W}{\{N\}} = 2$, and $R''_{\{N - 1\}} = \mat{\W}{\{N - 1\}} = 1$.
	Any tensor $\V \in \R^{D_1, \ldots, D_N}$ that satisfies $\rank \mat{\V}{\{n\}} \leq R_{\{n\}} \in \N$ for all $n \in [N]$ can be represented as:
	\[
	\V = \sum\nolimits_{r_1 = 1}^{R_{\{1\}}} \cdots \sum\nolimits_{r_N = 1}^{R_{\{N\}}} \mathcal{C}_{r_1, \ldots, r_N} \tenp_{n = 1}^N \Ubf^{(n)}_{:, r_n}
	\text{\,,}
	\]
	where $\mathcal{C} \in \R^{R_{\{1\}}, \ldots, R_{\{N\}}}$ and $\big ( \Ubf^{(n)} \in \R^{D_n, R_{\{n\}}} \big )_{n = 1}^N$ (see, \eg,~Section 4 in~\citet{kolda2009tensor}).
	Thus, there exist $c_1, c_2 \in \R$, $\big ( \Ubf^{(n)} \in \R^{D_n, 1} \big )_{n = 1}^{N - 1}$, and $\Ubf^{(N)} \in \R^{D_N, 2}$ such that:
	\[
	\W'' = c_1 \cdot \bigl ( \tenp_{n = 1}^{N - 1} \Ubf^{(n)}_{:, 1} \bigr ) \tenp \Ubf^{(N)}_{:, 1} + c_2 \cdot \bigl ( \tenp_{n = 1}^{N - 1} \Ubf^{(n)}_{:, 1} \bigr ) \tenp \Ubf^{(N)}_{:, 2} 
	\text{\,.}
	\]
	By multilinearity of the tensor product, we may write: $\W'' = \bigl ( \tenp_{n = 1}^{N - 1} \Ubf^{(n)}_{:, 1} \bigr ) \tenp \big ( c_1 \cdot \Ubf^{(N)}_{:, 1} + c_2 \cdot \Ubf^{(N)}_{:2} \big )$, and so $\W''$ has tensor rank one (it can be represented as a single non-zero tensor product between vectors).
	Since the tensor rank of a given tensor upper bounds the ranks of its matricizations (Remark 6.21 in~\citet{hackbusch2012tensor}), $R''_{\{ n \}} = \rank \mat{ \W'' }{\{ n \}} = 1$ for all $n \in [N]$ (the matrix ranks of these matricizations cannot be zero as $\W''$ is not the zero tensor).
	Hence, we have arrived at a contradiction~---~$2 = R''_{ \{N\}} \leq 1$.
	
	\item Contradictions in the remaining cases, where $\nu$ contains $N - 2, N - 1,$ and $N$, or any two of them, readily follow from the previous cases due to the fact that $\mat{\V}{\nu} = \mat{\V}{[N] \setminus \nu}^\top$ for any tensor $\V \in \R^{D_1, \ldots, D_N}$, and that the matrix rank of a matrix is equal to the matrix rank of its transpose.
	In particular, for any such $\nu$, it holds that $\rank \mat{ \W'' }{\nu} = \rank \mat{\W''}{[N] \setminus \nu}$ and $\rank \mat{\W}{\nu} = \rank \mat{\W}{[N] \setminus \nu }$.
	Therefore, if $\rank \mat{\W''}{\nu} = R''_\nu < \rank \mat{\W}{\nu}$, then $\rank \mat{\W''}{[N] \setminus \nu} < \rank \mat{\W}{[N] \setminus \nu}$.
	Since $\nu$ contains $N - 2, N - 1,$ and $N$, or any two of them, its complement $[N] \setminus \nu$ contains none or just one of them.
	Each of these scenarios was already covered in previous cases, which imply that $\rank \mat{\W''}{[N] \setminus \nu} < \rank \mat{\W}{[N] \setminus \nu}$ entails a contradiction.
\end{itemize}

In all cases, we have established that the existence of $( R''_\nu )_{ \nu \in \htmodetree \setminus \{ [N] \} } \in \RR_\htmodetree$, different from $( \rank \mat{\W}{\nu} )_{\nu \in \htmodetree \setminus \{ [N] \}}$ and $( \rank \mat{\W'}{\nu} )_{\nu \in \htmodetree \setminus \{ [N] \}}$, satisfying $( R''_\nu )_{ \nu \in \htmodetree \setminus \{ [N] \} } \leq ( \rank \mat{ \W }{\nu} )_{\nu \in \htmodetree \setminus \{ [N] \}}$ leads to a contradiction.
The claim for $\W'$, \ie~that there exists no such $( R''_\nu )_{ \nu \in \htmodetree \setminus \{ [N] \} }$ satisfying $( R''_\nu )_{ \nu \in \htmodetree \setminus \{ [N] \} } \leq ( \rank \mat{ \W' }{\nu} )_{\nu \in \htmodetree \setminus \{ [N] \}}$, is proven analogously.
Combined with the previous part of the proof, in which we established that neither $( \rank \mat{ \W }{\nu} )_{\nu \in \htmodetree \setminus \{ [N] \}}$ nor $( \rank \mat{ \W' }{\nu} )_{\nu \in \htmodetree \setminus \{ [N] \}}$ is smaller than or equal to the other, we conclude that $( \rank \mat{ \W }{\nu} )_{\nu \in \htmodetree \setminus \{ [N] \}}$ and $( \rank \mat{ \W' }{\nu} )_{\nu \in \htmodetree \setminus \{ [N] \}}$ are two different minimal elements of $\RR_\htmodetree$.
\qed

\subsection{Proof of Proposition~\ref{prop:htf_cnn}}
\label{app:proofs:htf_cnn}
For $l \in [L]$, the output of the $l$'th convolutional layer at index $n \in [N / P^{l - 1}]$ and channel $r \in [R_{l}]$ depends solely on inputs $\xbf^{( (n - 1) \cdot P^{l - 1} + 1 )}, \ldots ,\xbf^{(n \cdot P^{l - 1})}$.
Hence, we denote it by $conv_{l, n, r} \big ( \xbf^{( (n - 1) \cdot P^{l - 1} + 1 )}, \ldots ,\xbf^{(n \cdot P^{l - 1})} \big )$.
We may view the output linear layer as a $1 \times 1$ convolutional layer with a single output channel.
Accordingly, let $conv_{L + 1, 1, 1} \big ( \xbf^{(1)}, \ldots, \xbf^{(N)} \big ) := f_\theta \big ( \xbf^{(1)}, \ldots, \xbf^{(N)} \big )$ and $\tensorpart{L +1, 1}{1} := \tensorend$.

We show by induction over the layer $l \in [L + 1]$ that for any $n \in [N / P^{l - 1}]$ and $r \in [R_{l}]$:
\be
conv_{l, n, r} \big ( \xbf^{( (n - 1) \cdot P^{l - 1} + 1 )}, \ldots ,\xbf^{(n \cdot P^{l - 1})} \big ) = \inprod{ \tenp_{p = (n -1 ) \cdot P^{l - 1} + 1}^{n \cdot P^{l - 1}} \xbf^{(p)} }{ \tensorpart{l, n}{r} }
\text{\,.}
\label{eq:htf_cnn_inductive_claim}
\ee
For $l = 1$, let $n \in [N]$ and $r \in [R_1]$.
From the definition of $\tensorpart{1,n}{r}$ (\eq~\eqref{eq:ht_pary_end_tensor}) we can see that:
\[
conv_{1, n, r} \big ( \xbf^{(n)} \big ) = \inprod{\xbf^{ (n)} }{ \weightmat{1, n}_{:, r} } = \inprod{ \xbf^{(n)} }{ \tensorpart{1, n}{r} }
\text{\,.}
\]
Now, assuming that the inductive claim holds for $l - 1 \geq 1$, we prove that it holds for $l$.
Fix some $n \in [N / P^{l - 1}]$ and $r \in [R_{l}]$.
The $l$'th convolutional layer is applied to the output of the $l - 1$'th hidden layer, denoted $\big ( \hbf^{(l - 1, 1)}, \ldots, \hbf^{(l - 1, N / P^{l - 1})} \big ) \in \R^{R_{l - 1}} \times \cdots \times \R^{R_{l - 1}}$.
Each $\hbf^{(l-1, n)}$, for $n \in [N / P^{l - 1}]$, is a result of the product pooling operation (with window size $P$) applied to the output of the $l-1$'th convolutional layer. 
Thus:
\[
\begin{split}
conv_{l, n, r} \big ( \xbf^{( (n - 1) \cdot P^{l - 1} + 1 )}, \ldots ,\xbf^{(n \cdot P^{l - 1})} \big ) & = \sum_{r' = 1}^{R_{l - 1}} \weightmat{l, n}_{r', r} \cdot \hbf^{(l - 1, n)}_{r'} \\
& = \sum_{r' = 1}^{R_{l - 1}} \weightmat{l, n}_{r', r}  \cdot \,\,  \prod_{\mathclap{p = (n - 1) \cdot P + 1}}^{n \cdot P} \,\,\, conv_{l - 1, p. r'} \big ( \xbf^{ ((p - 1) \cdot P^{l - 2} + 1) }, \ldots, \xbf^{ (p \cdot P^{l - 2}) } \big )
\text{\,.}
\end{split}
\]
The inductive assumption for $l - 1$ then implies that:
\[
conv_{l, n, r} \big ( \xbf^{( (n - 1) \cdot P^{l - 1} + 1 )}, \ldots ,\xbf^{(n \cdot P^{l - 1})} \big ) = \sum_{r' = 1}^{R_{l - 1}} \weightmat{l, n}_{r', r} \cdot \,\, \prod_{\mathclap{p = (n - 1) \cdot P + 1}}^{n \cdot P} \,\,\, \inprod{ \tenp_{n' = (p - 1) \cdot P^{l - 2} + 1}^{p \cdot P^{l - 2}} \xbf^{ (n') } }{ \tensorpart{l - 1, p }{r'} }
\text{\,.}
\]
For any tensors $\A, \A', \B, \B'$ such that $\A$ is of the same dimensions as $\A'$ and $\B$ is of the same dimensions as $\B'$, it holds that $\inprod{ \A \tenp \B}{ \A' \tenp \B' } = \inprod{ \A }{ \A' } \cdot \inprod{ \B }{ \B' }$.
We may therefore write:
\[
\begin{split}
& conv_{l, n, r} \big ( \xbf^{( (n - 1) \cdot P^{l - 1} + 1 )}, \ldots ,\xbf^{(n \cdot P^{l - 1})} \big ) \\[1mm]
&\quad\quad = \sum\nolimits_{r' = 1}^{R_{l - 1}} \weightmat{l, n}_{r', r} \cdot \inprod{ \tenp_{p = (n - 1) \cdot P + 1}^{n \cdot P} \left ( \tenp_{n' = (p - 1) \cdot P^{l - 2} + 1}^{p \cdot P^{l - 2}} \xbf^{ (n') } \right ) }{ \tenp_{p = (n - 1) \cdot P + 1}^{n \cdot P} \tensorpart{l - 1, p }{r'} } \\[1mm]
&\quad\quad = \sum\nolimits_{r' = 1}^{R_{l - 1}} \weightmat{l, n}_{r', r} \cdot \inprod{ \tenp_{p = (n - 1) \cdot P^{l - 1} + 1}^{n \cdot P^{l - 1}} \xbf^{ (p) } }{ \tenp_{p = (n - 1) \cdot P + 1}^{n \cdot P} \tensorpart{l - 1, p }{r'} } \\[1mm]\
&\quad\quad = \inprod{ \tenp_{p = (n - 1) \cdot P^{l - 1} + 1}^{n \cdot P^{l - 1}} \xbf^{ (p) } }{ \sum\nolimits_{r' = 1}^{R_{l - 1}} \weightmat{l, n}_{r', r}  \left [ \tenp_{p = (n - 1) \cdot P + 1}^{n \cdot P} \tensorpart{l - 1, p }{r'} \right ] }
\text{\,.}
\end{split}
\]
Noticing that $\tensorpart{l,n}{r} = \sum\nolimits_{r' = 1}^{R_{l - 1}} \weightmat{l, n}_{r', r} \big [ \tenp_{p = (n - 1) \cdot P + 1}^{n \cdot P} \tensorpart{l - 1, p }{r'} \big ]$ (\eq~\eqref{eq:ht_pary_end_tensor}) establishes \eq~\eqref{eq:htf_cnn_inductive_claim}.

Applying the inductive claim for $l = L + 1, n = 1,$ and $r = 1$, while recalling that $L = \log_P N$, yields:
\[
f_\theta \big ( \xbf^{(1)}, \ldots, \xbf^{(N)} \big ) = conv_{L + 1, 1, 1} \big ( \xbf^{(1)}, \ldots, \xbf^{(N)} \big ) = \inprod{ \tenp_{n = 1}^N \xbf^{(n)} }{ \tensorpart{L+1, 1}{1} } = \inprod{ \tenp_{n = 1}^N \xbf^{(n)} }{ \tensorend }
\text{\,.}
\]
\qed

\subsection{Proof of Theorem~\ref{thm:loc_comp_norm_unbal_dyn}}
\label{app:proofs:loc_comp_norm_unbal_dyn}
Let $t \geq 0$ be a time at which $\htfcompnorm{\nu}{r} (t) := \normnoflex{\tenp_{\wbf \in \localcomp (\nu, r)} \wbf (t) } = \prod\nolimits_{ \wbf \in \localcomp (\nu, r) } \norm{ \wbf (t) } > 0$.
We differentiate $\htfcompnorm{\nu}{r} (t)^2$ with respect to time as done in the proof of \thm~\ref{thm:loc_comp_norm_bal_dyn} (\subapp~\ref{app:proofs:loc_comp_norm_bal_dyn}).
From the product rule and \lem~\ref{lem:weightvec_sq_norm_time_deriv} we get that:
\[
\frac{d}{dt} \htfcompnorm{\nu}{r} (t)^2 = 2 \htfcompnorm{\nu}{r} (t) \inprod{ - \nabla \htfendloss ( \tensorend (t)) }{ \htfcomp{\nu}{r} (t) } \sum\nolimits_{ \wbf \in \localcomp (\nu, r) } \prod\nolimits_{ \wbf' \in \localcomp (\nu, r) \setminus \{ \wbf \} } \norm{ \wbf' (t) }^2
\text{\,.}
\]
Since according to the chain rule $\frac{d}{dt} \htfcompnorm{\nu}{r} (t)^2 = 2 \htfcompnorm{\nu}{r} (t) \cdot \frac{d}{dt} \htfcompnorm{\nu}{r} (t)$, the equation above leads to:
\be
\frac{d}{dt} \htfcompnorm{\nu}{r} (t) = \inprod{ - \nabla \htfendloss ( \tensorend (t)) }{ \htfcomp{\nu}{r} (t) } \sum\nolimits_{ \wbf \in \localcomp (\nu, r) } \prod\nolimits_{ \wbf' \in \localcomp (\nu, r) \setminus \{ \wbf \} } \norm{ \wbf' (t) }^2
\text{\,.}
\label{eq:localcompnorm_time_deriv_interm_unbal}
\ee
By \lem~\ref{lem:loc_comp_sq_norm_diff_invariant}, the unbalancedness magnitude is constant through time, and so it remains equal to $\epsilon$~---~its value at initialization.
Hence, for any $\wbf \in \localcomp (\nu , r)$:
\be
\norm{ \wbf (t) }^2 \leq \min_{\wbf' \in \localcomp (\nu, r) } \norm{ \wbf' (t) }^2 + \epsilon = \Big ( \min_{\wbf' \in \localcomp (\nu, r) } \norm{ \wbf' (t) } \Big )^{L_\nu \cdot \frac{2}{L_\nu} } + \epsilon \leq \htfcompnorm{\nu}{r} (t)^{ \frac{2}{ L_\nu } } + \epsilon
\text{\,.}
\label{eq:sq_weightvec_norm_up_local_comp_norm}
\ee

If $\inprodbig{- \nabla \htfendloss ( \tensorend (t) ) }{ \htfcomp{\nu}{r} (t) } \geq 0$,
applying the inequality above to each $\norm{ \wbf' (t) }^2$ in \eq~\eqref{eq:localcompnorm_time_deriv_interm_unbal} yields the upper bound from \eq~\eqref{eq:loc_comp_norm_unbal_pos_bound}:
\[
\begin{split}
	\frac{d}{dt} \htfcompnorm{\nu}{r} (t) & \leq \inprod{ - \nabla \htfendloss ( \tensorend (t)) }{ \htfcomp{\nu}{r} (t) } \sum\nolimits_{ \wbf \in \localcomp (\nu, r) } \prod\nolimits_{ \wbf' \in \localcomp (\nu, r) \setminus \{ \wbf \} } \left ( \htfcompnorm{\nu}{r} (t)^{ \frac{2}{ L_\nu } } + \epsilon \right ) \\
	& =  \left ( \htfcompnorm{\nu}{r} (t)^{\frac{2}{ L_\nu } } + \epsilon \right )^{ L_\nu - 1 } \cdot L_\nu \inprodbig{- \nabla \htfendloss ( \tensorend (t) ) }{ \htfcomp{\nu}{r} (t) }
	\text{\,.}
\end{split}
\]
To prove the lower bound from \eq~\eqref{eq:loc_comp_norm_unbal_pos_bound}, we multiply and divide each summand on the right hand side of \eq~\eqref{eq:localcompnorm_time_deriv_interm_unbal} by the corresponding $\norm{\wbf (t) }^2$ (non-zero because $\htfcompnorm{\nu}{r} (t) > 0$), \ie:
\[
\begin{split}
	\frac{d}{dt} \htfcompnorm{\nu}{r} (t)& = \inprod{ - \nabla \htfendloss ( \tensorend (t)) }{ \htfcomp{\nu}{r} (t) } \sum\nolimits_{ \wbf \in \localcomp (\nu, r) } \norm{ \wbf (t) }^{-2} \cdot \prod\nolimits_{ \wbf' \in \localcomp (\nu, r) } \norm{ \wbf' (t) }^2 \\
	& = \inprod{ - \nabla \htfendloss ( \tensorend (t)) }{ \htfcomp{\nu}{r} (t) } \htfcompnorm{\nu}{r} (t) \cdot \sum\nolimits_{ \wbf \in \localcomp (\nu, r) } \norm{ \wbf (t) }^{-2}
	\text{\,.}
\end{split}
\]
By \eq~\eqref{eq:sq_weightvec_norm_up_local_comp_norm} we know that $\norm{ \wbf (t) }^{-2} \geq \big ( \htfcompnorm{\nu}{r} (t)^{ \frac{2}{ L_\nu} } + \epsilon \big )^{-1}$.
Thus, applying this inequality to the equation above establishes the desired lower bound.

If $\inprodbig{- \nabla \htfendloss ( \tensorend (t) ) }{ \htfcomp{\nu}{r} (t) } <  0$, the upper and lower bounds in \eq~\eqref{eq:loc_comp_norm_unbal_neg_bound} readily follow by similar derivations, where the difference in the direction of inequalities is due to the negativity of $\inprodbig{- \nabla \htfendloss ( \tensorend (t) ) }{ \htfcomp{\nu}{r} (t) }$.
\qed

\subsection{Proof of Proposition~\ref{prop:matrank_eq_seprank}}
\label{app:proofs:matrank_eq_seprank}
We partition the proof into two parts: the first shows that $\rank \mat{\tensorend}{I} \geq \seprank (f_\Theta ; I)$, and the second establishes the converse.

\paragraph*{Proof of lower bound ($\rank \mat{\tensorend}{I} \geq \seprank (f_\Theta ; I)$)}
Denote $R := \rank \mat{ \tensorend}{I}$, and assume without loss of generality that $I = [\abs{I}]$.
Since $\mat {\tensorend}{I}$ is a rank $R$ matrix, there exist $\vbf^{(1)}, \ldots, \vbf^{(R)} \in \R^{\prod_{n = 1}^{\abs{I}} D_n}$ and $\bar{\vbf}^{(1)}, \ldots, \bar{\vbf}^{(R)} \in \R^{ \prod_{n = \abs{I} + 1 }^{ N } D_n }$ such that:
\[
\mat{ \tensorend }{I} = \sum\nolimits_{r = 1}^R \vbf^{(r)} \big ( \bar{\vbf}^{(r)}\big )^\top
\text{\,.}
\]
For each $r \in [R]$, let $\V^{(r)} \in \R^{D_1, \ldots, D_{ \abs{I} } }$ be the tensor whose arrangement as a column vector is equal to $\vbf^{(r)}$, \ie~$\mat{ \V^{(r)} }{I} = \vbf^{(r)}$.
Similarly, for every $r \in [R]$ let $\widebar{\V}^{(r)} \in \R^{D_{ \abs{I} + 1}, \ldots, D_N }$ be the tensor whose arrangement as a row vector is equal to $( \bar{\vbf}^{(r)} )^\top$, \ie~$\mat{ \widebar{\V}^{(r)} }{\emptyset} = ( \bar{\vbf}^{(r)} )^\top$.
Then:
\[
\begin{split}
\mat{ \tensorend }{I} &= \sum\nolimits_{r = 1}^R \vbf^{(r)}  \big ( \bar{\vbf}^{(r)} \big )^\top \\
& = \sum\nolimits_{r = 1}^R \mat{ \V^{(r)} }{I} \kronp \mat{ \widebar{\V}^{(r)} }{\emptyset} \\
& = \sum\nolimits_{r = 1}^R \mat { \V^{(r)} \tenp \widebar{\V}^{(r)} }{I} \\
& = \mat{ \sum\nolimits_{r = 1}^R \V^{(r)} \tenp \widebar{\V}^{(r)} }{I}
\text{\,,}
\end{split}
\]
where the third equality makes use of \lem~\ref{lem:tenp_eq_kronp}, and the last equality is by linearity of the matricization operator.
Since matricizations merely reorder the entries of tensors, the equation above implies that $\tensorend = \sum_{r = 1}^R \V^{(r)} \tenp \widebar{\V}^{(r)}$.
We therefore have that:
\[
\begin{split}
f_\Theta \big ( \xbf^{(1)}, \ldots, \xbf^{(N)} \big ) & = \inprod{ \tenp_{n = 1}^N \xbf^{(n)}  }{ \tensorend } \\
& = \inprodBig{ \tenp_{n = 1}^N \xbf^{(n)}  }{ \sum\nolimits_{r = 1}^R \V^{(r)} \tenp \widebar{\V}^{(r)} } \\
& = \sum\nolimits_{r = 1}^R \inprod{ \tenp_{n = 1}^N \xbf^{(n)}  }{ \V^{(r)} \tenp \widebar{\V}^{(r)} }
\text{\,.}
\end{split}
\]
For any $\A, \A' \in \R^{D_1, \ldots, D_{ \abs{I} } }$ and $\B, \B' \in \R^{ D_{ \abs{I} + 1 } , \ldots, D_N }$ it holds that $\inprod{ \A \tenp \B}{ \A' \tenp \B' } = \inprod{ \A }{ \A' } \cdot \inprod{ \B }{ \B' }$.
Thus:
\[
f_\Theta \big ( \xbf^{(1)}, \ldots, \xbf^{(N)} \big ) = \sum\nolimits_{r = 1}^R \inprod{ \tenp_{n = 1}^N \xbf^{(n)}  }{ \V^{(r)} \tenp \widebar{\V}^{(r)} } = \sum\nolimits_{r = 1}^R \inprod{ \tenp_{n = 1}^{ \abs{I} } \xbf^{ (n) } }{ \V^{(r)} } \cdot \inprod{ \tenp_{n = \abs{I}  + 1}^{ N } \xbf^{ (n) } }{ \widebar{\V}^{(r)} }
\text{\,.}
\]
By defining $g_r : \times_{n =1}^{ \abs{I} } \R^{D_n} \to \R$ and $\bar{g}_r :\times_{n = \abs{I} + 1}^{N} \R^{D_n} \to \R$, for $r \in [R]$, as:
\[
g_r \big ( \xbf^{(1)}, \ldots, \xbf^{(\abs{I} )} \big ) = \inprod{ \tenp_{n = 1}^{ \abs{I} } \xbf^{ (n) } }{ \V^{(r)} } \quad , \quad \bar{g}_r \big ( \xbf^{ ( \abs{I} + 1 ) }, \ldots, \xbf^{(N)} \big ) = \inprod{ \tenp_{n = \abs{I}  + 1}^{ N } \xbf^{ (n) } }{ \widebar{\V}^{(r)} }
\text{\,,}
\]
we arrive at the following representation of $f_\Theta$ as a sum, where each summand is a product of two functions --- one that operates over inputs indexed by $I$ and another that operates over inptus indexed by $[N] \setminus I$:
\[
f_\Theta \big ( \xbf^{(1)}, \ldots, \xbf^{(N)} \big ) = \sum\nolimits_{r = 1}^R g_r \big ( \xbf^{(1)}, \ldots, \xbf^{ ( \abs{I} ) } \big ) \cdot \bar{g}_r \big ( \xbf^{( \abs{I} + 1 )}, \ldots, \xbf^{ ( N ) } \big )
\text{\,.}
\]
Since the separation rank of $f_\Theta$ is the minimal number of summands required to express it in such a manner, we conclude that $\rank \mat{\tensorend}{I} = R \geq \seprank (f_\Theta ; I)$.

\paragraph*{Proof of upper bound ($\rank \mat{\tensorend}{I} \leq \seprank (f_\Theta ; I)$)}
Towards proving the upper bound, we establish the following lemma.
\begin{lemma}
	\label{lem:grid_tensor_mat_rank_ub_by_sep_rank}
	Given $f : \times_{n = 1}^N \R^{D_n} \to \R$ and any $\big ( \xbf^{( 1, h_1 )} \in \R^{D_1} \big )_{h_1 = 1}^{H_1}, \ldots, \big ( \xbf^{ (N, h_N) } \in \R^{D_N} \big )_{h_N = 1}^{H_N}$, let $\W \in \R^{H_1, \ldots, H_N}$ be the tensor defined by $\W_{h_1, \ldots, h_N} := f \big ( \xbf^{ (1, h_1) }, \ldots, \xbf^{ (N, h_N) }\big )$ for all $(h_1, \ldots, h_N) \in [H_1] \times \cdots \times [H_N]$.
	Then, for any $I \subset [N]$:
	\[
	\rank \mat{\W}{I} \leq \seprank (f ; I)
	\text{\,.}
	\]
	In words, for any tensor holding the outputs of $f$ over a grid of inputs, the rank of its matricization according to $I$ is upper bounded by the separation rank of $f$ with respect to $I$.
\end{lemma}
\begin{proof}
	If $\seprank (f ; I)$ is $\infty$ or zero, \ie~$f$ cannot be represented as a finite sum of separable functions (with respect to $I$) or is identically zero, then the claim is trivial. 
	Otherwise, denote $R := \seprank (f; I)$, and assume without loss of generality that $I = [\abs{I}]$.
	Let $g_1, \ldots, g_R : \times_{n =1}^{ \abs{I} } \R^{D_n} \to \R$ and $\bar{g}_1, \ldots, \bar{g}_R :\times_{n = \abs{I} + 1}^{N} \R^{D_n} \to \R$ such that:
	\be
	f \big ( \xbf^{(1)}, \ldots, \xbf^{(N)} \big ) = \sum\nolimits_{r = 1}^R g_r \big ( \xbf^{(1)}, \ldots, \xbf^{ ( \abs{I} ) } \big ) \cdot \bar{g}_r \big ( \xbf^{( \abs{I} + 1 )}, \ldots, \xbf^{ ( N ) } \big )
	\text{\,.}
	\label{eq:grid_tensor_mat_ub_by_sep_rank:sep_rank}
	\ee
	We define $\big ( \V^{(r)} \in \R^{D_1, \ldots, D_{\abs{I}} } \big )_{r = 1}^R$ to be the tensors holding the outputs of $( g_r )_{r = 1}^R$ over the grid of inputs
	\[
	\big ( \xbf^{( 1, h_1 )} \big )_{h_1 = 1}^{H_1}, \ldots, \big ( \xbf^{ (\abs{I} , h_{ \abs{I} } ) } \big )_{h_{ \abs{I} } = 1}^{ H_{ \abs{I} } }
	\text{\,,}
	\]
	\ie~for all $h_1, \ldots, h_{\abs{I}} \in [H_1] \times \cdots \times [ H_{ \abs{I} } ]$ and $r \in [R]$ it holds that $\V^{(r)}_{ h_1, \ldots, h_{\abs{I}} } = g_r \big ( \xbf^{(1, h_1)}, \ldots, \xbf^{(\abs{I}, h_{ \abs{I} } )} \big )$.
	Similarly, we let $\big ( \widebar{\V}^{(r)} \in \R^{D_{\abs{I} + 1}, \ldots, D_{N} } \big )_{r = 1}^R$ be the tensors holding the outputs of $( \bar{g}_r )_{r = 1}^R$ over their respective grid of inputs, \ie~for all $h_{\abs{I} + 1}, \ldots, h_N \in [H_{ \abs{I} + 1}] \times \cdots \times [H_N]$ and $r \in [R]$ it holds that $\widebar{\V}^{(r)}_{ h_{\abs{I} + 1}, \ldots, h_N  } = \bar{g}_r \big ( \xbf^{ ( \abs{I} + 1, h_{\abs{I} + 1} )}, \ldots, \xbf^{ ( N, h_N ) } \big )$.
	
	By \eq~\eqref{eq:grid_tensor_mat_ub_by_sep_rank:sep_rank} and the definitions of $\W, ( \V^{(r)} )_{r = 1}^R$, and $( \widebar{\V}^{(r)} )_{r = 1}^R$, we have that for any $h_1, \ldots, h_N \in [H_1] \times \cdots \times [H_N]$:
	\[
	\begin{split}
		\W_{h_1 ,\ldots, h_N} & = f \big ( \xbf^{ (1, h_1) }, \ldots, \xbf^{ (N, h_N) }\big ) \\
		& = \sum\nolimits_{r = 1}^R g_r \big ( \xbf^{(1, h_1)}, \ldots, \xbf^{ ( \abs{I}, h_{ \abs{I} } ) } \big ) \cdot \bar{g}_r \big ( \xbf^{( \abs{I} + 1, h_{\abs{I} + 1} )}, \ldots, \xbf^{ ( N, h_N ) } \big )\\
		& =  \sum\nolimits_{r = 1}^R \V^{(r)}_{h_1, \ldots, h_{ \abs{I} }} \cdot \widebar{\V}^{(r)}_{ h_{ \abs{I} +  1}, \ldots, h_N }
		\text{\,,}
	\end{split}
	\]
	which means that $\W = \sum_{r = 1}^R \V^{(r)} \tenp \widebar{\V}^{(r)}$.
	From the linearity of the matricization operator and \lem~\ref{lem:tenp_eq_kronp} we then get that $\mat{\W}{I} = \sum_{ r = 1}^R \mat{\V^{(r)}}{I} \kronp \mat{ \widebar{\V}^{(r)} }{\emptyset}$.
	Since $\mat{\V^{(r)}}{I}$ is a column vector and $\mat{ \widebar{\V}^{(r)} }{\emptyset}$ is a row vector for all $r \in [R]$, we have arrived at a representation of $\mat{\W}{I}$ as a sum of $R$ tensor products between two vectors.
	A tensor product of two vectors is a rank one matrix, and so, due to the sub-additivity of rank we conclude: $\rank \mat{\W}{I} \leq R = \seprank (f; I)$.
\end{proof}

\medskip 

Now, consider the grid of inputs defined by the standard bases of $\R^{D_1}, \ldots, R^{D_N}$, \ie~by:
\[
\big ( \ebf^{( 1, d_1 )} \in \R^{D_1} \big )_{d_1 = 1}^{D_1}, \ldots, \big ( \ebf^{ (N, d_N) } \in \R^{D_N} \big )_{d_N = 1}^{D_N} 
\text{\,,}
\]
where $\ebf^{(n, d_n)}$ is the vector holding one at its $d_n$'th entry and zero elsewhere for $n \in [N]$ and $d_n \in [D_n]$.
With Lemma~\ref{lem:grid_tensor_mat_rank_ub_by_sep_rank} in hand, $\rank \mat{\tensorend}{I} \leq \seprank (f_\Theta ; I)$ follows by showing that $\tensorend$ is the tensor holding the outputs of $f_\Theta$ over this grid of inputs.
Indeed, for all $d_1, \ldots, d_N \in [D_1] \times \cdots \times [D_N]$:
\[
f_\Theta \big ( \ebf^{(1, d_1)}, \ldots, \ebf^{(N, d_N)} \big ) = \inprodbig{ \tenp_{n = 1}^N \ebf^{(n, d_n)}  }{ \tensorend } = (\tensorend)_{d_1, \ldots, d_N}
\text{\,.}
\]
\qed

\end{document}